\documentclass[11pt]{article}
\usepackage{amsmath,amsfonts,amssymb,mathrsfs}
\usepackage{graphicx,psfrag,epsf,url,wrapfig}
\usepackage{enumerate,color,centernot}
\usepackage{float}
\usepackage[numbers]{natbib}
\usepackage{adjustbox,booktabs,arydshln}
\usepackage{caption}
\usepackage{subcaption}
\usepackage{xcolor}

\usepackage[ruled,vlined]{algorithm2e}
\usepackage{graphicx}

\newcommand{\bR}{{\boldsymbol R}}
\newcommand{\bX}{{\boldsymbol X}}
\newcommand{\bY}{{\boldsymbol Y}}
\newcommand{\by}{{\boldsymbol y}}
\newcommand{\bx}{{\boldsymbol x}} 
\newcommand{\bz}{{\boldsymbol z}}
\newcommand{\bZ}{{\boldsymbol Z}}

\newcommand{\be}{{\boldsymbol e}}
\newcommand{\bw}{{\boldsymbol w}}

\newcommand{\bv}{{\boldsymbol v}}
\newcommand{\bt}{{\boldsymbol t}}

\newcommand{\bP}{{\boldsymbol P}}
\newcommand{\mK}{\mathcal{K}}
\newcommand{\bV}{{\boldsymbol V}}

\newcommand{\mG}{\mathcal{G}}
\newcommand{\mH}{\mathcal{H}}
\newcommand{\mE}{\mathbb{E}}
\newcommand{\mX}{\mathcal{X}}
\newcommand{\mM}{\mathcal{M}}

\newcommand{\bu}{{\boldsymbol u}}

\newcommand{\bbeta}{{\boldsymbol \beta}}
\newcommand{\bgamma}{{\boldsymbol \gamma}}
\newcommand{\btheta}{{\boldsymbol \theta}}

\newcommand{\bmu}{{\boldsymbol \mu}}
\newcommand{\bSigma}{{\boldsymbol \Sigma}}

\newtheorem{theorem}{Theorem}[section]
\newtheorem{lemma}{Lemma}[section]
\newtheorem{definition}{Definition}[section]

\newtheorem{remark}{Remark}
\newtheorem{assumption}{Assumption}
\newenvironment{proof}{\trivlist\item[\hskip \labelsep{\sc Proof:}]}
 {\unskip\nobreak\ \lower.3ex\hbox{$\Box$}\endtrivlist}

\begin{document}
 

\title{Extended Fiducial Inference: Toward an Automated Process of Statistical Inference}
 
\author{Faming Liang\thanks{Correspondence author: Faming Liang, email: fmliang@purdue.edu, 
 Department of Statistics, Purdue University, West Lafayette, IN 47907, USA; 
 $^\dag$ Department of Population Medicine, Harvard Medical School/ Harvard Pilgrim Health Care Institute, Boston, MA 02215, USA; 
$^\ddag$ Department of Biostatistics, Epidemiology, and Informatics, University of Pennsylvania, Philadelphia, PA 19104, USA. 
 }, \  Sehwan Kim$^{\dag}$, and Yan Sun$^{\ddag}$}



\date{} 

 
\maketitle

\begin{abstract}

While fiducial inference was widely considered a big blunder by R.A. Fisher, the goal he initially set --`inferring the uncertainty of model parameters on the basis of observations' -- has been continually pursued by many statisticians. To this end, we develop a new statistical inference method called extended Fiducial inference (EFI). The new method 
achieves the goal of fiducial inference by leveraging advanced statistical computing techniques 
while remaining scalable for big data. EFI involves jointly imputing random errors realized in observations using stochastic gradient Markov chain Monte Carlo 
 and estimating the inverse function using a sparse deep neural network (DNN). The consistency of the sparse DNN estimator ensures that the uncertainty embedded in observations is properly propagated to model parameters through the estimated inverse function, thereby validating downstream statistical inference.
Compared to frequentist and Bayesian methods, EFI offers significant advantages in parameter estimation and hypothesis testing. Specifically, EFI provides higher fidelity in parameter estimation, especially when outliers are present in the observations; and eliminates the need for theoretical reference distributions in hypothesis testing, thereby automating the statistical inference process. EFI also provides an innovative framework for semi-supervised learning.

\vspace{2mm}
  
\end{abstract}

 



\section{Introduction}

 Statistical inference is a fundamental task in modern data science, which studies how to propagate the uncertainty embedded in data to model parameters. 
 During the past century, frequentist and Bayesian methods have evolved 
 as two major frameworks of statistical inference. 
  However, due to some intrinsic issues (see Section \ref{conceptsection}), these methods may lack one or more features --- such as fidelity, automaticity, and scalability ---
 necessary for performing statistical inference on complex models in modern data science. Specifically, the frequentist methods often estimate
 model parameters using the maximum  likelihood approach and test hypotheses by comparing a test statistic with a  known theoretical reference distribution. 
 It is well-known that the maximum likelihood estimator (MLE) can be significantly 
 influenced by outliers, 
 which reduces the fidelity of parameter estimates. For hypothesis testing, the required theoretical reference distribution is test statistic-dependent, making statistical inference difficult to automate. Although this issue can be partially mitigated by asymptotic normality, the sample size required to achieve asymptotic normality can be very large especially in high-dimensional scenarios.  
For Bayesian methods, their dependence on prior distributions has been a subject of criticism throughout the history of Bayesian statistics, often raising concerns about their fidelity.

As a possible way to overcome the drawbacks of frequentist and Bayesian methods, the fiducial
method has been proposed by R.A. Fisher in a series of papers starting from 1930s 
(see \cite{Zabell1992Fisher} for a review), which quantifies uncertainty of model parameters by the so-called fiducial distribution. 
Fisher originally introduced this method, motivated by the observation that pivotal quantities permit uncertainty quantification for an unknown parameter  
 in the same way as the frequentist method. However, he encountered difficulties in extending this pivotal quantity-based method to models with multiple parameters. 
It is worth noting that for some models, the fiducial distribution is the same as the posterior distribution derived with Jeffreys' prior, but Fisher argued that the logic behind the Bayesian method is unacceptable because the use of prior is unjustifiable \citep{hannig2009gfi}. This argument also distinguishes the fiducial method from objective Bayesian methods, even though non-informative priors are used in the latter. 

Fiducial inference was generally regarded as a big blunder by Fisher. However, the goal he initially set, {\it making inference about unknown parameters on the basis of observations} \citep{Fisher1956Book}, has been continually pursued by many statisticians.  Building on early works in sparse deep learning 
\citep{Liang2018BNN, SunSLiang2021,SunXLiang2021NeurIPS} and adaptive stochastic gradient Markov chain Monte Carlo (MCMC) \cite{deng2019adaptive,LiangSLiang2022, DongZLiang2022}, 
 this paper develops a new statistical inference framework called the {\it extended fiducial inference} (EFI), 
 which achieves the initial goal of fiducial inference while possessing
 necessary features like {\it fidelity}, {\it automaticity}, and {\it scalability} that are essential for statistical inference in modern data science.  
 
 Our contributions in this work are in three folds: 
 \begin{itemize}
    \item {\it Development of the EFI framework:} We develop a scalable and effective method for conducting fiducial inference. Our method involves jointly imputing the random errors contained in the data and estimating the inverse function for the model parameters. It ensures that the uncertainty embedded in the data is properly propagated to the model parameters through the estimated inverse function, thereby validating downstream statistical inference. Compared to frequentist and Bayesian methods, EFI provides higher-fidelity inference, especially in the presence of outliers.

\item {\it Innovative statistical framework for semi-supervised learning:} EFI provides an innovative   framework of statistical inference for missing data problems, especially in scenarios where missing values are present in response data, as encountered in semi-supervised learning problems. This innovation 
can have profound implications for modern data science, particularly in biomedical research where obtaining labeled data can be costly.   
     
\item {\it  Automaticity of statistical inference:}  EFI enables automatic statistical inference for complex  models, at least conceptually. It can be as flexible 
as frequentist methods in parameter estimation. However, unlike frequentist methods,  
it eliminates the requirement for theoretical reference distributions  (including asymptotic normality as a special case) in hypothesis testing. 
Compared to Bayesian methods, EFI eliminates the requirement for prior distributions, which can vary depending on the problem or analyst's choice, thus enhancing the fidelity of statistical inference.
 \end{itemize}

 In summary, with the aid of advanced statistical computing techniques, EFI holds the potential to significantly advance modern data science. Specifically, it provides higher-fidelity inference, introduces an innovative statistical framework for semi-supervised learning, and automates statistical inference for 
 complex models.

   
The remaining part of the paper is organized as follows. Section 2 
distinguishes the concepts of frequentist, Bayesian and EFI from the perspective of structural inference  \citep{Fraser1966StructuralPA,Fraser1968Book}. 
Section 3 provides a theoretical framework for EFI. 
Section 4 describes an effective algorithm for performing EFI 
and studies its theoretical properties. 
Section 5 presents some numerical examples validating EFI as a statistical inference method. 
Section 6 presents applications of EFI on semi-supervised learning.
Section 7 presents applications of EFI for complex hypothesis tests. 
Section 8 concludes the paper with a brief discussion.

\section{Frequentist, Bayesian, and Extended Fiducial Inference} \label{conceptsection}

This section elaborates the conceptual difference between frequentist, Bayesian, and EFI methods from the perspective of structural inference \citep{Fraser1966StructuralPA,Fraser1968Book}. Consider a regression model: 
\begin{equation} \label{modeleq}
Y=f(X,Z,\btheta), 
\end{equation}
where   $Y\in \mathbb{R}$ and $X\in \mathbb{R}^{d}$ represent the response and explanatory variables, respectively; $\btheta\in \mathbb{R}^p$ represents the vector of unknown parameters; 
and $Z\in \mathbb{R}$ represents a scaled random error that follows  
 a known distribution denoted by $\pi_0(\cdot)$. 
Suppose that a random sample of size $n$ has been collected from the model, denoted by $\{(y_1,x_1), (y_2,x_2),\ldots,(y_n,x_n)\}$, and our goal is to quantify uncertainty of $\btheta$ based on the collected samples (also known as observations). 

In the view of structural inference \citep{Fraser1966StructuralPA,Fraser1968Book}, we can express the observations  $\{(y_1,x_1), (y_2,x_2),\ldots$, $(y_n,x_n)\}$ 
in the data generating equation as follow:  
\begin{equation} \label{dataGeneq}
y_i=f(x_i,z_i,\btheta), \quad i=1,2,\ldots,n.
\end{equation}
This system of equations consists of $n+p$ unknowns, namely, $\{ \btheta, z_1, z_2, \ldots, z_n
\}$, while there are only $n$ equations. Therefore, the values of $\btheta$ cannot be uniquely determined by the data-generating equation, which gives the source of uncertainty of the parameters as illustrated by Figure \ref{uncertainty_source1}. 
For convenience, we will refer to $z_1, z_2,\ldots, z_n$ as latent variables in the context of data-generating equations, while still calling them random errors when appropriate.



\begin{figure}[htbp]
    \centering
    \includegraphics[width=0.65\textwidth]{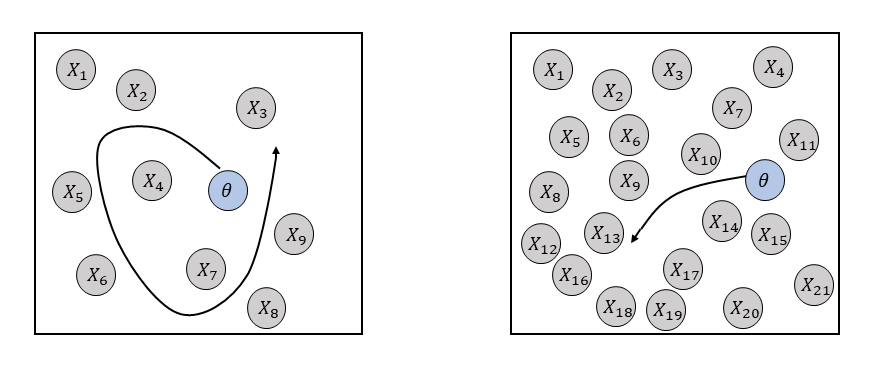}
    \caption{Illustration for the source of uncertainty of model parameters: the space that $\btheta$ can take values becomes smaller and smaller as the sample size increases.}
    \label{uncertainty_source1}
\end{figure}

\paragraph{Frequentist Methods}
The frequentist methods treat $\btheta$ as fixed unknowns. To solve for $\btheta$ from the undetermined system (\ref{dataGeneq}), they  often impose a  constraint on the system such that the latent variables can be dismissed and $\btheta$ can be uniquely determined. 
For example, the maximum likelihood estimation method works under the constraint that the joint likelihood function of the samples, or equivalently, the likelihood of 
 $\{z_1,z_2,\ldots,z_n\}$, is maximized. 
As an illustration, let's consider the linear regression model:
 \begin{equation} \label{structeq} 
 y_i=x_i^T \bbeta+ \sigma z_i, \quad i=1,2,\ldots,n,
  \end{equation}
 where $\bbeta\in \mathbb{R}^{p-1}$ is the regression coefficient vector, 
 $\sigma \in \mathbb{R}^+$ is a positive scale parameter, and $z_1,z_2,\ldots,z_n$ are  
 i.i.d standard Gaussian random variables. For this model, 
the maximum likelihood estimation method is to solve  for $\btheta:=(\bbeta,\sigma)$ subject to the constraint 
\begin{equation} \label{MLEconstraint}
\prod_{i=1}^n \phi(z_i)=\max_{(\tilde{z}_1,\tilde{z}_2,\ldots,\tilde{z}_n)\in \mathbb{R}^n} \prod_{i=1}^n \phi(\tilde{z}_i),
\end{equation}
where $\phi(\cdot)$ denotes the standard Gaussian density function. As it turns out, this is equivalent to solving the optimization problem:
\begin{equation} \label{MLEobjeq}
\max_{(\bbeta,\sigma)} \sum_{i=1}^n \log \phi\left( \frac{y_i-x_i^T\bbeta}{\sigma} \right),
\end{equation}
and the resulting estimator is given by 
\begin{equation} \label{MLEeq} 
\hat{\bbeta}=(\bX_n^T \bX_n)^{-1} \bX_n^T \bY_n, \quad  \hat{\sigma}^2=\frac{1}{n}\sum_{i=1}^n (y_i-x_i^T\hat{\bbeta})^2, 
\end{equation} 
where  $\bY_n=(y_1,\ldots,y_n)^T$ and $\bX_n=(x_1,x_2,\ldots,x_n)^T$. 

Another example of frequentist methods is  moment estimation, which solves for $\btheta$ under the constraint that the sample moments are equal to the population moments. 
For the model (\ref{dataGeneq}), the moment constraint can be expressed as   
\[
\sum_{i=1}^n y_i^k= \sum_{i=1}^n \int [f(x_i,z,\btheta)]^k \pi_0(z) dz, \quad i=1,2,\ldots, p,
\]
where the latent variables $z_1,z_2,\ldots,z_n$ are dismissed via integration. 


Let $\hat{\btheta}$ denote an estimator of $\btheta$. The frequentist method assesses the uncertainty of $\btheta$ in an unconditional mode, where the distribution of $\hat{\btheta}$ is derived based on the preassumed distribution $\pi_0(z)$ instead of the random errors ${z_1,z_2,\ldots,z_n}$ realized in the observations.
For example, considering the MLE given in (\ref{MLEeq}), one can derive that 
$\hat{\bbeta} \sim N(\bbeta, \sigma^2 (\bX_n^T \bX_n)^{-1})$ and $\frac{n \hat{\sigma}^2}{\sigma^2} \sim \chi^2(n-p+1)$ based on the preassumed Gaussian distribution for the random errors.
This unconditional mode makes the inference procedure challenging to automate; in particular, the distribution of $\hat{\btheta}$ is problem-dependent and generally difficult to derive. Additionally, the constraints used for the solution of $\btheta$ might be violated by observations. For example, when outliers exist, the maximum likelihood constraint might not hold, and the resulting MLE can significantly differ from the true value of $\btheta$.
 Refer to Section \ref{outliersection} for numerical examples.

\paragraph{Bayesian Methods}
  In contrast to frequentist methods, Bayesian methods treat $\btheta$ as random variables and circumvent the issue of latent variables 
  by adopting a conditional approach. Specifically, Bayesian methods assume that $\btheta$ follows a prior distribution, and quantify uncertainty of $\btheta$ based on the  conditional distribution (also known as the posterior distribution):
\begin{equation} \label{Bayeseq} 
\pi(\btheta|(y_1,x_1),(y_2,x_2),\ldots,(y_n,x_n)) = \frac{ \prod_{i=1}^n p(y_i|x_i,\btheta) \pi(\btheta)}{ \int  \prod_{i=1}^n p(y_i|x_i,\btheta) \pi(\btheta) d\btheta},
\end{equation} 
where  $p(y_i|x_i,\btheta)$ denotes the likelihood function of $y_i$, and $\pi(\btheta)$ represents the prior distribution of $\btheta$. 
The dependence of the inference on the prior distribution has been subject to criticism throughout the history of Bayesian statistics, as the prior distribution introduces subjective elements that may affect the fidelity of statistical inference.
 
\paragraph{Extended Fiducial Inference}
 Let $\bZ_n:=\{z_1,z_2,\ldots,z_n\}$ denote the collection of latent variables, and let $G(\bY_n,\bX_n,\bZ_n)$ denote an inverse function for the solution of $\btheta$ in the system (\ref{dataGeneq}). 
 As a general computational procedure, EFI jointly imputes  $\bZ_n$ and estimates $G(\bY_n,\bX_n,\bZ_n)$, and then quantifies the uncertainty of $\btheta$ 
 based the estimated inverse function and the imputed values of $\bZ_n$, where 
 the estimated inverse function serves as an uncertainty propagator from $\bZ_n$ to $\btheta$. 
Technically, EFI approximates $G(\bY_n,\bX_n,\bZ_n)$ using a sparse deep neural network (DNN) \citep{Liang2018BNN,SunSLiang2021,SunXLiang2021NeurIPS}, and employs an adaptive stochastic gradient MCMC algorithm  \citep{deng2019adaptive,LiangSLiang2022} to jointly simulate the values of $\bZ_n$ and estimate the parameters of the sparse  DNN.

\textcolor{black}{ While treating $\btheta$ as fixed unknowns, EFI distinguishes itself from frequentist methods by conducting inference for $\btheta$ in a conditional mode and sidestepping the imposition of any constraints on the latent variables. Additionally, unlike Bayesian methods, EFI eliminates the need for 
 placing a prior distribution on $\btheta$. In summary, EFI aims to make statistical inference of  $\btheta$ based solely on observations.
 }

\paragraph{Related Works}

During the past several decades, there have been quite a few works on statistical inference with the attempt to achieve the goal of fiducial inference, although some gaps remain. These works are briefly reviewed in what follows. 

\vspace{2mm}
\noindent {\it Generalized Fiducial Inference (GFI)}.
  Like EFI, GFI \citep{hannig2009gfi, Hannig2013GeneralizedFI, hannig2016gfi,Liu2022AGP,Murph2022GeneralizedFI} also attempts to solve the data generating equation, but employs an acceptance-rejection procedure 
  similar to the approximate Bayesian computation (ABC) algorithm \citep{Beaumont2002ABC}. As an illustration, let's consider model  (\ref{structeq}), for which the acceptance-rejection procedure 
   consists of the following steps: 
  \begin{itemize}
  \item[(a)] (Proposal) Generate $\tilde{\bZ}_n=(\tilde{z}_1, \tilde{z}_2,\ldots, \tilde{z}_n)^T$ from the Gaussian distribution $N(0,I_n)$.
  \item[(b)] ($\btheta$-fitting) Find the best fitting parameters $\tilde{\btheta}=\arg\min_{\btheta} \|\bY_n-\bX_n\bbeta-\sigma \tilde{\bZ}_n\|$, where $\|\cdot\|$ denotes an appropriate norm, and compute  the fitted value $\tilde{\bY}_n=\bX_n \tilde{\bbeta}+\tilde{\sigma} \tilde{\bZ}_n$.
  \item[(c)] (Acceptance-rejection) Accept  $\tilde{\btheta}$ if $\|\bY_n-\tilde{\bY}_n\| \leq \epsilon$ for some pre-specified small value $\epsilon$, and reject otherwise. 
\end{itemize}
Subsequently, statistical inference is made based on the accepted samples of $\tilde{\btheta}$. However, 
as $n$ increases, this procedure can become 
extremely inefficient due to its decreasing   acceptance rate. 

As a potential solution to resolving this computational issue, the limiting distribution of accepted $\tilde{\btheta}$ (as $\epsilon \to 0$) was derived in \cite{Hannig2013GeneralizedFI, hannig2016gfi, Liu2022AGP}. However, as shown in \cite{hannig2016gfi}, the limiting 
distribution  depends on the norm used in the above procedure.
Furthermore, for many problems, direct simulation of the limiting distribution might be challenging, 
  as shown in \cite{Liu2022AGP}, which involves the calculation of the determinant of an $n\times n$ matrix 
 at each iteration.
Quite recently, \cite{li2020deep} proposed replacing the $\btheta$-fitting step of the acceptance-rejection procedure 
with a mapping $\tilde{G}: (\bY_n,\bX_n,\tilde{\bZ}_n)\to \tilde{\btheta}$ pre-learned using a DNN. However, this replacement cannot improve the acceptance rate of $\tilde{\btheta}$, since  $\tilde{\bZ}_n$ is still proposed from an independent trial distribution. Other concerns about the replacement include the consistency of the DNN estimator and its difficulty in dealing with  cases where $\bX_n$ and $\bY_n$ contain missing data. 
\textcolor{black}{Compared to GFI, EFI provides a more feasible computational scheme for conducting fiducial inference, in addition to some conceptual differences in defining the fiducial distribution as discussed later.}

\vspace{2mm}  
\noindent 
{\it Structural Inference}.
Fraser \citep{Fraser1966StructuralPA,Fraser1968Book} introduced the concept of modeling the data as a function of  parameters and random errors through a structural equation (also known as data generating equation), under which statistical inference would be conditioned on the realized random errors. The structural inference approach  has successfully addressed some difficulties suffered by the pivotal quantity-based fiducial method. In particular, it avoids the issues of improper normalization \citep{Stein1959AnEO} and non-uniqueness \citep{Fieller1954SOMEPI,Mauldon1955PIVOTALQF}. 
However, like the Bayesian method, the structural inference method can suffer from the marginalization paradox \citep{Dawid1973Mar} 
that can cause inconsistency of inference. 
It is important to note that the structural equation concept 
has led to a fruitful framework for statistical inference. 
Both GFI and EFI are developed based on it. 
\textcolor{black}{
However, EFI differs significantly from structural inference in its treatment of 
$\btheta$. EFI regards $\btheta$ as fixed unknowns, whereas structural inference treats $\btheta$ as variables. As a result, EFI successfully sidesteps the marginalization paradox like a frequentist method.}
In EFI, $\btheta$ can be determined only in the limit $n\to \infty$, where the inverse function $G(\bY_n,\bX_n,\bZ_n)$ derived with finite samples can be understood as a stochastic estimator of $\btheta$ with a random component formed by $\bZ_n$.

\textcolor{black}{The {\it Dempster–Shafer theory} (see e.g., \cite{Dempster1967UL}, \cite{Shafer1976}, and \cite{Dempster2008TheDC}) and the {\it inferential model} (see e.g., \cite{Martin2013InferentialMA}, \cite{Martin2015ConditionalIM}, and \cite{Martin2015Book})  provide interesting frameworks for statistical reasoning with uncertainty. However, they are not primarily concerned with fiducial inference in the form Fisher conceived. The {\it inferential model} method 
 avoids imposing any constraints on the latent variables $\bZ_n$ but instead conducts inference for $\btheta$ in an unconditional mode, as discussed in \cite{Martin2014DiscussionFO}. }
 It achieves this by working with a low-dimensional association model,  which is built upon the sufficient or summary statistics for $\btheta$ and includes only a limited number of latent variables. Leveraging this association model, it subsequently constructs a confidence set for $\btheta$ using the {\it Dempster-Shafer theory} by considering a set of plausible random errors pre-constructed for the association model in an unconditional mode.
  For many statistical models, it yields the same confidence set as the maximum likelihood estimation method. To maintain conciseness of  this review, we omit detailed descriptions for them.

\section{Extended Fiducial Inference}

\subsection{Extended Fiducial Distribution} \label{EFDsection}

Before introducing the EFI method, we first define the extended fiducial distribution (EFD) as a confidence distribution (CD) estimator \citep{Xie2013ConfidenceDT} of $b(\btheta)$, 
where $b(\cdot)$ is a function of interest. Let's revisit the data generating equation (\ref{dataGeneq}) and begin by making several assumptions.

\begin{assumption} \label{ass:existence}
There exists an inverse function $G: \mathbb{R}^n \times \mathbb{R}^{n\times d} \times \mathbb{R}^{n} \to \mathbb{R}^p$: 
\begin{equation} \label{Inveq}
\btheta=G(\bY_n,\bX_n,\bZ_n).
\end{equation} 
\end{assumption}
In this context, ``inverse'' implies that if $(\bX_n, \bY_n, \bZ_n)$ satisfies 
$\bY_n = f(\bX_n, \bZ_n, \btheta)$ for some $\btheta$, 
then   $G(\bY_n, \bX_n, \bZ_n) = \btheta$ follows.
 From the perspective of parameter estimation, Assumption \ref{ass:existence} implies that the parameters are identifiable given the random error-augmented data 
$\{\bY_n,\bX_n,\bZ_n\}$. This is generally true when $n\geq p$, as in this case the system (\ref{dataGeneq}) has no more  unknowns than the number of equations by considering $\bZ_n$ as known. For the case $p>n$, we recommend reducing the dimension of the problem through an  application of a model-free sure independence screening procedure. More discussions on this issue can be found at the end of the paper. 

\textcolor{black}{It is worth noting that the inverse function is not necessarily constructed using all $n$ samples. For example, it can be simply constructed by solving any $p$ equations in (\ref{dataGeneq}) 
for $\btheta$. This raises an issue about non-uniqueness of $G(\cdot)$.
In what follows, we will study how the non-uniqueness of $G(\cdot)$ 
impacts the statistical inference for the unknowns $\bZ_n$ and $\btheta$.
}

For a given inverse function, we define an energy function: 
\[
U_n(\bz):=U(\bY_n,\bX_n,\bz,G(\cdot)).
\]
To ensure proper inference for the unknowns, the energy function $U_n(\cdot)$ needs to satisfy certain regularity conditions as outlined in Assumptions \ref{ass:zero}-\ref{ass:hwang1b}.
\begin{assumption} \label{ass:zero} 
 The energy function $U_n(\cdot)$ is  non-negative, 
 $\min_{\bz} U_n(\bz)$ exists and equals 0, 
  and $U_n(\bz)=0$ if and only if  $\bY_n=f(\bX_n,\bz,G(\bY_n,\bX_n,\bz))$. 
\end{assumption}

Let $\mathcal{Z}_n$ denote the zero-energy set 
\[
\mathcal{Z}_n=\big\{\bz \in \mathbb{R}^n: U_n(\bz)=0 \big\}. 
\]

\begin{lemma} \label{lem:manifold}
If Assumptions \ref{ass:existence}-\ref{ass:zero} hold, then the zero-energy set 
  $\mathcal{Z}_n$ is invariant to the choice of $G(\cdot)$.
\end{lemma} 
\begin{proof}
Suppose that there exist two inverse functions $G_1(\cdot)$ and $G_2(\cdot)$. Let 
$\mathcal{Z}_n^{(1)}$ and $\mathcal{Z}_n^{(2)}$ denote their respective zero-energy sets. 
For any $\bz \in \mathbb{R}^n$, if $\bz \in \mathcal{Z}_n^{(1)}$, then
$\bY_n=f(\bX_n,\bz,G_1(\bY_n,\bX_n,\bz))$ holds by Assumption \ref{ass:zero}. 
Let $\tilde{\btheta}=G_1(\bY_n,\bX_n,\bz)$. Hence,  $(\bY_n,\bX_n,\bz)$ satisfies the 
data generating equation (\ref{dataGeneq}) with the parameter $\tilde{\btheta}$. 

Since $G_2(\cdot)$ is also an inverse function for the data generating equation, we have 
$G_2(\bY_n,\bX_n,\bz)=\tilde{\btheta}$ by Assumption \ref{ass:existence}. 
This implies $\bY_n=f(\bX_n,\bz,G_2(\bY_n,\bX_n,\bz))$,  
and thus $\bz \in \mathcal{Z}_n^{(2)}$ according to Assumption \ref{ass:zero}.
That is, $\mathcal{Z}_n^{(1)} \subseteq \mathcal{Z}_n^{(2)}$ holds. Vice versa, we can show 
$\mathcal{Z}_n^{(2)} \subseteq \mathcal{Z}_n^{(1)}$. Therefore,  
$\mathcal{Z}_n^{(1)}=\mathcal{Z}_n^{(2)}$ and the zero-energy set is invariant 
to the choice of the inverse function. 
\end{proof}

Let $p_n^*(\bz|\bY_n,\bX_n)$ denote the extended fiducial density function of $\bZ_n$ 
on $\mathcal{Z}_n$.
To properly define  $p_n^*(\bz|\bY_n,\bX_n)$, we adopt a
limiting way. More precisely, we first define the conditional distribution 
\begin{equation} \label{z-dist}
p_{n}^{\epsilon}(\bz|\bX_n,\bY_n) \propto \exp\left\{-  \frac{U_n(\bz)}{\epsilon} \right\} \pi_0^{\otimes n}(\bz),
\end{equation}
where $\epsilon>0$ represents the temperature, and $\pi_0^{\otimes n}(\bz)=\pi_0(z_1)\times \pi_0(z_2) \times \cdots \times \pi_0(z_n)$ serves as 
the marginal distribution in the construction of this conditional distribution. 
Then, we define $p_n^*(\bz|\bY_n,\bX_n)$ as the limit 
\begin{equation} \label{distlimit}
p_n^*=\lim_{\epsilon \downarrow 0} p_{n}^{\epsilon}.
\end{equation}
This type of convergence  has been studied in \cite{Hwang1980LaplacesMR}. 
Specifically, the convergence can be studied in  two cases: (a) 
$\Pi_n(\mathcal{Z}_n)>0$ and (b) $\Pi_n(\mathcal{Z}_n)=0$, where
 $\Pi_n(\cdot)$ denotes a probability measure on $(\mathbb{R}^n, \mathscr{R})$  with $\mathscr{R}$ being the Borel $\sigma$-algebra and 
  the corresponding density function given by $\pi_0^{\otimes n}$.

\subsubsection{Case (a): $\Pi_n(\mathcal{Z}_n)>0$}
For this case, we follow \cite{Hwang1980LaplacesMR} to further assume that $U_n(\bz)$ satisfies:
\begin{assumption} \label{ass:hwang1a}
$\Pi_n(U_n(\bz)<a)>0$ for any $a>0$.
\end{assumption}
Then, following Proposition 2.2 of \cite{Hwang1980LaplacesMR}, it can be shown that  
the limiting probability measure of $p_n^{\epsilon}$ exists and is uniformly distributed 
on $\mathcal{Z}_n$ with respect to $\Pi_n$.  This is summarized in the following 
Theorem: 

\begin{theorem}\label{corPositive}
If Assumptions \ref{ass:existence}-\ref{ass:hwang1a} hold and $\Pi_n(\mathcal{Z}_n)>0$, then 
$p_n^*(\bz|\bX_n,\bY_n)$ is invariant to the choice of the inverse function $G(\cdot)$ and \textcolor{black}{the energy function $U_n(\cdot)$}, and it is given by   
\begin{equation} \label{theoreticalzdista}
\frac{d P_n^*(\bz|\bX_n,\bY_n)}{d\bz}= \frac{1}{\Pi_n(\mathcal{Z}_n)} \pi_0^{\otimes n}(\bz), \quad 
\bz \in \mathcal{Z}_n,
\end{equation}
where $P_n^*$ represents the cumulative distribution function (CDF) corresponding to $p_n^*$. 
\end{theorem}

The proof of Theorem \ref{corPositive} follows Proposition 2.2 of 
\cite{Hwang1980LaplacesMR} and Lemma \ref{lem:manifold} directly, and it is thus omitted.
An example of this case is the logistic regression as discussed in 
Section \ref{EFI:logistic} of the supplement, for which the energy function is defined as   
 \begin{equation} \label{penalty_logistic10}
U_n(\bz)= \sum_{i=1}^n \rho\Big((z_i-x_i^T G(\bY_n,\bX_n,\bZ_n))(2y_i-1) \Big),
   \end{equation}
where $z_1,z_2,\ldots,z_n \stackrel{iid}{\sim} Logistic(0,1)$ with the CDF given by $F(z)=1/(1+e^{-z})$, and $\rho(\cdot)$ is the ReLU function: $\rho(s)=s$ if $s>0$ and 0 otherwise.

\subsubsection{Case (b): $\Pi_n(\mathcal{Z}_n)=0$}

For this case,  we assume that $\mathcal{Z}_n$ forms a manifold in $\mathbb{R}^n$ with the highest dimension $p$.
 Following \cite{Hwang1980LaplacesMR}, by the {\it tubular neighborhood theorem} \citep{Milnor1974CharacteristicC},  we can decompose $\bz \in \mathcal{Z}_{n}$ as follows: 
\begin{equation} \label{decompeq}
\bz=m(u_1,u_2,\ldots,u_p)+t_1 \mathcal{N}(1)+\cdots+ t_{n-p} \mathcal{N}(n-p),
\end{equation}
where $m(u_1,u_2,\ldots,u_p)$ is local coordinates, and 
$\mathcal{N}(1), \ldots, \mathcal{N}(n-p)$ are normalized smooth normal vectors perpendicular to $\mathcal{Z}_{n}$. Let $\bt=(t_1,t_2,\ldots,t_{n-p})^T$. 
In addition to Assumptions \ref{ass:existence}-\ref{ass:hwang1a}, we assume the following conditions hold:

\begin{assumption} \label{ass:hwang1b}
    \item[(i)]  There exists $a>0$ such that $\{U_n(\bz) \leq a\}$ is compact. 
    \item[(ii)] $\pi_0^{\otimes n}(\bz)$ is continuous, and $U_n(\bz) \in C^3(\mathbb{R}^n)$ is three-time continuously differentiable. 
    \item[(iii)] $\mathcal{Z}_{n}$ has finitely many components and each component is a compact smooth manifold with the highest dimension $p$. 
    \item[(iv)]  $\pi_0^{\otimes n}(\bz)$ is not identically zero on the $p$-dimensional manifold,
    and det$( \frac{\partial^2 U}{\partial \bt^2}(\bz)) \ne 0$ for $\bz\in \mathcal{Z}_{n}$.
\end{assumption} 

\begin{lemma} \label{lemma:hwang3.1} (Theorem 3.1;  \cite{Hwang1980LaplacesMR}) \textcolor{black}{
If Assumptions \ref{ass:existence}-\ref{ass:hwang1b} hold, then  
 the limiting probability measure $p_n^*$ concentrates on the highest dimensional manifold and is given by 
 \begin{equation} \label{theoreticalzdist}
  \frac{d P_n^*(\bz|\bX_n,\bY_n)}{d \mathscr{\nu}}(\bz)= \frac{ \pi_0^{\otimes n}(\bz) \left( {\rm det}( \nabla^2_{\bt} U_n(\bz))\right)^{-1/2}} { \int_{\mathcal{Z}_{n}} \pi_0^{\otimes n}(\bz) \left( {\rm det}(\nabla_{\bt}^2 U_n(\bz) \right)^{-1/2} d \mathscr{\nu}}, \quad \bz \in \mathcal{Z}_{n},
 \end{equation} 
 where $\mathscr{\nu}$ is the sum of intrinsic measures on the $p$-dimensional manifold in $\mathcal{Z}_n$.}
\end{lemma}

Lemma \ref{lemma:hwang3.1} is a restatement of Theorem 3.1 of 
\cite{Hwang1980LaplacesMR} and its proof is thus omitted.  
We note that the distribution $P_n^*$ can also be derived using the co-area formula (see e.g., \cite{Fedrer1969GeometricMT}, section 3.2.12; \cite{Diaconis2013SamplingFA}, Proposition 2; \cite{Liu2022AGP}, Theorem 1) under similar conditions.

 Given the inverse function $G(\cdot)$, we define the parameter space 
 \[
  \Theta=\{\btheta \in \mathbb{R}^p: \btheta=G(\bY_n,\bX_n,\bz), \bz\in \mathcal{Z}_n\},
  \]
which represents the set of all possible values of $\btheta$ that $G(\cdot)$ takes when $\bz$ runs over $\mathcal{Z}_n$. Then 
for any function $b(\btheta)$, its EFD associated with the inverse function $G(\cdot)$ can be defined as follows: 

\begin{definition} \label{EFDdef} [EFD of $b(\btheta)$] Consider the data generating equation   (\ref{dataGeneq}) and an inverse function $\btheta=G(\bY_n,\bX_n,\bz)$. 
For any function $b(\btheta)$ of interest, its EFD associated with the inverse function $G(\cdot)$ is defined as  
\begin{equation} \label{EFDeq}
\begin{split}
\mu_n^*(B|\bY_n,\bX_n) &=\int_{\mathcal{Z}_n(B)} d P_n^*(\bz|\bY_n,\bX_n),   \quad 
\mbox{for any measurable set $B \subset \Theta$},
\end{split}
\end{equation}
where $\mathcal{Z}_n(B)=\{\bz\in \mathcal{Z}_n: b(G(\bY_n,\bX_n,\bz)) \in B\}$, 
and $P_n^*(\bz|\bY_n,\bX_n)$ is given by (\ref{theoreticalzdist}). 
\end{definition} 

Essentially, $b(G(\bY_n,\bX_n,\bZ_n))$ can be considered as an estimator of $b(\btheta)$, and 
Eq. (\ref{EFDeq}) represents the CD estimator of $b(\btheta)$ associated with the inverse function ${G}(\cdot)$. 
Here we would like to emphasize that viewing $\mu_n^*$ as a distribution function of $b(\btheta)$ (i.e., regarding $\btheta$ as a variable) is not appropriate, as in this context it will easily lead to a Bayesian approach for jointly simulating of $(\btheta,\bZ_n)$. The resulting sample pair $(\btheta,\bZ_n)$ will break the inverse mapping (\ref{Inveq}) and, in consequence, the uncertainty of $\bZ_n$ will not be properly propagated to $\btheta$.

\textcolor{black}{
For an effective implementation of EFI, we propose the following 
importance resampling procedure: 
\begin{itemize}
 \item[(a)] (Manifold sampling) For any given inverse function $\tilde{G}(\cdot)$, simulate $M$ samples, denoted by $\mathcal{S}_M=\{\bz_1, \bz_2, \ldots,\bz_M\}$, from  $\pi_0^{\otimes n}(\bz)$ subject to the constraint $U_n (\bz)=0$. This can be done using a constrained Monte Carlo algorithm such as constrained Hamiltonian Monte Carlo \citep{Brubaker2012AFO,Reich1996SymplecticIO}. 
 \item[(b)] (Weighting) Calculate the importance weight $\omega_i= \left( {\rm det}( \nabla^2_{\bt} U_n(\bz_i))\right)^{-1/2}$ for each sample $\bz_i\in \mathcal{S}$ using an inverse function $G(\cdot)$ of interest.
 \item[(c)] (Resampling) Draw $m$ samples from $\mathcal{S}_M$ without replacement according to the probabilities: $\frac{\omega_i}{\sum_{j=1}^M \omega_j}$ for $i=1,2,\ldots,M$.
 \item[(d)] (Inference) For $b(\btheta)$, find the EFD  associated with $G(\cdot)$ according to 
  (\ref{EFDeq}) based on the $m$ samples obtained in step (c). 
\end{itemize}
This procedure involves two inverse functions: $\tilde{G}(\cdot)$ and 
$G(\cdot)$. By Lemma \ref{lem:manifold}, any inverse function $\tilde{G}(\cdot)$ 
can be used in step (a) to generate samples from $\mathcal{Z}_n$, 
and this greatly facilitates comparisons of the inference results from 
different choices of $G(\cdot)$. 
If $G(\cdot)$ and $\tilde{G}(\cdot)$ are chosen to be the same,  $p_n^*(\bz|\bX_n,\bY_n) \propto \pi_0^{\otimes n}(\bz) \left( {\rm det}( \nabla^2_{\bt} U_n(\bz))\right)^{-1/2}$ 
can also be directly simulated on $\mathcal{Z}_n$ 
using a constrained Monte Carlo algorithm. }

\begin{remark} \label{rem:flexibility} (On the flexibility of EFI)
EFI provides a flexible framework of statistical inference. One can adjust the inverse function $G(\cdot)$ and the energy function $U_n(\cdot)$ to ensure that the resulting CD estimator of $b(\btheta)$ satisfies desired properties, such as efficiency, unbiasedness, and robustness. 
This mirrors the flexibility of frequentist methods, where different estimators of $b(\btheta)$ can be designed for different purposes. However, its conditional inference nature makes EFI even more attractive than frequentist methods, as it circumvents the need for  derivations of theoretical distributions of the estimators.
\end{remark}

Lastly, we note that the fiducial distribution defined above conceptually differs from that defined in GFI \cite{Hannig2013GeneralizedFI, hannig2016gfi, Liu2022AGP}. 
 Specifically, GFI interprets the fiducial distribution as the $\btheta$-marginal of a distribution defined on the manifold formed by the data generating equations in the joint space of $(\btheta, \bZ_n)\in \mathbb{R}^p \times 
\mathbb{R}^n$, while EFI interprets it as the $\btheta$-transformation of a distribution 
defined on a subset or manifold formed by the data generating equations in the sample space of $\bZ_n \in \mathbb{R}^n$. Our definition is consistent with the EFI algorithm developed in this paper.

\subsection{EFI for the Models with Additive Noise}

The importance resampling procedure proposed in Section \ref{EFDsection} is general for simulations of $p_n^*$,  but computing the importance weights can be challenging 
when the sample size $n$ is large. 
Specifically, it involves calculating the determinant of an $(n-p)\times (n-p)$-matrix at each iteration.  
To address this issue, we consider models with additive noise, which represent a broad class of models and have been extensively studied in the context of causal inference (see e.g., \cite{Peters2013CausalDW} and \cite{Hoyer2008NonlinearCD}). 
Additionally, we suggest setting the energy function as 
prescribed in Assumption \ref{ass:invariance}-(i), with 
the $L_2$-norm        
$U_n(\bz)=\|\bY_n-f(\bX_n,\bz,G(\bY_n,\bX_n,\bz))\|^2$ as a special case. 
Consequently, for these models, we show that 
the importance weight is reduced to a constant and, therefore, one can simulate from $p_n^*$ 
by directly simulating from $\pi_0^{\otimes n}(\bz)$ 
using a constrained Monte Carlo algorithm.

 \begin{assumption} \label{ass:invariance} 
 (i)  $U_n(\cdot)$ is specified in the form: $U_n(\bz)=h(J(\bz))=\sum_{i=1}^n h(e_i)$ for 
some function $h(\cdot)$ satisfying $\frac{\partial h(J)}{\partial J}(\bz)=0$ for any $\bz \in \mathcal{Z}_n$, where $J(\bz)=\bY_n-f(\bX_n,\bz,G(\bY_n,\bX_n,\bz))=(e_1,e_2,\ldots, e_n)^T$, and $e_i=y_i-f(x_i,z_i,\btheta)$ for $i=1,2,\ldots,n$; 
and (ii) the model noise is additive; i.e., the function $f(X,Z,\btheta)$ in model (\ref{modeleq}) is a linear function of $Z$.  
\end{assumption}



\begin{theorem} \label{corZero}
If Assumptions \ref{ass:existence}-\ref{ass:invariance} hold, then  
$P_n^*(\bz|\bY_n,\bX_n)$ given in (\ref{theoreticalzdist}) is invariant to the choices of $G(\cdot)$ and $U_n(\cdot)$. Furthermore, $P_n^*$ reduces to a truncated distribution of 
$\pi_0^{\otimes n}$ on the manifold $\mathcal{Z}_n$.
\end{theorem} 
\begin{proof}
Under Assumptions \ref{ass:existence}-\ref{ass:zero}, the uniqueness of $\mathcal{Z}_n$ 
has been established in Lemma \ref{lem:manifold}.  
If $U_n(\cdot)$ is specified as in Assumption \ref{ass:invariance}, 
the condition $\frac{\partial h(J)}{\partial J}(\bz)=0$ on $\mathcal{Z}_n$ implies 
 \begin{equation} \label{Udeteq}
 \nabla_{\bt}^2 U_n(\bz)= \left(\nabla_{\bt} J(\bz)\right)^T \nabla_J^2 h(J(\bz)) \nabla_{\bt} J(\bz).
 \end{equation}
Furthermore, with the aid of Assumption \ref{ass:hwang1b}-(iv) 
and the symmetric form of $h(\cdot)$ (with respect to $e_i$'s),
 $\nabla_J^2 h(J(\bz))$ reduces to a diagonal matrix of $\varsigma I_{n-k}$
for some positive constant $\varsigma>0$. 
This ensures the factor $\varsigma$ to be canceled out for the numerator and denominator in (\ref{theoreticalzdist}). Consequently,  
$P_n^*$ is invariant to the choice of $U_n(\cdot)$.

To further establish the invariance of $P_n^*$ with respect to the choice of $G(\cdot)$, 
we consider an inverse function 
\[
G(\bY_n,\bX_n,\bZ_n)=\hat{\btheta}(z_1,z_2,\ldots,z_p),
\]
where $\hat{\btheta}(z_1,z_2,\ldots,z_p)$ is obtained by solving the first $p$ equations 
in (\ref{dataGeneq}).  Here we assume the solution $\hat{\btheta}(z_1,z_2,\ldots,z_p)$ is unique for the $p$ equations. 
Let $\tilde{\bt}=(z_{p+1},z_{p+2},\ldots,z_n)^T$, which corresponds to a transformation 
of $\bt$ in (\ref{decompeq}). 
Then, it is easy to verify that at any point $\bz\in \mathcal{Z}_n$, 
 the first $p$ rows of the matrix $\nabla_{\tilde{\bt}} J(\bz) \in \mathbb{R}^{n \times (n-p)}$ are all zero, and 
 the remaining $(n-p)$-rows forms a $(n-p)\times (n-p)$-diagonal matrix for which 
 the diagonal elements are nonzero and expressed as a function of $(\bX_n,\btheta)$ 
 (i.e., a constant function of $\bz$) by the assumption 
 that $f(X,Z,\btheta)$ is a linear function of $Z$. 
 Therefore, at any point $\bz\in \mathcal{Z}_n$,  $\nabla_{\bt}^2 U_n(\bz)$ forms a positive-definite constant matrix with rank $n-p$;
 and $P_n^*$ in (\ref{theoreticalzdist}) reduces to  a truncated distribution of $\pi_0^{\otimes n}$ on $\mathcal{Z}_n$.

 In the same way, we can construct $\binom{n}{p}$ different inverse functions, 
 each obtained by choosing  a different set of $p$ equations to solve for $\btheta$. Therefore, each of them results in a positive definite matrix  $\nabla_{\bt}^2 U_n(\bz)$ and the same distribution $P_n^*$. 
 For any appropriate linear combination of these inverse functions, which still forms an inverse function, 
 the above result still holds. For the combination case, the desired result can be established via appropriate matrix operations, as illustrated using a linear regression example 
 in Section \ref{EFDproof} of the supplement. 
 
 Finally, we note that for any inverse function, since it solves all $n$ equations, it must also be a solver for a selected set of $p$ equations. 
 By the uniqueness of the solution for $p$ equations, the inverse function 
 can be regarded as  a linear combination
 of these $\binom{n}{p}$ basis inverse functions.  
 \end{proof}

\paragraph{Example 1} Consider the linear regression model (\ref{structeq}) again. Let $\btheta=(\bbeta,\sigma^2)$.
To conduct EFI for $\btheta$, we set $G(\bY_n,\bX_n,\bz)=\hat{\btheta}(z_1,z_2,\ldots,z_{p}):=(\hat{\bbeta}^T,\hat{\sigma})^T$,  a solver for the first $p$ equations in (\ref{dataGeneq}), and set the energy function
\begin{equation} \label{energyeq}
U_n(\bz)=\|\bY_n-f(\bX_n,\bz,G(\bY_n,\bX_n,\bz))\|^2.
\end{equation}
  Consequently, we  have 
\[
J(\bz)  =\begin{pmatrix} Y_{1:p}-X_{1:p} \hat{\bbeta}-\hat{\sigma} Z_{1:p} \\ Y_{(p+1):n}-X_{(p+1):n} \hat{\bbeta}-\hat{\sigma} Z_{(p+1):n} \end{pmatrix},  \quad
\nabla_{Z_{(p+1):n}} J(\bz)  = \begin{pmatrix} 0 \\ \hat{\sigma} I_{n-p} \end{pmatrix},
\]
where $Y_{1:p}$, $X_{1:p}$ and $Z_{1:p}$ to denote the response, explanatory, and noise variables of the first $p$ samples in the dataset, respectively; likewise, 
$Y_{(p+1):n}$, $X_{(p+1):n}$ and $Z_{(p+1):n}$ denote the response, explanatory, and noise variables of the last $n-p$ samples. 
 At any $\bz\in \mathcal{Z}_n$, we have $\hat{\btheta}=\btheta$, i.e., $\hat{\btheta}$ 
 can be treated as constants, and thus 
 $\nabla_{\bt}^2 U_n(\bz)=2 \tilde{D}^T \tilde{D}$ for some matrix $\tilde{D}$ of rank
 $n-p$. This yields the result that $p_n^*(\bz|\bX_n,\bY_n)$ is a truncation of 
 $\pi_0^{\otimes n}$ on $\mathcal{Z}_n$.

\paragraph{Example 1 (continuation)} 
 For simplicity, let's first consider the case where $\sigma^2$ is known. In this scenario, we set 
 \[
 G(\bY_n,\bX_n,\bZ_n)= (\bX_n^T \bX_n)^{-1} \bX_n^T (\bY_n-\sigma \bZ_n), 
 \]
 which utilizes all available data. Then the EFD of $\bbeta$, denoted by $\mu_n^*(\bbeta|\bY_n,\bX_n,\sigma^2)$, is given by $N((\bX_n^T\bX_n)^{-1} \bX_n^T\bY_n, \sigma^2 (\bX_n^T\bX_n)^{-1})$ after normalizing $p_n^*(\bz|\bY_n,\bX_n)$ on 
$\mathcal{Z}_n$,   which coincides with 
 the posterior distribution of $\bbeta$ under Jeffery's prior $\pi(\bbeta) \propto 1$.  
Furthermore, the resulting confidence set for $\bbeta$ is identical to \textcolor{black}{those obtained 
by the GFI and} ordinary least square (OLS) methods.

Next, let's consider the case where $\sigma^2$ is unknown. To find the EFD of $\sigma^2$, 
we solve the first $p-1$ equations for $\bbeta$, resulting in the solution: 
\[
 \tilde{\bbeta}= (\bX_{1:(p-1)}^T\bX_{1:(p-1)})^{-1} \bX_{1:(p-1)}^T(\bY_{1:p-1}-\sigma \bZ_{1:(p-1)}),
\]
where  $Y_{1:(p-1)}$, $X_{1:(p-1)}$ and $Z_{1:(p-1)}$ denote the response, explanatory, and noise variables of the first $p-1$ samples in the dataset, respectively. 
By substituting $\tilde{\bbeta}$ into each of the remaining $n-p+1$ equations and adjusting with the covariance of the $Z$-terms,
we obtain a combined 
solution for $\sigma^2$: 
\[
\small
\begin{split}
\tilde{\sigma}^2 &= \frac{(\bY_{p:n}-\bX_{p:n} (\bX_{1:(p-1)}^T\bX_{1:(p-1)})^{-1} \bX_{1:(p-1)}^T\bY_{1:p-1})^T \Sigma^{-1}  (\bY_{p:n}-\bX_{p:n} (\bX_{1:(p-1)}^T\bX_{1:(p-1)})^{-1} \bX_{1:(p-1)}^T\bY_{1:p-1})}{ (\bZ_{p:n}-\bX_{p:n} (\bX_{1:(p-1)}^T\bX_{1:(p-1)})^{-1} \bX_{1:(p-1)}^T\bZ_{1:p-1})^T \Sigma^{-1}  (\bZ_{p:n}-\bX_{p:n} (\bX_{1:(p-1)}^T\bX_{1:(p-1)})^{-1} \bX_{1:(p-1)}^T\bZ_{1:p-1})}, \\
 & := \frac{A}{W},
\end{split}
\]
where   
and  $Y_{p:n}$, $X_{p:n}$ and $Z_{p:n}$ denote the response, explanatory, and noise variables of the last $n-p+1$ samples in the dataset, respectively; and $\Sigma=I_{n-p+1}+\bX_{p:n} (\bX_{1:(p-1)}^T\bX_{1:(p-1)})^{-1} \bX_{p:n}^T$, representing the covariance matrix of $(\bZ_{p:n}-\bX_{p:n} (\bX_{1:(p-1)}^T\bX_{1:(p-1)})^{-1} \bX_{1:(p-1)}^T\bZ_{1:p-1})$. Here, $A$ forms an unbiased estimator of $(n-p+1)\sigma^2$, 
and $W$ follows 
a $\chi^2$-distribution with a degree-of-freedom of $n-p+1$. 
Therefore, if we set $\tilde{\sigma}^2$ as the inverse function ${G}(\cdot)$ 
for $\sigma^2$, the resulting EFD of $\sigma^2$ is given by 
\begin{equation} \label{EFDsigma}
\mu_n^*(\sigma^2|\bY_n,\bX_n)=\pi_{\chi^{-2}_{n-p+1}}\left(\frac{\sigma^2}{A}\right) \frac{1}{A},
\end{equation}
where $\pi_{\chi^{-2}_k}(u)$ denotes 
the density function of an inverse-chi-squared distribution with a degree-of-freedom of $k$.   
If we use the mean of $\mu_n^*(\sigma^2|\bY_n,\bX_n)$ as an estimator of $\sigma^2$, 
 it can be shown that it has a bias of $\frac{2}{n-p-3} \sigma^2$. In contrast, the MLE
 of $\sigma^2$ has a bias of $-\frac{p-1}{n} \sigma^2$. Therefore, the EFD results in a smaller bias than the MLE when $n>(p+3)(p-1)/(p-3)$. Note that, as stated in Remark \ref{rem:flexibility}, we 
 can adjust $\tilde{\sigma}^2$ by the factor $\frac{n-p-3}{n-p-1}$ to make the mean of the EFD 
 unbiased for $\sigma^2$, if desired. 


Finally, we can obtain the EFD of $\bbeta$ by completing the integration:
\begin{equation} \label{integraleq}
\mu_n^*(\bbeta|\bY_n,\bX_n)= \int \mu_n^*(\bbeta|\bY_n,\bX_n,\sigma^2) \mu_n^*(\sigma^2|\bY_n,\bX_n) d\sigma^2,
\end{equation}
which is a multivariate non-central t-distribution $t(\bmu_{\bbeta}, \bSigma_{\bbeta}, \nu_{\bbeta})$ with the parameters given by 
\[
\bmu_{\bbeta}=(\bX_n^T\bX_n)^{-1} \bX_n^T\bY_n, \quad 
\bSigma_{\bbeta}=\frac{A}{n-p+1} (\bX_n^T\bX_n)^{-1},
\quad \nu_{\bbeta}=n-p+1.
\]
The mean and covariance matrix of the EFD is given by $(\bX_n^T\bX_n)^{-1} \bX_n^T\bY_n$ and $\frac{A}{n-p-1} (\bX_n^T \bX_n)^{-1}$. 

\vspace{0.1in}

It is worth noting that our EFD (\ref{EFDsigma}) matches the result obtained with OLS. The latter often presents the result as 
\[
\frac{\bY_n^T(I_n-\bX_n(\bX_n^T \bX_n)^{-1} \bX_n^T)\bY_n}{\sigma^2} \sim \chi^2_{n-p+1},
\]
where the numerator forms an unbiased estimator of $(n-p+1)\sigma^2$. Similarly, in EFI,  $A$ serves as 
an unbiased estimator of $(n-p+1)\sigma^2$. Also, the EFD (\ref{EFDsigma}) can be represented 
as an inverse-Gamma distribution IG$(\alpha_g, \beta_g)$ 
with $\alpha_g=\frac{n-p+1}{2}$ and $\beta_g=\frac{A}{2}$, which is the same as the GFI solution \citep{hannig2016gfi} except for the expression of $A$. 

\begin{remark} \label{Rem:freq}  The EFD derivation procedure described in  
Example 1 can be extended to general nonlinear regression problems with additive noise. 
Consider the model $\bY_n=f(\bX_n,\bbeta)+\sigma \bZ_n:=\bmu_y+\sigma \bZ_n$, where $\bZ_n$ is
assumed to follow a known distribution symmetric about {\bf 0}. First, let's assume that 
$\sigma^2$ is known. 
Let $T(\bY_n)$ be the OLS estimator of $\bbeta$, which makes use of all $n$ samples. 
Let $\pi_T$ denote the distribution of $T(\bY_n)$, and  
 let $\mu_T^{(i)}=\eta^{(i)}(\bmu_y)$ denote its $i^{th}$ moment for $i=1,2,\ldots, k$.  
 We can regard $T(\bY_n- \sigma \bz)$ as an 
 inverse function for $\bbeta$. Consequently, by (\ref{EFDeq}), 
 the EFD of $\bbeta$, associated with $T(\bY_n- \sigma\bz)$, has the $i^{th}$ moment given by $\eta^{(i)}(\bY_n)$ for $i=1,2,\ldots,k$. Furthermore, EFI shares the same distribution $\pi_T$ as the frequentist method for quantifying the uncertainty of $\bbeta$.
 
 If $\sigma^2$ is unknown, we can follow the same procedure as described in Example 1 to find 
 the EFD for $\sigma^2$. Finally, we can obtain the EFD for $\bbeta$ by completing an integration similar to (\ref{integraleq}).
 \end{remark}

Theorem \ref{corZero} implies that for an additive noise model, if any 
inverse function $\tilde{G}(\cdot)$ is known, then 
$p_n^*(\bz|\bX_n,\bY_n)$ can be directly simulated from $\pi_0^{\otimes n}(\bz)$ subject to the constraint $U_n(\bz)=0$ using a constrained Monte Carlo algorithm.  
Furthermore, for $b(\btheta)$, the empirical EFD 
associated with a known inverse function $G(\cdot)$ can be constructed based on the samples simulated from $p_n^*(\bz|\bX_n,\bY_n)$.

\section{Extended Fiducial Inference with a Sparse DNN Inverse Function} \label{EFIsect}

As implied by Theorem \ref{corPositive}, Lemma \ref{lemma:hwang3.1}, 
and Theorem \ref{corZero},  conducting fiducial inference requires finding appropriate inverse functions. However, in practice, these inverse functions are typically very difficult to determine. 
To address this issue,  we propose approximating $\tilde{G}(\cdot)$ with a sparse DNN and employing an adaptive stochastic gradient MCMC algorithm to simultaneously simulate from $p_n^*(\bz|\bX_n,\bY_n)$ and train the sparse DNN.  Then an empirical EFD of $b(\btheta)$ associated with $G(\cdot)=\tilde{G}(\cdot)$ can be constructed based on the $\bz$-samples simulated from $p_n^*(\bz|\bX_n,\bY_n)$.  
We call this proposed algorithm the EFI-DNN algorithm.

\subsection{The EFI-DNN Algorithm} \label{sect:EFI-DNN}



 \begin{figure}[!ht]
    \centering
    \includegraphics[width=0.45\textwidth]{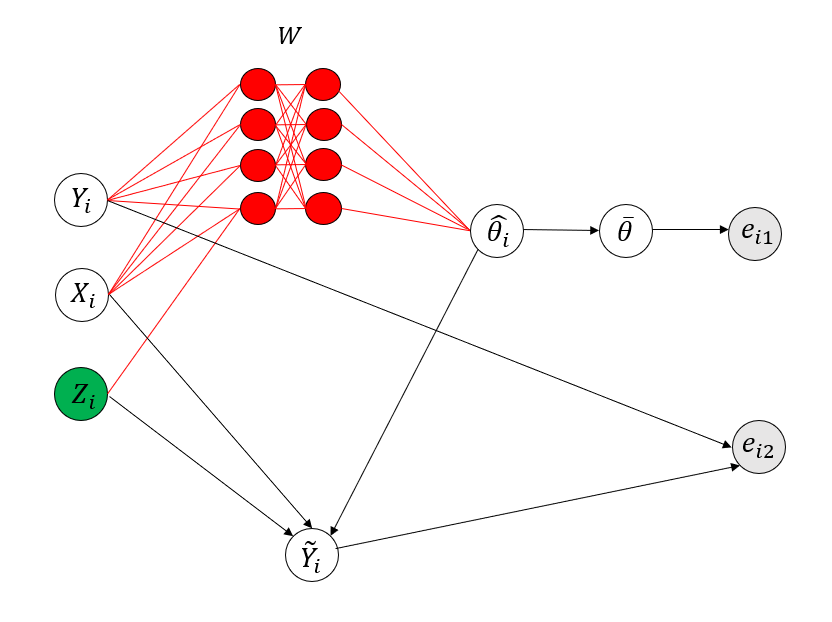}
    \caption{Illustration of the EFI network, where the red nodes and links form a DNN (parameterized by the weights $\bw$) to learn, the green node represents latent variables to impute, and the black lines represent deterministic functions. }
    \label{EFInetwork}
\end{figure}

The entire structure of the algorithm is depicted by the so-called EFI network, as shown in Figure \ref{EFInetwork}. 
Let $\hat{\btheta}_i:=\hat{g}(y_i,x_i,z_i,\bw)$ denote the DNN prediction function parameterized by the weights $\bw$ in the EFI network, and let 
\begin{equation} \label{thetabareq}
\bar{\btheta}_n:=\frac{1}{n} \sum_{i=1}^n \hat{\btheta}_i=\frac{1}{n} \sum_{i=1}^n \hat{g}(y_i,x_i,z_i,\bw),
\end{equation}
which works as an estimator for the inverse function $G(\bY_n,\bX_n,\bZ_n)$. 
Henceforth, we will call $\bar{\btheta}_n$ an EFI-DNN estimator of $G(\bY_n,\bX_n,\bZ_n)$. 
The EFI network has two output nodes defined, respectively, by 
\begin{equation} \label{outputeq}
\begin{split} 
e_{i1} &:=\|\hat{\btheta}_i-\bar{\btheta}_n\|^2, \\
e_{i2} &:=d(y_i,\tilde{y}_i):=d(y_i,x_i, z_i, \hat{\btheta}_i), \\
\end{split}
\end{equation}
where $\tilde{y}_i=f(x_i,z_i,\hat{\btheta}_i)$,  the function $f(\cdot)$ is as defined in (\ref{dataGeneq}), and $d(\cdot)$ is a function that measures the difference between $y_i$ and $\tilde{y}_i$. With a slight abuse of notation, we rewrite 
$d(y_i,\tilde{y}_i)$ as a function of
$y_i$, $x_i$, $z_i$, and $\hat{\btheta}_i$. For example, for  normal linear/nonlinear regression, we define 
\[
d(y_i,x_i,z_i,\hat{\btheta}_i)=\|y_i-f(x_i,z_i,\hat{\btheta}_i)\|^2.
\]
For logistic regression, we define 
$d(y_i,x_i,z_i,\hat{\btheta}_i)$ via a ReLu function, see Section \ref{EFI:logistic} of the supplement.
  
For the EFI network, we consider $\bw$ as the parameters to estimate, $\bZ_n$ as the latent variable (or missing data) to impute, $(\bX_n,\bY_n)$ as the observed data (or incomplete data), and $(\bX_n,\bY_n,\bZ_n)$ as the complete data. Regarding the EFI network, we have a few further remarks.  

\begin{remark} \label{rem:DNN}
The DNN in the EFI network is a fully-connected feedforward neural network, which maps $(y_i,x_i,z_i)$ to $\hat{\btheta}_i$ for each $i \in \{1,2,\ldots,n\}$. Both the depths and widths of the DNN can increase with the sample size $n$ but under a constraint as given in Assumption \ref{ass9}-(ii-1) (in the supplement). To ensure Assumption \ref{ass:hwang1b}-(ii), the activation function needs to be continuously differentiable and,   
therefore, can be chosen from options like {\it tanh}, softplus or sigmoid. In practice, the ReLU activation function can also be used, as the resulting energy function is non-continuously differentiable at isolated points only. Consequently, for case (b), we will have
(\ref{theoreticalzdist}) holding almost surely, 
as implied by the proof provided in  
\cite{Hwang1980LaplacesMR}; for case (a),   (\ref{theoreticalzdista})  still holds as the continuously differentiability condition is not required. 
\end{remark}

\begin{remark}
To address the potential overfitting issue in the DNN, we treat $\bw$ in a Bayesian approach. We impose a sparse prior on $\bw$, as given in (\ref{mixtureprior}), based on the sparse deep learning theory in \cite{SunSLiang2021}.
However, this Bayesian treatment is optional, as  $\bw$ can still be consistently estimated 
within the frequentist framework when the training sample size $n$ is sufficiently large.
Furthermore, as discussed in Remark \ref{Remprior}, the prior hyperparameters can be entirely determined by the data through cross-validation, aligning the EFI-DNN algorithm with the principle of fiducial inference. 
Sparse learning enables the EFI-DNN algorithm to exhibit robust performance 
across a wide range of DNNs with different depths and widths, provided they possess sufficient capacity to approximate desired inverse functions. \textcolor{black}{Regarding the interpretability of the sparse DNN, we refer to \cite{Liang2018BNN} and \cite{SunSLiang2021}. 
In the context of EFI networks, the sparse DNN provides a parsimonious approximation to the inverse function. If the inverse function is a sparse neural network function, then its structure 
can be consistently recovered (up to  some loss-invariant transformations).}
\end{remark}

\begin{remark}
While the EFI network shares a similar structure with the fiducial autoencoder used in \cite{li2020deep}, the DNNs in the two works are trained in  different ways. In \cite{li2020deep}, the DNN is pre-trained using data simulated from the model with a wide range of parameter values. In the present work, the DNN is trained concurrently with the imputation of latent variables. 
\end{remark}

 Let $\pi(\bw)$ denote the prior density function of $\bw$, and let $\pi(\bY_n,\bZ_n|\bX_n,\bw)$ denote 
 the conditional density function of $(\bY_n,\bZ_n)$ given $(\bX_n,\bw)$.  
 The form of $\pi(\bw)$ will be detailed later; as discussed in Remark \ref{Remprior}, $\pi(\bw)$ should be chosen such that $\bar{\btheta}_n$ forms a consistent estimator for the inverse mapping  $G(\bY_n,\bX_n,\bZ_n)$.
 We propose to estimate $\bw$ by maximizing the posterior distribution 
 $\pi(\bw|\bX_n,\bY_n) \propto  \pi(\bw) \int \pi(\bY_n,\bZ_n|\bX_n,\bw) d\bZ_n$. 
 This can be done by solving the equation
 \begin{equation} \label{solutioneq0}
 \nabla_{\bw} \log \pi(\bw|\bX_n,\bY_n)=0.
 \end{equation}
 Further, by the Bayesian version of Fisher's identity, see Lemma 1 of \cite{SongLiang2020eSGLD},  (\ref{solutioneq0}) can be expressed as  
 \begin{equation} \label{identityeq}
\nabla_{\bw} \log \pi(\bw|\bX_n,\bY_n)=\int  \nabla_{\bw} \log \pi(\bw|\bX_n,\bY_n,\bZ_n) \pi(\bZ_n|\bX_n,\bY_n,\bw) d\bw=0. 
\end{equation} 
 To define  $\pi(\bw|\bX_n,\bY_n,\bZ_n)$ and $\pi(\bZ_n|\bX_n,\bY_n,\bw)$, we first define a scaled energy function for the distribution  $\pi(\bY_n|\bX_n,\bZ_n,\bw)$, up to an additive constant and a multiplicative constant:  
 \begin{equation} \label{energyfunction00}
 \tilde{U}_n(\bZ_n,\bw;\bX_n,\bY_n)=\eta \sum_{i=1}^n \|\hat{\btheta}_i-\bar{\btheta}_n\|^2 + \sum_{i=1}^n 
 d(y_i,x_i,z_i,\hat{\btheta}_i), 
 \end{equation}
where the first term serves as 
a penalty function enforcing $\hat{\btheta}_i$'s to converge to the same value, 
 and $\eta>0$ is a regularization parameter.
This penalty allows us to address possible non-uniqueness  of  
the inverse functions $\hat{\btheta}_i=\hat{g}(y_i,x_i,z_i,\bw)$ for $i=1,2,\ldots,n$. 
Let 
\[
\pi(\bY_n|\bX_n,\bZ_n,\bw) = C e^{-\lambda \tilde{U}_n(\bZ_n,\bw;\bX_n,\bY_n)},
\]
for some constants $C>0$ and $\lambda>0$. 
Then we have the following conditional distributions: 
\begin{equation} \label{eqA12}
\begin{split}
\pi(\bw|\bX_n,\bY_n,\bZ_n) &  
\propto \pi(\bw) e^{- \lambda \tilde{U}_n(\bZ_n,\bw;\bX_n,\bY_n) }, \\
 \pi(\bZ_n|\bX_n,\bY_n,\bw) & 
 \propto \pi_0^{\otimes n}(\bZ_n) e^{- \lambda \tilde{U}_n(\bZ_n,\bw;\bX_n,\bY_n)},
 \end{split} 
\end{equation} 
where $\lambda$ is a tuning  parameter resembling the inverse of the temperature 
in (\ref{z-dist}), and $\pi_0^{\otimes n}(\bZ_n)$ is the marginal distribution of $\bZ_n$ in the space  $\mathbb{R}^n$. 
With respect to the EFI network, we call $\pi(\bY_n|\bX_n,\bZ_n,\bw)$, $\pi(\bw|\bX_n,\bY_n,\bZ_n)$, and $\pi(\bZ_n|\bX_n,\bY_n,\bw)$
 the complete-data likelihood function, the complete-data posterior distribution, and the missing-data predictive distribution, respectively. 


\begin{remark} \label{Remarkeq}
Alternative to (\ref{energyfunction00}), we can define the energy function as
\begin{equation} \label{energyfunction11}
 \tilde{U}_n^{\prime}(\bZ_n,\bw; \bX_n,\bY_n) = \eta \sum_{i=1}^n\| \hat{\btheta}_i- \bar{\btheta}_n \|^2 +  \sum_{i=1}^n d(y_i,x_i,z_i,\bar{\btheta}_n). 
\end{equation}
Without confusion, we will refer to the EFI-DNN algorithm with the energy functions (\ref{energyfunction00}) and (\ref{energyfunction11}) as EFI-a (alternative version) and EFI (default version), respectively, in the remaining of the paper. Compared to (\ref{energyfunction11}), (\ref{energyfunction00}) is more regular, where the fitting errors are assumed to be mutually independent given $\bw$. 
As $\lambda \to \infty$,
EFI-a and EFI are asymptotically equivalent, leading to the same zero-energy set. 
\end{remark}

With the distributions given in (\ref{eqA12}), Eq. (\ref{identityeq}) is now well defined and 
can be solved using an adaptive stochastic gradient MCMC algorithm 
\cite{deng2019adaptive,DongZLiang2022,LiangSLiang2022}. 
The algorithm works by iterating between the following two steps, where $k$ indexes the iterations: 
\begin{itemize} 
\item[(a)] ({\it Latent variable sampling}) Generate $\bZ_n^{(k+1)}$ from a transition kernel induced by a stochastic gradient MCMC algorithm. For example, we can simulate $\bZ_n^{(k+1)}$ 
using the stochastic gradient 
Langevin dynamics (SGLD) algorithm \citep{Welling2011BayesianLV}: 
\begin{equation} \label{Algeq1}
\begin{split}
\bZ_n^{(k+1)} &= \bZ_n^{(k)} +\epsilon_{k+1} \widehat{\nabla}_{\bz_n} \log \pi(\bZ_n^{(k)}|\bX_n,\bY_n,\bw^{(k)}) +\sqrt{2 \tau \epsilon_{k+1}} \be^{(k+1)}, 
\end{split}
\end{equation}
where $\be^{(k+1)} \sim N({\bf 0}, I_{d_{\bz}})$ is a standard Gaussian random vector of dimension $d_{\bz}$, $\epsilon_{k+1}$ is the learning rate,  
$ \widehat{\nabla}_{\bZ_n} \log \pi(\bZ_n^{(k)}|\bX_n,\bY_n,\bw^{(k)})$ 
denotes an unbiased estimator of $\nabla_{\bZ_n} \log \pi(\bZ_n^{(k)}|\bX_n,\bY_n$, $\bw^{(k)})$, 
and $\tau$ is the temperature that is generally set to 1 in simulations. 
 
 \item[(b)] ({\it Parameter updating}) Update the estimate of $\bw$ by stochastic gradient descent (SGD):  
 \begin{equation} \label{Algeq2}
 \bw^{(k+1)}=\bw^{(k)}+ \frac{\gamma_{k+1}}{n} \widehat{\nabla}_{\bw} \log \pi(\bw^{(k)}|\bX_n,\bY_n,\bZ_n^{(k+1)}),
 \end{equation}
 where $\gamma_{k+1}$ denotes the step size of 
 stochastic approximation \citep{robbins1951stochastic}, and 
 $\widehat{\nabla}_{\bw} \log \pi(\bw^{(k)}|\bX_n,\bY_n,\bZ_n^{(k+1)})$ denotes an unbiased estimator of $\nabla_{\bw} \log \pi(\bw^{(k)}|\bX_n,\bY_n,\bZ_n^{(k+1)})$.
 \end{itemize}
 
 The algorithm is referred to as ``adaptive'' as the transition kernel in step (a) changes along with the update of $\bw$.
 Applying the adaptive SGLD algorithm to the EFI network leads to Algorithm \ref{EFIalgorithm}, where the parameter updating step is implemented with mini-batches, and a fiducial sample collection step is added. Note that, given the current estimate of $\bw$,
 the latent variable sampling step can be executed in parallel for each observation $(x_i,y_i)$.
  Therefore, the whole algorithm is scalable  
  with respect to big data.

 \begin{algorithm}[!ht]
 \caption{Adaptive SGLD for Extended Fiducial Inference}
 \label{EFIalgorithm}
 \SetAlgoLined 
 {\bf (i) (Initialization)} Initialize the DNN weights $\bw^{(0)}$ and the latent variable $\bZ_n^{(0)}$. set $M$ as the number of fiducial samples to collect.
Let $\mK$ denote the number iterations to perform in the burn-in period, and let $\mK+M$ be the total number of iterations to perform in a run.  

\For{k=1,2,\ldots,\mbox{$\mK+M$}}{

 {\bf (ii) (Latent variable sampling)} Given $\bw^{(k)}$, simulate $\bZ_n^{(k+1)}$ by the SGLD algorithm: 
  \begin{equation} \label{SGLDtempEq}
\begin{split}
\bZ_n^{(k+1)} &= \bZ_n^{(k)} +\epsilon_{k+1} \nabla_{\bZ_n} \log \pi(\bZ_n^{(k)}|\bZ_n,\bY_n,\bw^{(k)}) +\sqrt{2 \tau \epsilon_{k+1}} \be^{(k+1)}, 
\end{split}
\end{equation}
where $\be^{(k+1)} \sim N(0,I_{d_{\bz}})$, $\epsilon_{k+1}$ is the learning rate, and 
$\tau=1$ is the temperature.

 {\bf (iii) (Parameter updating)} Draw a minibatch $\{(y_1,x_1,z_1^{(k)}),\ldots,(y_m,x_m,z_m^{(k)})\}$ 
 and update the network weights by the SGD algorithm: 
\begin{equation}
\label{sgd_update}
\bw^{(k+1)}=\bw^{(k)} +\gamma_{k+1} \left[ \frac{n}{m} \sum_{i=1}^m \nabla_{\bw} \log \pi(y_i|x_i,z_i^{(k)},\bw^{(k)})+ \nabla_{\bw} \log \pi(\bw^{(k)}) \right],
\end{equation}
where $\gamma_{k+1}$ is the step size, and $\log \pi(y_i|x_i,z_i^{(k)},\bw^{(k)})$ can be appropriately defined according to (\ref{energyfunction00}) or (\ref{energyfunction11}). 
  
{\bf (iv) (Fiducial sample collection)} If $k+1 > \mK$, calculate 
 $\hat{\btheta}_i^{(k+1)}=\hat{g}(y_i,x_i,z_i^{(k+1)},\bw^{(k+1)})$ 
 for each $i\in \{1,2,\ldots,n\}$ and average them to get a fiducial $\bar{\btheta}_n$-sample as calculated in (\ref{thetabareq}). 
}

{\bf (v) (Statistical Inference)} Conducting statistical inference for the model based on the collected fiducial samples. 
\end{algorithm}

\subsection{Convergence Theory of the EFI-DNN Algorithm} \label{theorysect}

To indicate the dependency of $\bw$ on the sample size $n$, we rewrite $\bw$  as $\bw_n$ in this subsection.
\textcolor{black}{We note that the theoretical study is conducted under the assumption 
that the EFI network has been correctly specified such that there exists a sparse solution $\tilde{\bw}_n^*$, at which  $(\bX_n,\bY_n,\bZ_n^*)$ can be generated from the EFI network; specifically, $\bZ_n^* \sim \pi(\bZ|\bX_n,\bY_n,\tilde{\bw}_n^*)$ holds, where $\bZ_n^*$ represents the values of the latent variables realized in the observations.} 
The convergence of the EFI-DNN algorithm is studied in a few steps. 
First, we show in Theorem \ref{thm1} that $\|\bw_n^{(k)} -\bw_n^* \|\stackrel{p}{\to} 0$ as $k\to \infty$, where $\bw_n^*$ is a solution to (\ref{solutioneq0}) and $\stackrel{p}{\to}$ denotes convergence in probability. 
Second, we show in Theorem \ref{thm2} 
that $\bZ_n^{(k)}$ converges weakly to $\pi(\bZ_n|\bX_n,\bY_n,\bw_n^*)$ 
in 2-Wasserstein distance as $k \to \infty$. 
 Third,  we show in Theorem \ref{thm3} and the followed discussions that  with 
an appropriate choice of the prior distribution $\pi(\bw_n)$ and as $n\to \infty$ and $\lambda\to \infty$, $\hat{g}(y_i,x_i,z_i,\bw_n^*)$ constitutes a consistent estimator of $\btheta^*$ and, 
subsequently, the EFI-DNN estimator 
\begin{equation} \label{targetest}
\bar{\btheta}_n^*:=\frac{1}{n} \sum_{i=1}^n \hat{g}(y_i,x_i,z_i,\bw_n^*),
\end{equation}
constitutes a consistent estimator for the inverse mapping $G(\bY_n,\bX_n,\bZ_n)$.
By summarizing the three theorems, we conclude that the EFI-DNN algorithm 
 leads to valid uncertainty quantification for $\btheta$.  
Finally, we show that if $\bar{\btheta}_n^*$ is consistent, $\pi(\bZ|\bX_n,\bY_n,\bw_n^*)$
is reduced to the extended fiducial distribution of $\bZ_n$ as defined in Section \ref{EFDsection}.

\subsubsection{Convergence of Algorithm \ref{EFIalgorithm}}

\begin{theorem} \label{thm1} Suppose Assumptions \ref{ass1}-\ref{ass5} (in the supplement) hold. If we set the learning rate sequence $\{\epsilon_k: k=1,2,\ldots\}$ and the step size sequence $\{\gamma_k: k=1,2,\ldots\}$ in the form $\epsilon_k=\frac{C_{\epsilon}}{c_{\epsilon}+k^{\alpha}}$ and 
$\gamma_k=\frac{C_{\gamma}}{c_{\gamma}+k^{\beta}}$ for some constants $C_{\epsilon}>0$, $c_{\epsilon}>0$, $C_{\gamma}>0$ and $c_{\gamma}>0$, $\alpha,\beta\in (0,1]$, and $\beta \leq \alpha \leq \min\{1, 2 \beta\}$, then there exists a root $\bw_n^* \in \{\bw: \nabla_{\bw} \log \pi(\bw|\bX_n,\bY_n)=0\}$ such that 
\[
\mathbb{E} \|\bw_n^{(k)}-\bw_n^*\|^2 \leq \xi \gamma_k, \quad k \geq k_0,
\]
for some constant $\xi>0$ and iteration number $k_0>0$.
\end{theorem} 

Since the adaptive SGLD algorithm can be viewed as a special case of the adaptive pre-conditioned SGLD algorithm \citep{DongZLiang2022}, Theorem \ref{thm1} can be proved by following the proof of 
Theorem A.1 of \cite{DongZLiang2022} with minor  modifications. Regarding the convergence rate of the algorithm, \cite{DongZLiang2022} provides an explicit form of $\xi$. To make the presentation concise, we omit it in the paper.

Let $\pi^*=\pi(\bZ_n|\bX_n,\bY_n,\bw_n^*)$, let $T_k = \sum_{i=0}^{k-1}\epsilon_{i+1}$,  and let $\mu_{T_k}$ denote the probability law of $\bZ_n^{(k)}$.
   Theorem \ref{thm2} establishes convergence of $\mu_{T_k}$ in 2-Wasserstein distance. 

\begin{theorem} \label{thm2}  
Suppose Assumptions \ref{ass1}-\ref{ass6} (in the supplement) hold, and 
$\{\epsilon_k\}$ and $\{\gamma_k\}$ are set as in Theorem \ref{thm1}. Then,   
for any $k \in \mathbb{N}$, 
\[  
\mathbb{W}_2(\mu_{T_k}, \pi^*) \leq (\hat{C}_0 \delta_g^{1/4}+\tilde{C}_1 \gamma_1^{1/4})T_k + \hat{C}_2 e^{-T_k/c_{LS}},
\]
for some positive constants $\hat{C}_0$, $\hat{C}_1$, and $\hat{C}_2$, where $\mathbb{W}_2(\cdot,\cdot)$ denotes the 2-Wasserstein distance, $c_{LS}$ denotes the logarithmic Sobolev constant of $\pi^*$, and $\delta_g$ is a coefficient 
as defined in Assumption \ref{ass3} and reflects the variation of the stochastic gradient $\widehat{\nabla}_{\bZ_n} \log \pi(\bZ_n^{(k)}|\bX_n,\bY_n,\bw^{(k)})$.
\end{theorem}

We use the full data in the sampling step such that $\delta_g=0$, choose $\alpha \in (0,1]$, 
and choose $\gamma_1 \prec \frac{1}{T_k^4}$ for any $T_k$, which ensures $\mathbb{W}_2(\mu_{T_k}, \pi^*) \to 0$ as $k\to \infty$.

\subsubsection{On the Consistency of $\bar{\btheta}_n^*$}

 Let $\mathcal{W}_n \subset \mathbb{R}^{d_w}$ denote the space of $\bw_n$, where $d_w$ denotes the dimension of $\bw_n$. Let each component of $\bw_n$ be subject to a truncated mixture Gaussian distribution with the density function given by
 \begin{equation} \label{mixtureprior}
 \pi(w_n^{(i)}) =  \rho_n f(w_n^{(i)};0,\sigma_{1,n}^2)+(1-\rho_n) f(w_n^{(i)};0,\sigma_{0,n}^2), \quad w_n^{(i)} \in \mathcal{W}_n^{(i)}, \quad i=1,2,\ldots,d_w,
 \end{equation}
 where  $\mathcal{W}_n^{(i)} \subset \mathbb{R}$ denotes the $i$th component of $\mathcal{W}_n$, 
 $\rho_n$ is the mixture proportion,  $\sigma_{0,n}<\sigma_{1,n}$,  the density function of each component of the mixture distribution is given by 
 \[
 f(w;0,\sigma^2)=\phi(w/\sigma)/\int_{\mathcal{W}_n^{(i)}} [\rho_n \phi(w/\sigma_{1,n})+(1-\rho_n) \phi(w/\sigma_{0,n})] dw, 
 \]
 and $\phi(\cdot)$ denotes the standard Gaussian density function.
 All components of $\bw_n$ are {\it a priori} independent. 
 In our experience, the weights of DNNs often cluster around a small subset near the origin ${\bf 0}$ in the space $\mathbb{R}^{d_w}$. Therefore, it is reasonable to constrain $\mathcal{W}_n$ to a compact set, as stipulated in Assumption \ref{ass7}.
 
 To establish the consistency of $\bar{\btheta}_n^*$, we first define 
 \begin{equation} \label{Q1eq}
 \widehat{\mG}(\bw_n|\tilde{\bw}_n^*):=\frac{1}{n} \log\pi(\bY_n,\bZ_n^*|\bX_n,\bw_n)+\frac{1}{n} \log \pi(\bw_n),
 \end{equation}
 where $\bZ_n^* \sim \pi(\bZ|\bY_n,\bX_n,\tilde{\bw}_n^*)$ as defined previously.  
Therefore,
 \[
 \hat{\bw}_n^*:=\arg\max_{\bw_n\in \mathcal{W}_n} \widehat{\mG}(\bw_n|\tilde{\bw}_n^*),
 \]
 is also the global maximizer of  the log-posterior $\log \pi(\bw_n|\bX_n,\bY_n,\bZ_n^*)$, given the pseudo-complete data. Further, we define 
 \begin{equation} \label{Q2eq}
 \begin{split}
 \widetilde{\mG}(\bw_n|\tilde{\bw}_n^*)& :=\frac{1}{n}\int \log \pi(\bY_n,\bZ_n^*|\bX_n,\bw_n) d\pi(\bZ_n^*|\bX_n,\bY_n,\tilde{\bw}_n^*) +\frac{1}{n} \log\pi(\bw_n) \\ 
 &=  \frac{1}{n} \Big\{ \log\pi(\bw_n|\bX_n,\bY_n) - \int \log \frac{\pi(\bZ_n^*|\bX_n,\bY_n,\tilde{\bw}_n^*)}{ \pi(\bZ_n^*|\bX_n,\bY_n,\bw_n)} d \pi(\bZ_n^*|\bX_n,\bY_n,\tilde{\bw}_n^*) \\
 & + \int \log \pi(\bZ_n^*|\bX_n,\bY_n,\tilde{\bw}_n^*) d \pi(\bZ_n^*|\bX_n,\bY_n,\tilde{\bw}_n^*)+c \Big\},\\
 \end{split}
 \end{equation}
 where $c=\log\int_{\mathcal{W}_n} \pi(\bY_n|\bX_n,\bw_n) \pi(\bw_n)d\bw_n$ is the log-normalizing constant of the posterior $\pi(\bw_n|\bX_n,\bY_n)$.
 In the derivation of (\ref{Q2eq}), $\bX_n$ can be ignored for simplicity as it is constant. For simplicity of notation, we let 
 $D_{KL}(\bw_n)= \int \log \frac{\pi(\bZ_n^*|\bX_n,\bY_n,\tilde{\bw}_n^*)}{ \pi(\bZ_n^*|\bX_n,\bY_n,\bw_n)} d \pi(\bZ_n^*|\bX_n,\bY_n,\tilde{\bw}_n^*)$ be the Kullback-Leibler divergence between $\pi(\bZ_n^*|\bX_n,\bY_n,\tilde{\bw}_n^*)$ and 
 $\pi(\bZ_n^*|\bX_n,\bY_n,\bw_n)$ in what follows. 
 

 Let $Q^*(\bw_n)=\mathbb{E}(\log \pi(Y,Z|X,\bw_n))+\frac{1}{n} \log \pi(\bw_n)$, where the expectation is taken with respect to the joint distribution of $(X,Y,Z)$. 
 Further, by Assumption \ref{ass7} and the weak law of large numbers, 
\begin{equation}\label{eq:sameloss2}
    \frac{1}{n}\log\pi(\bw_n|\bX_n,\bY_n,\bZ_n)-Q^*(\bw_n)\overset{p}{\rightarrow} 0,
\end{equation}
holds uniformly over the parameter space $\mathcal{W}_n$. Assumption \ref{ass8} restricts the shape of $Q^*(\bw_n)$ around the global maximizer, which cannot be discontinuous or too flat. Given nonidentifiability of the neural network model, see e.g. \cite{SunSLiang2021}, we have implicitly assumed that each $\bw_n$  is unique up to the loss-invariant transformations, e.g., reordering the hidden neurons of the same hidden layer and simultaneously changing the signs of some weights and biases. The same assumption has often been used in theoretical studies of neural networks, see e.g. \cite{Liang2018BNN} and  \cite{SunSLiang2021}. 

On the other hand,
by Theorem 1 of \cite{liang2018imputation}, under some regularity conditions 
we have 
 \begin{equation} \label{QQeq}
 \sup_{\bw_n\in \mathcal{W}_n}\left| \widehat{\mG}(\bw_n|\tilde{\bw}_n^*)-\widetilde{\mG}(\bw_n|\tilde{\bw}_n^*) \right| \stackrel{p}{\to} 0, \quad \mbox{as $n\to \infty$}.
 \end{equation} 
Putting (\ref{eq:sameloss2}) and (\ref{QQeq}) together and assuming 
that $Q^*(\bw_n)$ satisfies Assumption \ref{ass8}, then   
we have the following lemma, whose proof is given in the supplement. 

{\color{black}
 \begin{lemma} \label{lemma:equivalent} Suppose Assumptions \ref{ass7}-\ref{ass8} (in the supplement) hold, and
 $\pi(\bY_n,\bZ_n|\bX_n,\bw_n)$ is continuous in $\bw_n$. 
 If $\hat{\bw}_n^*$ is unique, then $\bw_n^*$ that 
 maximizes $\pi(\bw_n|\bX_n,\bY_n)$ and  minimizes $D_{KL}(\bw_n)$ is unique and, subsequently, 
$\|\hat{\bw}_n^*-\bw_n^*\| \stackrel{p}{\to} 0$ holds as $n \to \infty$.
 \end{lemma}
}


\textcolor{black}{
The uniqueness of $\hat{\bw}_n^*$, up to some loss-invariant transformations, can be ensured by the consistency of the posterior 
$\pi(\bw_n|\bY_n,\bX_n,\bZ_n)$ as established in Theorem \ref{thm3} with an appropriate prior $\pi(\bw_n)$. 
The condition minimizing $D_{KL}(\bw_n)$ is generally 
implied by $\tilde{U}_n(\bZ_n,\bw_n;\bX_n,\bY_n)=0$ provided the consistency 
of $\bar{\btheta}_n^*$, and the convergence of $\bw_n^*$ to a maximum of $\pi(\bw_n|\bX_n,\bY_n)$ is generally implied 
by the Monte Carlo nature of Algorithm \ref{EFIalgorithm}. 
Therefore, by Theorem \ref{thm1}, if $\bw_n^{(k)}$ converges and 
$\tilde{U}_n(\bZ_n,\bw_n^{(k)};\bX_n,\bY_n)$ converges to 0, we would 
 have $\|\hat{\bw}_n^*-\bw_n^*\| \stackrel{p}{\to} 0$, provided that  
the prior has been appropriately chosen such that the posterior consistency holds and $\hat{g}(y_i,x_i,z_i,\hat{\bw}_n^*)$ constitutes a  consistent estimator 
of $\btheta^*$.}

Suppose our choice of the prior $\pi(\bw_n)$ ensures that  
the posterior consistency holds and 
$\hat{g}(y_i,x_i,z_i,\hat{\bw}_n^*)$ is consistent for $\btheta^*$.
By Lemma \ref{lemma:equivalent}, $\hat{g}(y_i,x_i,z_i,\bw_n^*)$ 
would also be consistent for $\btheta^*$, provided  
$\hat{g}(\cdot)$ is continuous.
The posterior consistency and the consistency of $\hat{g}(y_i,x_i,z_i,\hat{\bw}_n^*)$ can be proved based on the results of \cite{SunSLiang2021}.  This is summarized in Theorem \ref{thm3}, whose proof
can be found in the supplement. Note that working on $\hat{\bw}_n^*$ is simpler than working on $\bw_n^*$, as the former is based on the complete data.

\begin{theorem} \label{thm3} 
 Suppose that $\pi(\bw_n)$ is a truncated mixture Gaussian prior distribution 
  as specified in (\ref{mixtureprior}) and  Assumptions \ref{ass1}-\ref{ass10} (in the supplement) hold. Then, under the limit  $\lambda \to \infty$, 
  the posterior consistency holds for   $\pi(\bw_n|\bY_n,\bX_n,\bZ_n)$ 
  and the inverse mapping estimator $\hat{g}(\cdot)$ (with  either the energy function (\ref{energyfunction00}) or (\ref{energyfunction11})) constitutes a consistent estimator for the model parameters, i.e., 
\[ 
 \|\hat{g}(y,x,z,\bw_n^*) -\btheta^*\| \stackrel{p}{\to} 0, \quad \mbox{as $n\to \infty$},
 \]
 where $\btheta^*$ denotes the fixed unknown parameter values, and $(y,x,z)$ denotes a generic element of $(\bY_n,\bX_n,\bZ_n)$. 
\end{theorem}

Following from Theorem \ref{thm3}, we immediately have 
$\| \frac{1}{n} \sum_{i=1}^n \hat{g}(y_i,x_i, z_i, \bw_n^*)-\btheta^*\| \stackrel{p}{\to} 0$ as $n\to \infty$. As a slight relaxation of Assumption \ref{ass:existence}, we can write (\ref{Inveq}) as 
\begin{equation} \label{ass*}
\btheta^*=\lim_{n\to \infty} G(\bY_n,\bX_n,\bZ_n),
\end{equation}
where $\bZ_n$ is assumed to be known.
For example,  consider the normal mean model
\begin{equation} \label{normalmean}
y_i=\btheta+z_i, \quad z_i \sim N(0,1), \quad i=1,2,\ldots,n,
\end{equation}
for which $G(\bY_n,\bZ_n)=\sum_{i=1}^n (y_i- z_i)/n \equiv \btheta^*$   and, therefore,  (\ref{ass*}) holds trivially. By combining the above two limits, we have
\begin{equation} \label{rootn-eq}
\left\| \frac{1}{n} \sum_{i=1}^n \hat{g}(y_i,x_i, z_i, \bw_n^*)-G(\bY_n,\bX_n,\bZ_n)\right\| \stackrel{p}{\to} 0, \quad \mbox{as $n\to \infty$},
\end{equation}
i.e., the EFI-DNN estimator $\bar{\btheta}_n^*:=\frac{1}{n} \sum_{i=1}^n \hat{g}(y_i,x_i, z_i, \bw_n^*)$ is consistent for the inverse mapping $G(\bY_n,\bX_n,\bZ_n)$.  
Further, by Slutsky's theorem, the uncertainty of $\bZ_n$ can be propagated to $\btheta$ via
the EFI-DNN estimator. Therefore, 
the confidence distribution of $\btheta$ can be approximated by 
\begin{equation} \label{CDestimator}
\tilde{\mu}_n(d\btheta)= \frac{1}{\mM} \sum_{k=1}^{\mM} \delta_{\bar{\btheta}_n^{*,k}} (d\btheta), \quad \mbox{as $\mM\to \infty$},
\end{equation}
where $\delta_a$ stands for the Dirac measure at a given point $a$, 
$\bar{\btheta}_n^{*,k}:= \frac{1}{n} \sum_{i=1}^n \hat{g}(x_i,y_i, z_{i}^{*,k}, \bw_n^*)$, and 
$\bZ_n^{*,k}:=(z_1^{*,k},z_2^{*,k},\ldots,z_n^{*,k})$ for $k=1,2,\ldots, \mM$  denote $\mM$ random draws from the distribution 
$\pi(\bZ_n|\bX_n,\bY_n,\bw_n^*)$ under the limit setting of $\lambda$.

In this paper, although we set both the learning rate and step size sequences to decay with iterations, for which we particularly set $0.5<\beta \leq \alpha<1$, we can still treat $(\bw_n^{(k)},\bz_n^{(k)})$ approximately equally weighted by Theorem 2 of \cite{SongLiang2020eSGLD} and some classical results of stochastic approximation MCMC (see e.g., Theorem 3.3 of \cite{LiangCL2014}). That is, we can approximate 
the confidence distribution of $\btheta$ by 
\begin{equation} \label{CDestimator2}
\hat{\mu}_n(d\btheta)= \frac{1}{\mM} \sum_{k=1}^{\mM} \delta_{\bar{\btheta}_n^k} (d\btheta), \quad \mbox{as $\mM \to \infty$},
\end{equation}
where $\bar{\btheta}_n^k:= \frac{1}{n} \sum_{i=1}^n \hat{g}(x_i,y_i, z_{i}^{k}, \bw_n^{(k)})$, 
$\bZ_n^{(k)}:=(z_1^{k},z_2^{k},\ldots,z_n^{k})$, and $(\bZ_n^{(k)},\bw_n^{(k)})$ denotes 
the sample and parameter estimate produced by Algorithm 
\ref{EFIalgorithm} at iteration $k$.  Some weighted estimation schemes, see e.g. \cite{Fluctuation2016}, 
also work, but involve extra computation.

 \begin{remark} \label{Remprior} 
 To obtain a consistent EFI-DNN estimator for the inverse mapping, we impose a truncated mixture Gaussian prior (\ref{mixtureprior}) on $\bw_n$.  It is worth noting that the hyperparameters of the prior distribution can be entirely determined from the data. Specifically, we can employ cross-validation to determine their values while constraining their orders to meet Assumption \ref{ass9}-(iv). \textcolor{black}{We refer to \cite{Yang2007CV} for the setup of the cross-validation procedure which, together with the sparse DNN approximation theory established above, ensures consistency of 
 the inverse mapping estimator.}
 This consistency property significantly mitigates the impact 
 of the prior distribution on downstream inference, aligning the EFI-DNN algorithm with the principle of fiducial inference. When the sample size $n$ is 
 much larger than the dimension of $\bw$, we can treat $\bw$ in a frequentist way. Mathematically, this is equivalent to setting $\pi(\bw^{(k)}) \propto 1$ in (\ref{sgd_update}) when running Algorithm \ref{EFIalgorithm}.
 \end{remark}

\subsubsection{On the Property of $\pi(\bz|\bY_n,\bX_n,\bw_n^*)$}

  We are now to study the property of $\pi(\bz|\bY_n,\bX_n,\bw_n^*)$.
  Consider the energy function defined in (\ref{energyfunction00}) again. For convenience, we rewrite it as
   \[
   \check{U}_{n}(\bz)= \eta \| \hat{g}(y_i,x_i,z_i,\bw_n^*) -
   \frac{1}{n} \sum_{i=1}^n \hat{g}(y_i,x_i,z_i,\bw_n^*) \|^2 
   + \sum_{i=1}^n 
   d(y_i, x_i, z_i, \hat{g}(y_i,x_i,z_i,\bw_n^*)),
   \]
   where we replace $\hat{\btheta}_i$'s and $\bar{\btheta}_n$ with their DNN expressions. Define 
   \begin{equation} \label{Z-set}
   \mathcal{Z}_{\check{U}_{n}}=\left\{\bz\in \mathbb{R}^n: \check{U}_{n}(\bz)=0 \right\}.
   \end{equation}
   Let $\Pi_n$ denote a probability measure on $(\mathbb{R}^n, \mathscr{R})$, where $\mathscr{R}$ is the Borel $\sigma$-algebra, 
   and let $\pi_0^{\otimes n}$ be the corresponding density function. 
   Further, we rewrite  $\pi(\bZ_n|\bY_n,\bX_n,\bw_n^*)$ 
   as the following:   
   \begin{equation} \label{EFDlimitZ}
    p_{n,\lambda}(\bz|\bX_n,\bY_n) \propto \pi_0^{\otimes n}(\bz)
    e^{-\lambda \check{U}_{n}(\bz)}.
   \end{equation}
   A direct application of the theory in 
   \cite{Hwang1980LaplacesMR} to (\ref{EFDlimitZ}) leads to the following lemma, for 
   which  Assumptions \ref{ass:zero}-\ref{ass:invariance} will be be justified in Remark  \ref{EFDlimitdist}. 
   
 \begin{lemma} (Proposition 2.2 and Theorem 3.1 of \cite{Hwang1980LaplacesMR}) \label{lemma:Hwang} Suppose that the EFI network, the energy function $\check{U}_n(\bz)$, the probability measure $\Pi_n$, and the zero-energy set $\mathcal{Z}_{\check{U}_n}$ satisfy Assumptions \ref{ass:zero}-\ref{ass:hwang1b}.
 \begin{itemize}
 \item[(a)] If $\Pi_n(\mathcal{Z}_{\check{U}_n})>0$, then 
 $\lim_{\lambda\to \infty} p_{n,\lambda}(\bz|\bX_n,\bY_n)$ is given by 
 \begin{equation} \label{FDlimita}
 \frac{P_n^*(\bz|\bX_n,\bY_n)}{d\bz}=\frac{1}{\Pi_n( \mathcal{Z}_{\check{U}_n})} \pi_0^{\otimes n}(\bz), 
 \quad \bz \in \mathcal{Z}_{\check{U}_n}. 
 \end{equation}
  \item[(b)] If $\Pi_n(\mathcal{Z}_{\check{U}_n})=0$, 
 then  $\lim_{\lambda\to \infty} p_{n,\lambda}(\bz|\bX_n,\bY_n)$ is given by 
  \begin{equation} \label{FDlimitb}
   \frac{P_n^*(\bz|\bX_n,\bY_n)}{d\nu}= \frac{\pi_0^{\otimes n}(\bz) \left( {\rm det} \nabla_{\bt}^2 \check{U}_{n}(\bz)) (\bz) \right)^{-1/2}}{\int_{\mathcal{Z}_{\check{U}_{n}}} \pi_0^{\otimes n}(\bz) \left( 
   {\rm det} \nabla_{\bt}^2 \check{U}_{n}(\bz)) (\bz) \right)^{-1/2} d \nu }, \quad \bz \in \mathcal{Z}_{\check{U}_n},
  \end{equation}
   where $\mathscr{\nu}$ is the sum of intrinsic measures on 
   the  $p$-dimensional manifold in $\mathcal{Z}_{\check{U}_n}$.
  \end{itemize}
  \end{lemma}

   \begin{remark}  \label{EFDlimitdist}
 The conditions specified in Assumptions \ref{ass:zero}-\ref{ass:hwang1b} are readily met by the EFI network. The existence of the minimum $\min_{\bz} \check{U}_n(\bz)=0$ is asymptotically guaranteed by the consistency of $\hat{g}(y_i,x_i,z_i,\bw_n^*)$. 
 In particular, we have $\check{U}_n(\bZ_n^*) \stackrel{p}{\to} 0$ as $n \to \infty$.
    The condition $\Pi_n(\mathcal{Z}_{\check{U}_n})>0$ is  satisfied for logistic regression as discussed in Section \ref{EFI:logistic} of the supplement. While  the condition $\Pi_n(\mathcal{Z}_{\check{U}_n})=0$ is naturally satisfied for normal linear/nonlinear regression problems, as $\mathcal{Z}_{\check{U}_n}$ forms a manifold in $\mathbb{R}^n$ in this case. In the model (\ref{modeleq}), 
 if the function $f$ satisfies the continuity condition as required in 
 Assumption \ref{ass:hwang1b}-(ii), we can ensure that the EFI network also satisfies it by employing appropriate activation functions, such as sigmoid, {\it tanh} and softplus. The other conditions are standard and generally hold.
 \end{remark}

 \begin{remark} 
 Lemma \ref{lemma:Hwang} implies that the choice of $\eta$ is not critical for the convergence of the EFI-DNN algorithm, as long as $\lambda \to \infty$. Specifically, different choices of $\eta$ will result in the same zero-energy set as $\lambda \to \infty$. 
 In practice, to enhance the convergence of the EFI-DNN estimator to the desired inverse function, one can set $\eta$ to a moderate value such as 2, 5, or 10, and set $\lambda$ to be reasonably large. \textcolor{black}{Recall that 
 $\eta$ represents a regularization parameter as defined in (\ref{energyfunction00}).}  
 An appropriate value of $\lambda$ can be determined by gradually increasing it until the resulting confidence intervals of the model parameters cease to shrink.
 \end{remark}


\textcolor{black}{In summary, we have developed a valid algorithm for conducting fiducial inference for general statistical models by leveraging a sparse DNN for the inverse function approximation. 
The EFI-DNN algorithm is computationally efficient. When simulating the latent variables, it essentially samples from (\ref{z-dist}) with a small value of $\epsilon$ rather than directly  from the limiting distribution (\ref{theoreticalzdist}). This circumvents the need to compute the determinant ${\rm det}(\nabla^2_{\bt} U_n(\bz))$, thereby significantly enhancing computational efficiency. On the other hand, since the algorithm
is designed  to sample from the limiting 
distribution of (\ref{z-dist}), it can be applied to models 
with any type of noise, whether additive or non-additive. 
Furthermore, thanks to the universal approximation capability of DNNs, the EFI-DNN algorithm is highly versatile and can be applied to 
statistical models of various complexities.  }

 \subsection{Some Variants of the EFI-DNN Algorithm}
 
 In Algorithm \ref{EFIalgorithm}, the latent variable sampling step is performed using a SGLD algorithm. This can be replaced with an advanced stochastic gradient MCMC algorithm, such as stochastic gradient Hamiltonian Monte Carlo (SGHMC) \citep{SGHMC2014}, 
 momentum SGLD \citep{kim2020stochastic}, 
 or preconditioned SGLD \citep{pSGLD}. 
 The convergence of adaptive SGHMC has been studied in \cite{LiangSLiang2022}, 
 where similar theoretical results to Theorem \ref{thm1} and Theorem \ref{thm2} were achieved.
 Compared to SGLD, SGHMC includes an extra momentum term, which 
 enables faster exploration of the sample space \citep{Li2019StochasticGH}.

 Other than adaptive SGHMC, we also recommend replacing SGLD   
 with tempering SGLD in Algorithm \ref{EFIalgorithm}. 
 In this tempering algorithm, the temperature $\tau$ in (\ref{SGLDtempEq}) is replaced by a decreasing sequence ${\tau_k}$ that converges to 1 
 along with iterations. Such a tempering algorithm is particularly useful for outlier 
 detection problems, as illustrated in Section \ref{outliersection}. With the tempering technique, random errors of large magnitudes can be easily drawn for some observations, accelerating the convergence of the simulation. 

Similar to the tempering technique discussed above, using an increasing sequence of $\{\lambda_k\}$  that converges to a target value along with iterations can also improve the convergence of the simulation. As $\lambda_k$ increases, the latent variable samples gradually shift toward the set $\mathcal{Z}_{\check{U}_n}$. In this setup, Algorithm \ref{EFIalgorithm} possesses a dual adaptive mechanism, adapting both the values of $\lambda_k$ and $\bw^{(k)}$. The convergence properties of such an algorithm will be investigated in future work, following a framework similar to \cite{LiangCL2014}.

\section{Illustrative Examples} 

\subsection{Linear Regression} \label{OLSsection}
We begin by considering a linear regression model given by
\begin{equation} \label{LinearEx1}
y_i=x_i^T \btheta+ \sigma z_i,  \quad i=1,2,\ldots,n, 
\end{equation}
where $z_i \sim N(0,1)$, $x_i=(x_{i,0},\ldots,x_{i,9})^T$, $x_{i,0}=1$, $x_{i,k}\sim N(0,1)$ for $k=1,\ldots,9$,  $\sigma=1$, and the regression coefficient $\btheta=(\theta_0,\theta_1,\ldots,\theta_9)^T=(1,1,1,1,1,0,0,0,0,0)^T$. 
For convenience, we refer to $(\theta_0,\theta_1,\ldots,\theta_4)$ as signal parameters and $(\theta_5,\theta_6,\ldots,\theta_9)$ as noise parameters. We simulated 100 datasets from this model, each with a sample size of $n=500$. 

EFI-a and EFI were applied to this example with $\sigma$ assumed to be known.  \textcolor{black}{For EFI, we have also tried different activation functions, including ReLU, softplus, tanh, and sigmoid.}
Refer to the supplement for the settings of the experiment. 
The numerical results are summarized in Table \ref{Table:LinearComparison}. 
Figure \ref{Figure:LinearComparison} illustrates the concept of EFI. The left plot displays a scatter plot of $\bZ_n$ versus $\hat{\bZ}_n$, where $\bZ_n$ represents the true random errors realized for the observations and $\hat{\bZ}_n$ represents a set of random errors imputed by EFI. The scatter plot highlights the presence of uncertainty in the random errors 
 contained in the data.  
 According to the theory of EFI, the uncertainty in $\bZ_n$ propagates to $\btheta$, giving rise to uncertainty in $\btheta$. 
 The middle plot is a quantile-quantile (Q-Q) plot for $\bZ_n$ and $\hat{\bZ}_n$,  indicating 
 that they follow the same distribution. 
 The right plot compares the confidence intervals of $\beta_1$ produced by 
 EFI and the OLS method. 
 For this dataset, the two methods produced nearly identical confidence intervals 
 for $\beta_1$. This complies with our theoretical result presented 
 in Example 1 of Section \ref{EFDsection}.

\begin{figure}[!ht]
    \begin{center}
    \begin{tabular}{c}
    \includegraphics[width=0.75\textwidth]{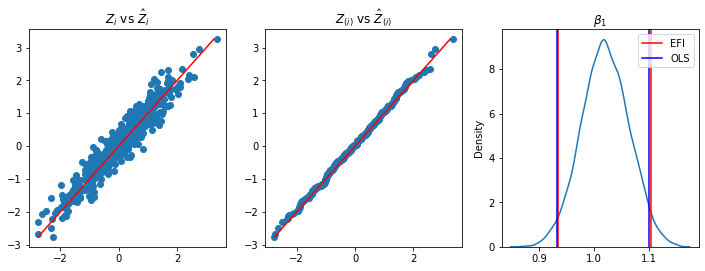} 
    \end{tabular} 
    \caption{ \label{Figure:LinearComparison} Results of EFI (with the ReLU activation function) for one dataset simulated from 
    (\ref{LinearEx1}) with $n=500$: (left) scatter plot of $\hat{\bz}_n$ ($y$-axis) versus $\bz_n$ 
    ($x$-axis), (middle) Q-Q plot of $\hat{\bz}_n$  and $\bz_n$, 
    (right) confidence intervals of $\beta_1$ produced by EFI and OLS. }
    \end{center}
\end{figure}

For comparison, we have applied OLS and GFI to this example. The OLS method is simple, whose implementation is available in many statistical packages such as {\it R Studio}.  
There are two ways to implement GFI as described in Section \ref{conceptsection}. 
One is to use the acceptance-rejection procedure as described in Section 2. However, due to its importance sampling nature, this procedure becomes highly inefficient 
for the problems with a large value of $n$. 
For instance, in this example, we attempted to generate 50,000,000 samples of $\bZ_n$ from $N(0,I_n)$ for $n=500$, but none of them was accepted. The other way 
involves direct simulations from the limiting distribution as given in Theorem 1 of \cite{hannig2016gfi}.  
For this example, the limiting distribution is given by 
$\btheta \sim N(\bX_n^T \bX_n)^{-1} \bX_n^T \bY_n, \sigma^2(\bX_n^T \bX_n)^{-1})$, where $\bX_n$ represents the design matrix of (\ref{LinearEx1}) and $\bY_n=(y_1,y_2,\ldots,y_n)^T$, which is identical to the extended fiducial distribution. 

\begin{table}[htbp]
\caption{Statistical inference results for the model (\ref{LinearEx1}) with known $\sigma^2$, where ``Coverage'' refers to the averaged coverage rate over 100 datasets and respective parameters, and ``CI-width'' refers to the average width of respective confidence intervals.}
\label{Table:LinearComparison}
\begin{center}
\begin{tabular}{ccccccc} \toprule
&  & \multicolumn{2}{c}{Signal parameters} & & \multicolumn{2}{c}{Noise parameters}  \\  \cline{3-4} \cline{6-7}  
 Method & Activation & Coverage rate &  CI-width & & Coverage rate & CI-width  \\ \midrule
       OLS & ---  & 0.95 & 0.177 & & 0.956 & 0.177 \\  
           GFI & ---  &    0.95 & 0.177  &  & 0.952 & 0.177  \\
          EFI-a & ReLU   & 0.948 & 0.176 &  & 0.95  & 0.171    \\  
          EFI  & Sigmoid  & 0.948 & 0.176 & & 0.956 & 0.176     \\
          EFI  & Tanh & 0.948 & 0.176   & & 0.956 & 0.176   \\
          EFI   & Softplus &  0.95 & 0.177 & & 0.95 & 0.176 \\
          EFI   & ReLU  & 0.95  & 0.176 &  & 0.95  & 0.176  \\   
           \bottomrule
\end{tabular}
\end{center}
\end{table}

Table \ref{Table:LinearComparison} shows that both versions of EFI work very well for this example. 
In our experience, EFI-a often requires a larger value of $\eta$ to control the variability of $\hat{\btheta_i}$ than EFI. Additionally, EFI tends to be more robust to parameter settings than EFI-a, as it directly use the average $\bar{\btheta}_n$  
in generating the fitted values $\tilde{y}_i$'s. 
Since EFI and EFI-a are asymptotically equivalent, as mentioned in Remark \ref{Remarkeq}, we will only present the results of EFI in the following analysis. \textcolor{black}{Furthermore, Table \ref{Table:LinearComparison} 
shows that EFI is robust to the choice of the activation 
functions. Note that each of these activation 
functions is Lipschitz continuous (with a 
Lipschitz constant of 1) and can result in a consistent estimator 
for the inverse function. }

For a comprehensive treatment of the model (\ref{LinearEx1}), we applied EFI to the simulated datasets with $\sigma^2$ assumed to be unknown. The results are summarized in 
Table \ref{Table:Linear_unknownVariance},  which demonstrates the validity of EFI for performing  
statistical inference on the model. \textcolor{black}{In this case, we experimented different settings 
of $\eta$ and $\lambda$,  and EFI 
proved to be robust to these settings.  }

\begin{table}[!ht]
\caption{Statistical inference results for the model (\ref{LinearEx1}) with unknown $\sigma^2$, where ``Coverage'' refers to the averaged coverage rate over 100 datasets and respective parameters, and `CI-width'' refers to the average width of respective confidence intervals.} 
\label{Table:Linear_unknownVariance}
\begin{center}
\begin{tabular}{cccccccccc} \toprule
        &  &  \multicolumn{2}{c}{Signal parameters} & & \multicolumn{2}{c}{Noise parameters} &  &  \multicolumn{2}{c}{Variance ($\sigma^2$)}  \\ \cline{3-4} \cline{6-7} \cline{9-10}   
 Method &  $(\eta,\lambda)$ & Coverage & CI-width & & Coverage & CI-width & & Coverage & CI-width   \\ \midrule
  OLS  & --- & 0.948 & 0.176 & & 0.948 & 0.175 & & 0.95 & 0.252 \\  
GFI  & --- & 0.952 & 0.177 & & 0.946 & 0.176 & & 0.95 & 0.251 \\  
EFI  & (2,30)  & 0.95 & 0.180 &  & 0.948 & 0.178 & & 0.95 & 0.255 \\
EFI  & (2,40) & 0.952 & 0.179 &  & 0.954 & 0.179 & & 0.95 & 0.252 \\
EFI & (2,50) & 0.95 & 0.178 &  & 0.946 & 0.177 & & 0.95 & 0.252 \\
EFI &  (4,50) & 0.954 & 0.178 &  & 0.946 & 0.175 & & 0.95 & 0.252 \\
\bottomrule  
\end{tabular}
\end{center}
\end{table}

In summary, EFI performs as expected for this example, yielding similar results to OLS and GFI. \textcolor{black}{This is consistent with our analytic results in Example 1, where we showed that EFI results in the same theoretical confidence distribution as OLS and GFI for the linear regression model.} 
It is worth noting that in this particular example, the observations precisely follow the presumed model. In Section \ref{outliersection}, we will demonstrate that EFI can outperform likelihood-based methods when this situation is altered.

\subsection{Behrens-Fisher problem} \label{BFsection}

Consider two Gaussian distributions $N(\mu_1,\sigma_1^2)$ and $N(\mu_2,\sigma_2^2)$.
Suppose that two independent random samples of sizes $n_1$ and $n_2$ are drawn from them, respectively. The structural equations are given by 
\begin{equation}\label{BFproblem}
\begin{split}
    y_{1i}&=\mu_1+\sigma_1 z_{1i}, \quad i=1,\ldots,n_1, \\
    y_{2i}&=\mu_2+\sigma_2 z_{2i}, \quad i=1,\ldots,n_2, \
    \end{split}
\end{equation}
where $z_{i1}, z_{i2}\sim N(0,1)$ independently. 
The Behrens-Fisher problem pertains to the inference for the difference $\mu_1-\mu_2$ when the ratio $\sigma_1/\sigma_2$ is unknown.
Behrens \cite{Behrens1929} proposed the first solution to the problem 
in the context of testing the hypothesis $H_0: \mu_1=\mu_2$ versus $H_1: \mu_1 \ne \mu_2$, 
based on the pivot: 
\begin{equation} \label{pivoteq}
T= \frac{(\bar{Y}_1-\bar{Y}_2)-(\mu_1-\mu_2)}{ \sqrt{S_1^2/n_1+S_2^2/n_2}},
\end{equation} 
where $\bar{Y}_i$ and $S_i^2$ denote, respectively, the sample mean and sample variance 
of population $i$ for $i=1,2$. 
Fisher \cite{Fisher1935The} pointed out that this solution could be justified using the fiducial theory. Jeffreys \cite{Jeffreys1961TheOP} showed that a Bayesian calculation with the prior 
$\pi(\btheta) \propto (\sigma_1 \sigma_2)^{-1}$ yields the same confidence interval as the fiducial method.
From a frequentist perspective, Bartlett \cite{Bartlett1936TheIA} noted that inverting Behrens' test can lead to a conservative confidence interval for $\mu_1 - \mu_2$, i.e., its coverage probability is greater than the nominal level.
Later, based on the same statistic $T$, Welch \cite{Welch1947} proposed a $t$-test for which the resulting confidence interval for $\mu_1 - \mu_2$ has a coverage probability nearly equal to the nominal level. However, Fisher \cite{Fisher1956OnAT} criticized Welch's test for its negatively biased relevant selections, i.e.,  the coverage rate of its confidence interval can be lower than the nominal level for some instances.
As shown in \cite{Linnik1968}, there are no exact fixed-level tests based on the complete sufficient statistics for this problem. However, exact solutions based on other statistics and approximate solutions based on the complete sufficient statistics do exist.
Recently,  Martin and Liu \cite{Martin2015ConditionalIM} applied the {\it inferential model} method to this problem, resulting in the same confidence interval as Hsu-Scheffé's \cite{Hsu1938BF, Scheffe1970BF}, but which is known to be conservative \cite{Dudewicz2007ExactST}.  Wang and Jia \cite{wang2022te} developed a non-asymptotic $t$-test for the problem based on a statistic different from $T$, but the efficiency of the test is still unclear.

\textcolor{black}{We applied EFI to this problem by solving the two structural equations in (\ref{BFproblem}) 
separately: one for $(\mu_1, \sigma_1)$ and the other for $(\mu_2,\sigma_2)$.  
 Let $\{\hat{\mu}_1^{(k)}: k=1,2,\ldots, \mM\}$ and $\{\hat{\mu}_2^{(k)}: k=1,2,\ldots,\mM\}$ denote, respectively, the fiducial samples for the population means produced by the two EFI solvers. Then, the 95\% confidence interval for $\mu_1-\mu_2$ can be directly constructed by finding the 2.5th and 97.5th percentiles of the samples  $\{\hat{\mu}_1^{(k)} - \hat{\mu}_2^{(k)}: k=1,2,\ldots,\mM\}$. 
 This confidence interval construction method sets EFI significantly apart from existing methods, as it doesn't directly seek the distribution of a test statistic. This advantage of EFI will be further illustrated in Section \ref{EFI:testH}.}

 In our first simulations, we set $n_1 = n_2 = 50$, $\mu_1 = 1$, $\mu_2 = 0$,  
   and varied the values of $(\sigma_1^2, \sigma_2^2)$ as provided in Table \ref{Table:Behrens-Fisher}. The widths and coverage rates of the resulting confidence intervals are reported in Table \ref{Table:Behrens-Fisher}, where the results were obtained with $\mM=10,000$ and by averaging over 200 independent datasets. 
   For comparison, we also report the results from the Behrens-Fisher method (available in the R package `asht' \citep{ashtR2023}), Welch's method, Hsu-Scheffé's method, and  Te-test \cite{wang2022te}.
 The comparison suggests that for this example, EFI tends to be more efficient than the existing methods, yielding shorter confidence intervals while maintaining the same level of coverage rates. 

\begin{table}[!ht]
\caption{Statistical inference results for the Behrens-Fisher problem, where ``Coverage'' refers to the coverage rate of $\mu_1-\mu_2$ calculated by averaging over 200 datasets, and ``CI-width'' refers to the average width of respective confidence intervals.} 
\label{Table:Behrens-Fisher}
\begin{center}
\begin{tabular}{cccccccc} \toprule
        &   \multicolumn{3}{c}{($\sigma_1^2,\sigma_2^2)=(0.25,1)$} & & \multicolumn{3}{c}{($\sigma_1^2,\sigma_2^2)=(1,1)$} 
        \\ \cline{2-4} \cline{6-8}    
 Method & Coverage & CI-width & std CI&  & Coverage & CI-width & std CI
 \\ \midrule
 $n_1=n_2=50$   &       &        &  & & & & \\ 
 Behrens-Fisher  &  0.95 & 0.634 &  0.0040& & 0.955 & 0.802 & 0.0043 \\
  Welch  &  0.95 & 0.630 & 0.0040&   & 0.95 &  0.794 & 0.0042 \\
  Hsu-Scheff\'e & 0.95 & 0.635 &  0.0040& &  0.955  & 0.804 & 0.0043 \\
Te-Test & 0.94 & 0.633 &  0.0045 & & 0.95 & 0.800 & 0.0055 \\
 EFI & 0.95 & 0.609 & 0.0058  & & 0.955 &  0.788 & 0.0047 \\
\midrule  
$n_1=n_2=500$   &       &        &  & & & & \\ 
Behrens-Fisher  &  0.95 & 0.196&  0.0004 & & 0.95 & 0.248 & 0.0004\\
  Welch  &  0.95 & 0.196 & 0.0004&   & 0.95 &  0.247 & 0.0004 \\
  Hsu-Scheff\'e & 0.95 & 0.196 &  0.0004 & &  0.95  & 0.248 & 0.0004 \\
Te-Test & 0.95 & 0.196 &  0.0005 & & 0.95 & 0.248& 0.0005 \\
 EFI & 0.95 & 0.198 & 0.0006 & & 0.95 &  0.245 & 0.0014 \\
\bottomrule  
\end{tabular}
\vspace{0.2in}
\end{center}
\end{table}


\begin{figure}[!ht]
    \begin{center}
    \begin{tabular}{c}
    \includegraphics[width=0.6\textwidth]{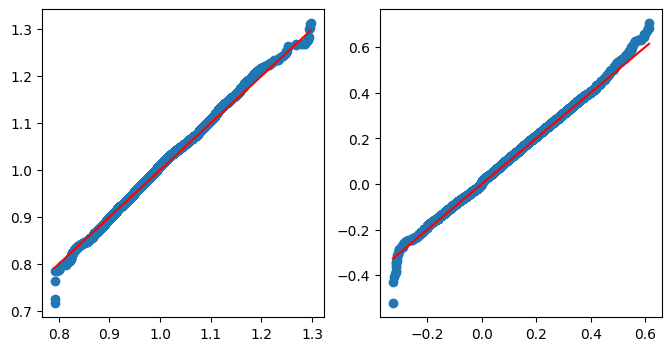} 
    \end{tabular} 
    \caption{Results of EFI for one dataset simulated from 
    (\ref{BFproblem}) with $n_1=n_2=50$: (left) Q-Q plot of $\{\hat{\mu}_1^{(k)}: k=1,2,\ldots \mM\}$ ($x$-axis)  and $\{\tilde{t}_1^{(k)}: k=1,2,\ldots \mM \}$
    ($y$-axis); (right) Q-Q plot of $\{\hat{\mu}_2^{(k)}: k=1,2,\ldots \mM\}$  ($x$-axis) and $\{\tilde{t}_2^{(k)}: k=1,2,\ldots \mM\}$ ($y$-axis)}. 
    \label{EFI-t-tail}
    \end{center}
\end{figure}

To explain the efficiency of EFI, we present in Figure \ref{EFI-t-tail} 
the Q-Q plots of $\{\hat{\mu}_i^{(k)}: k=1,2,\ldots,\mM\}$  
versus $\{ \tilde{t}_i^{(k)}: k=1,2,\ldots,\mM\}$ for $i=1,2$.
Here,  $\tilde{t}_i^{(k)}=\bar{y}_i - \frac{s_i}{\sqrt{n_i}} t_{n_i-1}^*(k)$, 
and $t_{n_i-1}^*(k)$ denotes the $k$th sample randomly drawn from a 
student $t$-distribution with $n_i-1$ degrees of freedom.   
Since the sample size $n$ is finite and thus the samples can be viewed as drawn from 
a tail-truncated distribution, EFI imputes the latent variables essentially from a tail-truncated distribution, due to its conditional inference nature. As a result, the Q-Q plots in Figure \ref{EFI-t-tail} display a tail-cut phenomenon. Therefore, when the sample size $n$ is small, the EFI confidence intervals can be shorter than those from unconditional inference methods, even they have the same coverage rates. 
However, when the sample size becomes large, this feature of conditional inference can disappear as illustrated by Table \ref{Table:Behrens-Fisher} 
with the results of $n_1=n_2=500$. 
We refer to this feature as the finite-sample effect for conditional inference.  
It is worth noting that since EFI solves for $(\mu_1,\sigma_1)$ and $(\mu_2,\sigma_2)$ separately, the Behrens-Fisher problem essentially becomes 
a linear regression problem with unknown variances for EFI. Therefore, it is not surprising that the empirical distribution of $\hat{\mu}_i$ closely matches 
 a location-scale student $t$-distribution.

\subsection{Bivariate Normal Distribution} \label{binormalsection}

Let $\by_1, \by_2,\ldots,\by_n$, with $\by_i=(y_{i,1},y_{i,2})^T$ for $i=1,2,\ldots,n$, be  
independent samples from a bivariate normal distribution 
with the mean vector and covariance matrix given as follows: 
\[
\begin{pmatrix} \mu_{1} \\ \mu_{2} \end{pmatrix}= \begin{pmatrix} 1 \\ 0 \end{pmatrix}, \quad \begin{pmatrix} \sigma_{1}^2 & \rho \sigma_{1} \sigma_{2} \\ \rho\sigma_{1} \sigma_{2} & \sigma_{2}^2 
\end{pmatrix}=\begin{pmatrix} 1 & 0.5 \\ 0.5 & 1 \end{pmatrix},
\]
where $\rho$ is the coefficient of correlation between two components of the bivariate normal vector. 
To perform EFI, we consider the following decomposition: 
\begin{equation}\label{BivariateFproblem}
\begin{split}
    y_{i,1}&=\mu_{1}+l_1z_{i,1}, \\
    y_{i,2}&=\mu_{2}+l_2z_{i,1}+l_3z_{i,2}, 
    \end{split}
\end{equation}
where $l_1>0$ and $l_3>0$, and $z_{i,k}$'s (for $k=1,2$ and $i=1,2,\ldots,n)$ are i.i.d standard normal random variables.  
It is easy to derive that $\sigma_{1}=l_1$, $\sigma_{2}=\sqrt{l_2^2+l_3^2}$, and  
$\rho=\frac{l_2}{\sqrt{l_2^2+l_3^2}}$.  Based on this decomposition, we set 
$\btheta=(\mu_{1},\mu_{2},\log(l_1),l_2,\log(l_3))^T$ for EFI. The results are presented 
in Table \ref{Table:binormal}, where we  
calculated the coverage rates and confidence interval widths based on 100 replications of the data set. 
The sample size is $n=100$ for each dataset. 




Inference for the parameters of the bivariate normal distribution has served as a classical example of fiducial inference. This can be seen in works such as Fisher \cite{Fisher1930Inv,Fisher1956Book}, Segal \cite{Segal1938FiducialDO}, and Bennett \cite{Bennett1969}. Their derivations have yielded the following established results:
\begin{itemize}
\item The marginal fiducial distribution of either $\mu_k$ is given by  
 $\sqrt{n} (\bar{y}_k-\mu_i)/s_k \sim t(n-2)$, where $\bar{y}_k=\frac{1}{n}\sum_{i=1}^n y_{i,k}$, $s_k=\frac{1}{\sqrt{n-1}} \sqrt{\sum_{i=1}^n (y_{i,k}-\bar{y}_k)^2}$, and $t(n-2)$ denotes a student-$t$ distribution with the degree of freedom $n-2$. 

\item The marginal fiducial distribution of either $\sigma_k^2$ is given by  
 $(n-1) s_k^2/\sigma_k^2 \sim \chi^2_{n-2}$, where $\chi^2_{n-2}$ denotes a chi-squared distribution with the degree of freedom being $n-2$. 
\end{itemize} 

According to \cite{Berger2006TheCF}, the marginal fiducial distribution of $\rho$ that was derived by Fisher  \citep{Fisher1930Inv} is the same as its marginal posterior distribution 
when the parameters are subject to the right-Haar prior $\pi(\mu_1,\mu_2,\sigma_1,\sigma_2,\rho) \propto \sigma_1^{-2}(1-\rho^2)^{-1}$. More precisely, the marginal fiducial distribution of $\rho$ has a stochastic representation as 
\[
\psi\left( - \sqrt{ \frac{\chi_1^{2*}}{\chi_{n-1}^{2*}}}+ 
 \sqrt{ \frac{\chi_{n-2}^{2*}}{\chi_{n-1}^{2*}}}
\frac{r}{\sqrt{1-r^2}} \right), \quad \mbox{where $\psi(x)=\frac{x}{\sqrt{1+x^2}}$},
\]
$r=\frac{1}{n-1} \sum_{i=1}^n (y_{i,1}-\bar{y}_1)(y_{i,2}-\bar{y}_2)/(s_1 s_2)$ is the sample correlation coefficient,  $\chi_1^{2^*}$, $\chi_{n-1}^{2*}$ and $\chi_{n-2}^{2*}$ are chi-squared random variables with the indicated degrees of freedom, and all the random variables 
are mutually independent.



  \begin{table}[htbp]
\caption{Comparison of the fiducial and EFI for inference of the parameters of the bivariate normal distribution, where the coverage rate and confidence interval length, given in the parentheses, 
were calculated by averaging over 100 datasets of sample size $n=100$. }
\label{Table:binormal}
\vspace{-0.2in}
\begin{center}
\begin{tabular}{lccccccc} \toprule
 Method & $\mu_{1}$ &   $\mu_{2}$ & $\sigma_{1}$  & $\sigma_{2}$ & $\rho$   & Average \\ \midrule
 Fiducial & 0.96 (0.398) & 0.96  (0.399)  & 0.97 (0.592)  & 0.96 (0.597) & 0.95 (0.295) & 0.96
  \\ 
 EFI & 0.95 (0.394) &  0.96  (0.404)  & 0.97 (0.564)  & 0.97 (0.555) & 0.95 (0.289)  & 0.96\\ \bottomrule 
\end{tabular}
\end{center}
\vspace{-0.2in}
\end{table}

The comparison suggests that for this example, EFI tends to produce shorter confidence intervals than the Fiducial method for the scale parameters $\sigma_1$, $\sigma_2$, and $\rho$, while the two methods tend to yield similar results for the location parameters $\mu_1$ and $\mu_2$. 
Once again, we attribute the efficiency of EFI in this example to the finite-sample effect, similar to 
the Behrens-Fisher problem.

 \subsection{Fidelity in Parameter Estimation} \label{outliersection}


The frequentist methods often conduct parameter estimation under the maximum likelihood principle.
As implied by the constraint (\ref{MLEconstraint}), 
the MLE can be easily contaminated by outliers. In contrast, as implied by (\ref{energyfunction00}) and (\ref{eqA12}), 
EFI essentially estimates $\btheta$ by maximizing the  predictive likelihood function  $\pi(\bZ_n|\bX_n,\bY_n,\btheta)\propto \pi_0^{\otimes n}(\bZ_n) e^{-\lambda \sum_{i=1}^n 
d(y_i,x_i,z_i,\btheta)}$,
which balances the fitting errors and the likelihood of random errors. Compared to the MLE  $\hat{\btheta}_{MLE}= \arg\max_{\btheta}\pi_0^{\otimes n}(\bZ_n)$, 
where $\bZ_n$ can be expressed as a function of $(\bY_n,\bX_n,\btheta)$, 
the EFI estimator tends to be more robust to outliers and provides higher fidelity in parameter estimation. However, if the model is correctly specified, no outliers exist, and the sample size is reasonably large, maximizing 
$\pi_0^{\otimes n}(\bZ_n)$ leads to an approximate minimization of the fitting error  $\sum_{i=1}^n d(y_i,x_i,z_i,\btheta)$. 
Specifically, when $\hat{\btheta}_{MLE} \stackrel{p}{\to} \btheta^*$, $\bZ_n^*$ can be recovered in probability and thus  $\sum_{i=1}^n 
d(y_i,x_i,z_i,\btheta) \stackrel{p}{\to} 0$. 
In such cases,  the two methods will yield similar estimates, refer to Table \ref{Table:Linear_unknownVariance} for an illustrative example 
of this issue. 

To illustrate EFI's robustness to outliers, we consider the model (\ref{structeq}) again. In this new simulation, we set $n=600$ and generated random errors from a mixture Gaussian distributions: $z_1,z_2,\ldots,z_{540} \sim N(0,1)$ and $z_{541},z_{542},\ldots,z_{600} \sim N(4,1)$. The latter cases were considered as outliers, although some of them might be indistinguishable from the former ones. 
 Figure \ref{residualscomparsionB} compares the performances of EFI and OLS on a simulated dataset. It suggests that EFI only slightly shrank 
the random errors and led to a more accurate estimate of $\sigma^2$ ($\approx 1.0$) and narrower 
confidence intervals for $\bbeta$, while the OLS estimate of $\sigma^2$ ($\approx 1.8$) was significantly enlarged by outliers and the resulting confidence intervals of $\bbeta$ were much wider.
The Bayesian method performs similarly to the maximum likelihood estimation method, 
as they both are likelihood-based.

\begin{figure}[htbp]
    \centering
    \includegraphics[width=0.8\textwidth]{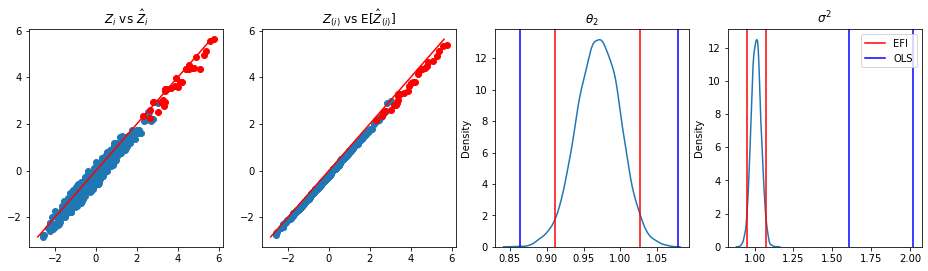}
    \caption{Fidelity of EFI in parameter estimation: (left) 
    scatter plot of residuals: $z_i$ versus $\hat{z}_i$; (middle left) scatter plot of ordered residuals: $z_{(i)}$ versus $\hat{z}_{(i)}$; (middle right) EFI and OLS confidence intervals for $\beta_1$; (right) EFI and OLS confidence intervals for $\sigma^2$. }
    \label{residualscomparsionB}
\end{figure}

In Section \ref{EFI:nonlinear} of the supplement, we present another example which shows that the EFI estimator is less prone to overfitting compared to those from the maximum likelihood or ordinary least square method. This is again attributed to its emphasis on balancing the fitting errors and the likelihood of random errors.

\section{EFI for Semi-Supervised Learning}
 
As mentioned previously, the incorporation of computer technology into science and daily life
has enabled scientists to collect massive volumes of data during the past two decades. 
However, many of the data are unlabeled, as acquisition of labeled data 
for many problems can be expensive.
In such situations, semi-supervised learning (SSL), which is to combine a small amount of labeled data with a large amount of unlabeled data to enhance the learning of a classifier, 
can be of great practical value. However, to make use of unlabeled data, some assumptions 
about the distribution of the data are needed \citep{Chapelle2006SemiSupervisedL}. 
For example, one often makes i) the {\it smoothness assumption} that the points closing to each other are more likely to share a label,  ii) the {\it cluster assumption} that the points form some clusters and those in the same cluster are more likely to share a label (although the data share a label may spread across multiple clusters), or iii) the {\it manifold assumption} that the high-dimensional data lie roughly on a low-dimensional manifold.  
 The existing SSL methods  can be roughly divided into categories such as  consistency regularization, proxy-label, generative models, and graph-based methods. See \cite{Zhu2005SemiSupervisedLL} and \cite{Ouali2020AnOO} for overviews.  

 For a better explanation of the idea behind the current SSL methods, let's consider a text classification problem.
   Let $\bx_l$ denote labeled text data, let $\by_l$ denote the labels, and let $\bx_u$ denote unlabeled text data. 
 \cite{Nigam2006SemiSupervisedTC} modeled the text data using a mixture multinomial 
 distribution as a generative model. 
 By treating  $\by_u$, the labels of $\bx_u$, as missing data, 
 they derived the incomplete data posterior: 
 \begin{equation} \label{SSLeq}
 \begin{split}
 \log \pi(\btheta|\by_l, \bx_l, \bx_u) & = Const+\log\pi(\btheta) 
  + \sum_{x_i \in \bx_l} \log\left( p(y_i=c_j|\btheta) p(x_i|y_i=c_j,\btheta) \right) \\
  & +\sum_{x_i \in \bx_u} \log\big(\sum_{c_j \in S} p(c_j|\btheta) p(x_i|c_j, \btheta) 
   \big), \\
\end{split}
 \end{equation}
where $S$ denotes the set of classes, $\btheta$ denotes the set of parameters 
of the mixture distribution, and $\pi(\btheta)$ denotes the  prior of $\btheta$. 
As implied by (\ref{SSLeq}), the key for SSL is to model the text data $(\bx_l,\bx_u)$ for its 
class-wise distribution, i.e., $p(x_i|c_j,\btheta)$. 
 Otherwise, under the conventional regression setting  where $(\bx_l,\bx_u)$ is treated 
as constants, the last term in (\ref{SSLeq}) will be dropped and the unlabeled data 
will not be able to help to improve the estimate of $\btheta$. 

In contrast, as indicated by Figure \ref{EFInetwork}, EFI uses both the text data $\bx$ and labels $\by$ as input, and models the distribution of $\bx$ in an implicit way. 
Moreover, such an implicit model is general and user friendly due to the universal approximation power of deep neural networks. Therefore, EFI can be easily adapted to SSL by treating $\by_u$ as missing data, which
will be sampled along with the latent variable $\bZ_n$ in step (ii) of Algorithm \ref{EFIalgorithm}. 
To illustrate the potential of EFI in SSL, we consider some classification problems
taken at UCI machine learning repository.

For binary classification, the second term (i.e.,  fitting error term)  
in (\ref{energyfunction00}) 
can be replaced by 
\begin{equation} \label{penalty_semisupervised}
   \sum_{i=1}^{n_{l}} \rho((u_i-x_i^T\hat{\btheta}_i)(2y_i-1))+\sum_{j=1}^{n_{u}} \rho((u_j^{miss}-x_j^T\hat{\btheta}_j)(tanh(\frac{v_j^{miss}}{\tau})),
   \end{equation}
where $\rho(\cdot)$ is a ReLU function, $n_l$ denotes the number of labeled data, $n_u$ denotes the number of unlabeled data, $u_i$ and $u_j^{miss}$ are latent variables, 
$v_j^{miss}$ is defined through the equation $P(y_j^{miss}=1)=\frac{1}{1+e^{-v_j^{miss}/\tau}}$ for the missed label, and $\tau$ is a scale parameter. 
In simulations, we set $\tau=1/50$, ensuring the probability $\frac{1}{1+e^{-v_i^{miss}/\tau}}$  is dichotomized to either 1 or 0, and treat $\{ u_i: i=1,2,\ldots,n_l\}$ and
$\{ u_j^{miss}, v_j^{miss}: j=1,2,\ldots,n_u\}$ as latent variables to simulate at each iteration.
For EFI, (\ref{penalty_semisupervised}) can be changed by replacing $\hat{\btheta}_i$'s and 
$\hat{\btheta}_j$'s with $\bar{\btheta}_n$. 
 For multiclass classification problems, (\ref{penalty_semisupervised}) can be slightly modified.


For each dataset, EFI was run in 5-fold cross-validation,  where the labels were removed from 50\% of the training samples. The results are summarized in Table \ref{SSLtab}, where the results of supervised 
learning were obtained with the classical logistic regression.
For comparison, the self-training algorithm \citep{Yarowsky1995UnsupervisedWS} and label-propagation algorithm \citep{Bengio2006Label,Delalleau2004EfficientNF}, which both 
belong to the category of proxy-label methods and are available in 
the package {\it scikit-learn 1.2.0}, were applied to the datasets. 
\textcolor{black}{In self-training, a model is first trained on labeled data, 
this trained model is then used to predict the classification probabilities of unlabeled data, and predictions with high confidence are added to the training set to retrain the model.
In label-propagation learning, a graph is first created to 
connect the training samples, and then the known labels are propagated through the edges of the graph to unlabeled samples 
in the training set.
A drawback of these methods is that the model is unable to correct its own mistakes, potentially amplifying wrong classifications or biases through the training process. The supervised learning methods are to learn a logistic regression model for each of the datasets.}
 

The comparison shows the superiority of EFI in SSL, which can generally perform much 
better than the self-training and label propagation algorithms. 
For the dataset ``Raisin'', EFI even outperforms supervised learning, 
\textcolor{black}{and we would attribute this performance of EFI to its fidelity in parameter estimation.}
For these datasets, we have also applied EFI to the full training set and labeled data only. 
The results are similar to those from the logistic regression. Refer to the supplement for the detail. 

 \begin{table}[htbp]
\caption{Comparison of EFI with supervised learning  and 
semi-supervised learning  algorithms 
 for some classification problems, where $\mu\pm se$ represents the mean prediction accuracy 
 of the 5-fold cross validation runs and the standard deviation  of the mean value.}
\label{SSLtab}
\begin{center}
\begin{tabular}{cccccccc} \toprule
        &  &  \multicolumn{2}{c}{Supervised Learning} & & \multicolumn{3}{c}{Semi-Supervised Learning} \\ 
         \cline{3-4} \cline{6-8}
 Dataset & size & Full & Labeled Only & & Self-training & Label-propagation  &  EFI \\ \midrule
 Divorce & 170 & 98.82$\pm$1.05 & 96.47$\pm$1.29 & &  92.94$\pm$3.87 &  96.47$\pm$1.29  & 98.82$\pm$1.05 \\
 Diabetes & 520 & 89.62$\pm$1.29 & 87.69$\pm$1.69 & & 87.31$\pm$2.01 & 85.77$\pm$2.01 & 88.08$\pm$ 0.64 \\
Breast Cancer & 699 & 96.52$\pm$0.66 & 95.36$\pm$0.26 & &  94.39$\pm$0.76 & 95.07$\pm$0.52 &  96.23$\pm$0.52 \\
Raisin & 900 & 82.89$\pm$1.16 & 83.78$\pm$0.24 & & 58.67$\pm$1.35 & 50.22$\pm$0.20 & 85.56$\pm$ 0.99 \\ \bottomrule
\end{tabular}
\end{center}
\end{table}

\section{EFI for Complex Hypothesis Tests} \label{EFI:testH}

 As the scale and complexity of scientific data grow, there is often an interest in testing more complex hypotheses. However, within the frequentist framework, it is usually challenging to derive the theoretical reference distributions for the corresponding test statistics. In contrast, EFI operates in the mode of conditional inference, circumventing the need for theoretical reference distributions and enabling easy hypothesis testing based on collected fiducial samples. In this sense, EFI is driving statistical inference toward an automated process.

  To illustrate the automaticity of EFI in hypothesis testing, we consider the following mediation  analysis model \citep{Baron1986TheMV}: 
 \begin{equation} \label{medmodel}
     \begin{split}
      Y&=\beta_T T+\beta M+\bbeta_x^T X +\epsilon_{Y}, \quad \epsilon_{Y}\sim N(0,\sigma_Y^2), \\
      M&=\gamma T +\bgamma_x^T X + \epsilon_{M}, \quad  \epsilon_{M}\sim N(0,\sigma_M^2),
     \end{split}
 \end{equation}
 where $Y$, $T$, $M$ and $X$ denote the outcome, treatment, mediator and design matrix, 
 respectively.  The mediator effect can be inferred by testing the hypothesis $H_0:\beta\gamma=0$ against $H_A:\beta\gamma\neq 0$ with the natural 
 test statistic $\hat{\beta}\hat{\gamma}$. As mentioned by \cite{Miles2021mmoptimal},
 this is a challenging inferential task due to the non-uniform asymptotics of 
 the univariate test statistic. Specifically, the null hypothesis consists of three cases: (i) $\beta=0,\gamma\neq 0$, (ii) $\beta\neq 0, \gamma=0$, and (iii) 
 $\beta=\gamma=0$, while the theoretical reference distribution of $\hat{\beta} \hat{\gamma}$ under case (iii) is different from that under cases (i) and (ii). 
 It is known that traditional statistical tests such as Sobel's test \citep{Sobel1982sobel} 
  and Max-P test \citep{Mackinnon2002MaxP} are conservative under case (iii).
 Recently, with a fine theoretical analysis, \cite{Miles2021mmoptimal} derived a test that is minimax optimal with respect to local power over the alternative parameter space while preserving type-I error. 
 
 In contrast, applying EFI to such a composite hypothesis test is straightforward. 
 The mediator effect can be directly inferred based on the fiducial samples of 
 $\beta$ and $\gamma$, which can be collected along with iterations of Algorithm \ref{EFIalgorithm}. \textcolor{black}{
 We note that the bootstrap method \citep{EfronT1993} works in a similar way to EFI, which performs conditional inference for the model parameters and approximates their confidence distributions in an 
 empirical way. In this paper, we implemented the bootstrap method for the model (\ref{medmodel}) using the R package ``mediation'' \citep{Tingley2014mediationRP} under the default setting.}

 \paragraph{Simulation Studies} For illustration, we simulated 100 datasets from the model 
 (\ref{medmodel}) under each of the cross settings of $n\in \{500,1000, 2000\}$ and  
 $(\beta,\gamma)\in \{(0.2,0), (0,0.2), (0,0)\}$, 
 where $X=(X_1,X_2)$ consists of two independent standard Gaussian random variables, \textcolor{black}{
 $\sigma_Y=\sqrt{2}$, $\sigma_M=1$}, $\bbeta_x=(0.2,0.4)^T$, $\beta_T=1$, $\bgamma_x=(0.4,0.6)^T$.  
 The results are summarized in Table \ref{medtab_typeI}, which 
 indicates the validity and superiority of EFI in testing complex hypotheses. 
 Compared to the other methods, the type-I errors
 of EFI are much closer to the nominal level 0.05, see Figure \ref{fig:Table6} in the supplement for a graphical view of the results.

\begin{table}[!ht]
\caption{Type-I errors of the Sobel, MaxP, minimax optimal (mm-opt), bootstrap, and EFI tests for the mediator effect, where the significance level of each test is $\alpha=0.05$. } 
\label{medtab_typeI}
 \begin{center} 
  \begin{tabular}{cccccccccccc} \toprule 
   & \multicolumn{3}{c}{$n=500$} & & \multicolumn{3}{c}{$n=1000$} & & \multicolumn{3}{c}{$n=2000$} \\ \cline{2-4} \cline{6-8} \cline{10-12} 
 $(\beta,\gamma)$  & (0.2,0) & (0,0.2) & (0,0) & &  (0.2,0) & (0,0.2) & (0,0) & & (0.2,0) & (0,0.2) & (0,0) \\ \midrule 
 Sobel & 0.01 & 0.00 & 0.00 & & 0.05 & 0.02 & 0.00 && 0.04 & 0.06 & 0.00 \\ 
 MaxP  & 0.04 & 0.03 & 0.00 & & 0.06 & 0.05 & 0.00 && 0.07 & 0.07 & 0.00  \\
 mm-opt & 0.05 & 0.04 & 0.03 && 0.06 & 0.05 & 0.07 && 0.07 & 0.07 & 0.07 \\ 
 Bootstrap & 0.06 & 0.05 & 0.01 && 0.04 & 0.07 & 0.00 && 0.13 & 0.04 & 0.00 \\ 
 EFI &    0.05 & 0.06 & 0.04 && 0.06 & 0.04 & 0.04 && 0.05 & 0.04 & 0.05 \\
 \bottomrule
 \end{tabular}
 \end{center}
 \end{table}

  Further, we simulated datasets for comparison of the powers of these tests, where 
 $(\beta,\gamma)\in \{(0.1,0.4)$, $(-0.1,0.4)$, $(0.2,0.2)\}$ and other parameters were as set 
 in the type-I error experiments. The results are summarized in Table \ref{medtab_power}, see also Figure \ref{fig:Table7} in the supplement for a graphical view of the results. 
 The comparison indicates that the EFI test has higher power than the other methods. \textcolor{black}{The superiority of EFI over the Bootstrap method is particularly encouraging, highlighting the great potential of EFI in conditional inference and advancing the automation of statistical inference.}

\begin{table}[htbp]
\caption{Powers of the Sobel, MaxP, minimax optimal (mm-opt), bootstrap, and EFI tests for the mediator effect, where the significance level of each test is $\alpha=0.05$.
Part of the results of mm-opt are not available (NA), as the test is inefficient for 
the alternative hypothesis settings of $(\beta,\gamma)$ when the sample size becomes large.
} 
\label{medtab_power}
 \begin{center} 
 \begin{adjustbox}{width=1.0\textwidth}
  \begin{tabular}{cccccccccccc} \toprule 
   & \multicolumn{3}{c}{$n=500$} & & \multicolumn{3}{c}{$n=1000$} & & \multicolumn{3}{c}{$n=2000$} \\ \cline{2-4} \cline{6-8} \cline{10-12} 
 $(\beta,\gamma)$  & (0.1,0.4) & (-0.1,0.4) & (0.2,0.2) & &  (0.1,0.4) & (-0.1,0.4) & (0.2,0.2) & & (0.1,0.4) & (-0.1,0.4) & (0.2,0.2) \\ \midrule 
 Sobel & 0.29 & 0.31 & 0.67 & & 0.65 & 0.57 & 0.96 && 0.78 & 0.89 & {\bf 1.00} \\ 
 MaxP  & 0.34 & 0.37 & 0.79 & & 0.66 & 0.59 & {\bf 0.98} && 0.78 & 0.89 & {\bf 1.00}  \\
 mm-opt & 0.34 & 0.37 & 0.79 && NA & NA & NA && NA & NA & NA \\ 
  Bootstrap & 0.33 & 0.42  & 0.52 &&  0.59 & 0.51 & 0.93 &&  {\bf 0.93} & 0.92 & {\bf 1.00}   \\ 
 EFI &   {\bf 0.48} & {\bf 0.64} & {\bf 0.84} && {\bf 0.70} & {\bf 0.74} & 0.97 && 0.86 & {\bf 0.95} & {\bf 1.00} \\
 \bottomrule
 \end{tabular}
 \end{adjustbox}
 \end{center}
 \end{table}

\begin{remark} \label{Remtest} This example demonstrates the potential of  EFI in hypothesis testing. 
Due to its conditional inference nature, EFI eliminates the need for theoretical reference distributions, thereby automating the process of hypothesis testing.
Moreover, compared to frequentist methods, EFI lowers the requirement for sample size. In particular, under high-dimensional scenarios where the model dimension $p$ grows with the sample size $n$,  frequentist methods typically require ${p^2}/n \to 0$  for achieving asymptotic normality  (see e.g. \cite{Portnoy1986OnTC} and \cite{Portnoy1988}). For EFI, we believe that $p/n \to 0$ is sufficient for achieving valid fiducial inference, which ensures   Assumption (\ref{ass*}) holds for many data generation equations.   
A further theoretical study on this issue will be reported elsewhere.
\end{remark}

\section{Discussion} 


We have developed EFI as a novel and flexible framework for statistical inference,  applicable to general statistical models regardless of the type of noise, whether additive or non-additive. 
We have also introduced the EFI-DNN algorithm for effective implementation of EFI, which jointly imputes the realized random errors in observations
using stochastic gradient Markov chain Monte Carlo 
and estimates the inverse function using a sparse DNN based on all available data. The consistency of the sparse DNN estimator ensures that the uncertainty embedded in the observations is properly propagated to the model parameters through the estimated inverse function, thereby validating downstream statistical inference. The EFI-DNN algorithm has demonstrated appealing properties in parameter estimation, hypothesis testing, and semi-supervised learning. 
Additionally, thanks to the conditional inference nature of EFI and the universal approximation power of DNNs, the EFI-DNN algorithm holds  great potential to automate statistical inference.  
Toward this direction, further study on the theoretical properties 
of the EFI-DNN inference is of great interest. 


The EFI-DNN algorithm is scalable, which can handle very large-scale datasets with the use of adaptive stochastic gradient MCMC algorithms. Specifically, its parameter updating step can be accelerated by the mini-batch strategy; and the latent variable sampling step can be executed separately for each observation, enabling straightforward implementation in a parallel architecture. Theoretical guarantees for the convergence of the algorithm have been studied; we established the weak convergence of the imputed random errors and the consistency of the inverse function estimator.
  
 This paper has considered only the problems
 where  $p$ is either fixed or grows with $n$ slowly enough to satisfy Assumption \ref{ass9}-(ii).  
 Extending the EFI-DNN algorithm to high-dimensional problems, where $p> n$ and/or $p$ grows with $n$ at a higher rate, is possible.  
 For instance, if the high-dimensional issue arises from including an excessively large number of covariates,  a model-free sure independence screening procedure (see e.g.,\cite{XueLiang2017a,Cui2015ModelFreeFS}) can be performed 
  on the data before applying the algorithm.
Furthermore, if one aims to examine the uncertainty of a parameter for an individual covariate, the Markov neighborhood regression (MNR) approach  \cite{LiangXJ2020MNR,Liang2022DoubleR,Sun2022MarkovNR} can be applied. This approach decomposes the high-dimensional inference problem into a sequence of low-dimensional inference problems based on the graphical model formed by the covariates.


 

\section*{Availability} 

The code that implements the EFI method can be found at \url{https://github.com/sehwankimstat/EFI}.

\section*{Acknowledgments} 
Liang's research is support in part by the NSF grants DMS-2015498 and DMS-2210819, and the NIH grants R01-GM126089 and R01-GM152717. The authors thank the editor, associate editor, and three referees for their constructive comments, which have led to significant improvement of this paper.

\newpage

\appendix

{\Large \bf Appendix: Supplement for ``Extended Fiducial Inference:  Toward an Automated Process of Statistical Inference''} 

\vspace{0.1in}

\setcounter{section}{0}
\renewcommand{\thesection}{$\S$\arabic{section}}
\setcounter{table}{0}
\renewcommand{\thetable}{S\arabic{table}}
\setcounter{figure}{0}
\renewcommand{\thefigure}{S\arabic{figure}}
\setcounter{equation}{0}
\renewcommand{\theequation}{S\arabic{equation}}
\setcounter{lemma}{0}
\renewcommand{\thelemma}{S\arabic{lemma}}
\setcounter{theorem}{0}
\renewcommand{\thetheorem}{S\arabic{theorem}}
\setcounter{remark}{0}
\renewcommand{\theremark}{S\arabic{remark}}
\setcounter{assumption}{0}
\renewcommand{\theassumption}{A\arabic{assumption}}

This supplement is organized as follows. Section \ref{sectproof1} provides the proofs for Theorem \ref{thm1} and Theorem \ref{thm2}. Section \ref{sectproof2} provides the proof for Theorem \ref{thm3}. 
Section \ref{EFDproof}  provides the proof for Example 1 of Section 3.1 of the main text. 
Section \ref{moreresults} provides more numerical results. Section \ref{moresettings} presents detailed parameter settings used in the numerical experiments. 


\section{Proof of Theorem \ref{thm1} and Theorem \ref{thm2}} \label{sectproof1}
 
 \noindent {\bf Notation:} For both Theorem \ref{thm1} and Theorem \ref{thm2}, the sample size $n$ is fixed. For simplicity of notation, we will replace the dataset notation $(\bX_n,\bY_n,\bZ_n)$ by $(\bx_n,\by_n,\bz_n)$ and further drop the subscripts of $\bx_n$, $\by_n$, 
 $\bz_n$, $\bw_n$ and $\mathcal{W}_n$ in the remaining part of this section. 
 Additionally, for convenience, we redefine $h(\bw):=\nabla_{\bw} \log \pi(\bw|\bx,\by)$, $\mathcal{H}(\bw,\bz):=\nabla_{\bw} \log \pi(\bw|\bx,\by,\bz)$, 
 $\pi_D(\bz|\bw):=\pi(\bz|\bx,\by,\bw)$, and $F_D(\bz,\bw):=\log \pi_D(\bz|\bw)$, where $D$ represents a training dataset. Furthermore, 
 with a slight abuse of notation, we use $\bz_k$ and $\bw_k$ to denote the latent variable sample and parameter estimate obtained at iteration $k$ of Algorithm \ref{EFIalgorithm}.

 \subsection{Proof of Theorem \ref{thm1}}
 
 With the simplified notation, 
 the equation (\ref{identityeq}) of the main text can be rewritten as  
 \begin{equation} \label{sa00}
 h(\bw)=\int  \mH(\bw,\bz) \pi_D(\bz|\bw) d\bw=0,
\end{equation}
where $\bw\in \mathbb{R}^{d_w}$, $\bz \in \mathbb{R}^{d_z}$, and $d_w$ and $d_z$ denote the dimensions of $\bw$ and $\bz$, respectively. 
The adaptive SGLD algorithm used for solving  equation (\ref{sa00}) can be written in a general form as 
\begin{equation} \label{ASGLDeq}
\begin{split} 
\bz_{k+1}& =\bz_{k}+\epsilon_{k+1} g(\bz_k,\bw_{k},u_{D,k})+\sqrt{2\epsilon_{k+1}} \be_{k+1}, \\
\bw_{k+1}&=\bw_{k}+\gamma_{k+1} \mH(\bw_{k},\bz_{k+1}),
\end{split}
\end{equation}
where $k \in \mathbb{N}$ indexes iterations, $\epsilon_{k+1}\in \mathbb{R^+}$ denotes the learning rate, $\gamma_{k+1} \in \mathbb{R^+}$ denotes the step size, $\be_k \sim N(0,I_{d_{\bz}})$ is a zero mean standard Gaussian random vector,  $g(\bz_k,\bw_{k},u_{D,k}): \mathbb{R}^{d_{\bz}} \times \mathbb{R}^{d_w} \times \mathcal{U} \to \mathbb{R}^{d_{\bz}}$ denotes an unbiased estimator of $\nabla_{\bz} F_D(\bz_k,\bw_{k})$, $\mathcal{U}=\{1,2,\ldots,n\}$ is the index set of the observations in $D$,  and $\{u_{D,k}: k=1,2,\ldots\}$ is a sequence of i.i.d random elements of $\mathcal{U}$ with probability measure $\mathcal{Q}_D$. In general, $u_{D,k}$ can be understood as the index set of a
mini-batch sample. In the case that the full dataset is used at each iteration, we have  
$u_{D,k}=\mathcal{U}$ for all $k$.

 To prove the convergence of the adaptive SGLD algorithm (\ref{ASGLDeq}), 
 we make the following assumptions. 
   
\begin{assumption} \label{ass1}
 The step size sequence $\{\gamma_k\}_{k \in \mathbb{N}}$ is a positive decreasing 
 sequence of real numbers such that 
\begin{equation}
   \lim_{k\to \infty} \gamma_{k}=0, \quad 
   \sum_{k=1}^{\infty} \gamma_{k}= \infty.
 \end{equation}
 There exist $\delta>0$ and a stationary point $\bw^*$ such that for any $\bw\in \mathcal{W}$,
 \[
  \langle \bw-\bw^*, h(\bw) \rangle \leq -\delta \|\bw-\bw^*\|^2,
  \]
 and, in addition, 
  \begin{equation}
  \liminf_{k\to \infty} 2\delta\frac{\gamma_{k}}{\gamma_{k+1}}+\frac{\gamma_{k+1}-\gamma_{k}}{\gamma_{k+1}^2}>0,
\end{equation}
where $\|\cdot\|$ denotes the default $l_2$-norm. 
\end{assumption}

\begin{assumption} \label{ass2} $F_D(\bw,\bz)$ is \textit{M-smooth} on $\bw$ and $\bz$ with $M>0$, and $(m,b)$-dissipative on $\bz$ for some constants $m>1$ and $b>0$. In other words, for any $\bz, \bz^{\prime}, \bz^{\prime\prime}\in \mathcal{X}$ and $\bw, \bw^{\prime} \in\mathcal{W}$, the following  conditions are satisfied:
\begin{align}
  \mbox{(smoothness)}  & \quad \|\nabla_{\bz}F_D(\bw,\bz^{\prime})- \nabla_{\bz}F_D(\bw^{\prime},\bz^{\prime\prime}) \| \leq M \|\bz^{\prime}-\bz^{\prime\prime} \|  + M\| \bw-\bw^{\prime} \|,  \label{M1eq} \\
  \mbox{(dissipativity)} &  \quad \langle \nabla_{\bz}F_D(\bw^*,\bz), \bz \rangle \leq b-m\|\bz\|^2, 
    \label{disspateq} 
\end{align}
where 
$\bw^*$ is a stationary point  as defined in Assumption \ref{ass1}. 
\end{assumption}

Let $(\bw^*,\bz^*)$ be a minimizer of $F_D(\bw,\bz)$ and $\bw^*$ be a stationary point such that $\nabla_{\bz}F_D(\bw^*,\bz^*)=0$. By (\ref{disspateq}),  we have $\|\bz^*\|^2\leq \frac{b}{m}$. Therefore,
\[
\begin{split}
    \|\nabla_{\bz}F_D(\bw,\bz)\|&\leq \| \nabla_{\bz}F_D(\bw^*,\bz^*)\|+M\|\bz^*-\bz\|+M\|\bw-\bw^*\| \\
    & \leq M\|\bw-\bw^*\|+M\|\bz\|+{B},
\end{split}    
\]
where $B=M \sqrt{\frac{b}{m}}$, and 
\begin{equation} \label{boundF}
\|\nabla_{\bz} F_D(\bw,\bz)\|^2\leq 3M^2\|\bz\|^2+3M^2\|\bw-\bw^*\|^2+3 B^2. 
\end{equation}

\begin{assumption} \label{ass3}   Let $R_k=g(\bw_k,\bz_k,u_{D,k})- \nabla_{\bz} F_D(\bw_k,\bz_k)$. Assume that $R_k$'s are mutually independent white noise, 
and they satisfy the conditions
\begin{equation} \label{noiseeq}
     \mE(R_k|\mathcal{F}_k)=0,   \quad  
    \mE\|R_k\|^2\leq \delta_g (M^2 \mE\|\bz_k\|^2+M^2 \mE\|\bw_k-\bw^*\|^2+B^2), 
\end{equation}
where $\delta_g$ and $B$ are positive constants, and $\mathcal{F}_k=\sigma\{\bw_1,x_1,\bw_2,x_2,\ldots, \bw_k,x_k\}$ denotes a $\sigma$-filtration.  
\end{assumption}

\begin{assumption} \label{ass4}
There exist positive constants $M$ and $B$ such that 
\[
\begin{split}
    \|\mH(\bw,\bz)\|^2 & \leq M^2 \|\bw-\bw^*\|^2+M^2 \|\bz\|^2+B^2.\\
\end{split}
\]
\end{assumption}

\begin{lemma}(Uniform $L^2$ bounds; Lemma A.2 of  \cite{DongZLiang2022}) \label{l2bound} Suppose Assumptions \ref{ass1}-\ref{ass4} hold, and the learning rate sequence $\{\epsilon_k: k=1,2,\ldots\}$ and the step size sequence $\{\gamma_k: k=1,2,\ldots\}$ are set in the form:
\[
\epsilon_k=\frac{C_{\epsilon}}{c_{\epsilon}+k^{\alpha}}, \quad \gamma_k=\frac{C_{\gamma}}{c_{\gamma}+k^{\beta}},
\]
for some constants $C_{\epsilon}>0$, $c_{\epsilon}>0$, $C_{\gamma}>0$, $c_{\gamma}>0$, $\alpha,\beta\in (0,1]$, and $\beta \leq \alpha \leq \min\{1,2 \beta\}$.
Then there exist constants $G_z$ and $G_{\bw}$ such that $\mE \|\bz_k\|^2 \leq G_z$ and $\mE \|\bw_k-\bw^*\|^2 \leq G_{\bw}$ for all $k=0,1,2,\ldots$.
\end{lemma}

\begin{assumption} \label{ass5} (Solution of Poisson equation) 
For any $\bw\in\mathcal{W}$, $\bz \in \mX$, and a function $V(\bz)=1+\|\bz\|$, 
there exists a function $\mu_{\bw}$ on $\mX$ that solves the Poisson equation $\mu_{\bw}(\bz)-\mathcal{T}_{\bw}\mu_{\bw}(\bz)=\mathcal{H}(\bw,\bz)-h(\bw)$, 
where $\mathcal{T}_{\bw}$ denotes a probability transition kernel with
$\mathcal{T}_{\bw}\mu_{\bw}(\bz)=\int_{\mX} \mu_{\bw}(\bz')\mathcal{T}_{\bw}(\bz,\bz') d \bz'$,  such that 
\begin{equation} \label{poissoneq0}
   \mathcal{H}(\bw_k,\bz_{k+1})=h(\bw_k)+\mu_{\bw_k}(\bz_{k+1})-\mathcal{T}_{\bw_k}\mu_{\bw_k}(\bz_{k+1}), \quad k=1,2,\ldots.
\end{equation}
Moreover, for all $\bw, \bw'\in \mathcal{W}$ and $\bz \in \mX$, 
$\|\mu_{\bw}(\bz)-\mu_{\bw'}(\bz)\|  + \|\mathcal{T}_{\bw}\mu_{\bw}(\bz)-\mathcal{T}_{\bw'} \mu_{\bw'}(\bz)\| \leq \varsigma_1 \|\bw-\bw'\| V(\bz)$ and
$\|\mu_{\bw}(\bz) \| + \|\mathcal{T}_{\bw} \mu_{\bw}(\bz) \|\leq \varsigma_2 V(\bz)$ for some constants $\varsigma_1>0$ and $\varsigma_2>0$. 
\end{assumption}

\begin{lemma}\label{thm1A} [Theorem A.1 of \cite{DongZLiang2022}] Suppose Assumptions \ref{ass1}-\ref{ass5} hold, and the learning rate sequence $\{\epsilon_k: k=1,2,\ldots\}$ and the step size sequence $\{\gamma_k: k=1,2,\ldots\}$ are chosen as in 
Lemma \ref{l2bound}. Then there exists a root $\bw^*\in \{\bw: h(\bw)=0\}$ such that
\begin{equation} \label{thetacon2}
 \mathbb{E}\|\bw_k-\bw^*\|^2\leq \xi \gamma_k, \quad k\geq k_0,
 \end{equation} 
 where $\xi$ and $k_0$ are some constants determined by  the sequences $\{\epsilon_k\}$ and $\{\gamma_k\}$ and
 the constants $(\delta,\delta_g,M,B$, $m,b,\varsigma_1,\varsigma_2)$.
\end{lemma}

\paragraph{Proof of Theorem \ref{thm1}}
\begin{proof} \cite{DongZLiang2022} proved the result (\ref{thetacon2}) for the adaptive 
Langevinized ensemble Kalman filter (LEnKF) algorithm, which is equivalent to an adaptive pre-conditioned SGLD algorithm. Since the SGLD algorithm (\ref{ASGLDeq}) is a special case 
of the pre-conditioned SGLD algorithm, this theorem can be proved by following the proof of 
\cite{DongZLiang2022} with minor modifications. 
We omit the details of the proof.
\end{proof}

\begin{remark} \label{remrate} 
Regarding the convergence rate of $\bw_k$, we note that \cite{DongZLiang2022} gives an explicit form of $\xi$. Refer to Theorem A.1 of \cite{DongZLiang2022} for the detail. 
\end{remark}

\subsection{Proof of Theorem \ref{thm2}}

Let $T_k = \sum_{i=1}^{k}\epsilon_{k}$. 
Let $\mu_{D,T_k}=\mathcal{L}(\bz_k|\bw_k,D)$ denote the probability law of $\bz_k$ at iteration $k$ of Algorithm \ref{EFIalgorithm}, let $\nu_{D,T_k} 
=\mathcal{L}(\bz(T_k)|\bw^*,D)$ denote the probability law of a continuous time diffusion process, 
and let $\pi^*=\pi_D(\bz|\bw^*)$. 

\begin{lemma} \label{KLbound}  Suppose the conditions of Lemma \ref{thm1A} hold. Then there exist 
some constants $C_0>0$ and $C_1>0$ such that 
\begin{equation} \label{KLboundA}
D_{\mathrm{KL}}(\mu_{D,T_k}\|\nu_{D,T_k}) \leq (C_0 \delta_g+C_1 \gamma_1) T_k,
\end{equation}
where $D_{\mathrm{KL}}|(\cdot\|\cdot)$ denotes the Kullback-Leibler divergence. 
\end{lemma}
\begin{proof}
Our proof follows the proof of Lemma 7 in \cite{raginsky2017non} closely, but changing from a constant learning rate sequence to a decaying learning rate sequence. 
Similar developments can also be found in  Appendix D of \cite{Zhang2020CyclicalSG}.

 Let $\bar{T}(s) = T_k$ for $T_k \leq s < T_{k+1}, k = 1,\dots, \infty$. Conditioned on $D$ and $\bw$, $\{\bZ_k\}$ forms a Markov process. Consider the following continuous-time interpolation of this process: 
\begin{equation}\label{eq: time-homogeneous Markov process}
    \bar{\bZ}(t)=\bZ_{0}-\int_{0}^{t} g\left(\bar{\bZ}(\bar{T}(s)),\bar{\bw}(s), \bar{u}_D(s)\right) \mathrm{d} s+\sqrt{2} \int_{0}^{t} \mathrm{~d} B(s), \quad t \geq 0,
\end{equation}
where $\bar{\bw}(s)= \bw_{k}$ for  $T_k \leq s < T_{k+1}$, and $\{B(s)\}_{s\geq 0}$ is the standard Brownian motion in $\mathbb{R}^{d_z}$.
Note that, for each $k$, $\bar{\bZ}(T_k)$ and $\bZ_k$ have the same probability law $\mu_{D,T_k}$. Moreover, by a result of \cite{gyongy1986mimicking}, the process $\bar{\bZ}(t)$ has the same one-time marginals as the It\^{o} process
\begin{equation}
\bZ'(t)=\bZ_{0}-\int_{0}^{t} g_{D,s}(\bZ'(s)) \mathrm{d} s+\sqrt{2} \int_{0}^{t} \mathrm{~d} B(s),
\end{equation}
where 
\begin{equation}
    g_{D,s}(z):=\mathbb{E}\left[g\left(\bar{\bZ}(\bar{T}(s)),\bar{\bw}(s),\bar{u}_D(s)\right)  \mid \bar{\bZ}(s)=z\right].
\end{equation} 
Crucially, $\bZ'(t)$ is a Markov process, while $\bar{\bZ}(t)$ is not. 

Let $\bP_{\bZ'}^{t}:=\mathcal{L}(\bZ'(s): 0 \leq s \leq t \mid D, \bar{\bw}(s))$ and $\bP_{\bZ}^{t}:=\mathcal{L}(\bZ(s): 0 \leq s \leq t \mid D, \bw^*)$.  The Radon-Nikodym derivative of $\bP_{\bZ}^{t}$ w.r.t. $\bP_{Z'}^{t}$ is given by the Girsanov formula 
\begin{equation}\label{eq: RN derivative}
\begin{aligned}
\frac{\mathrm{d} \bP_{\bZ}^{t}}{\mathrm{~d} \bP_{\bZ'}^{t}}(\bZ')&=\exp \left\{\frac{1}{2} \int_{0}^{t}\left(\nabla F_{D}(\bZ'(s), \bw^*)-g_{D,s}(\bZ'(s))\right)^{*} \mathrm{~d} B(s)\right.\\
&\left.-\frac{1}{4} \int_{0}^{t}\left\|\nabla F_{D}(\bZ'(s), \bw^*)-g_{D, s}(\bZ'(s))\right\|^{2} \mathrm{~d} s \right\}.
\end{aligned}
\end{equation}
Using (\ref{eq: RN derivative}) and the martingale property of the It\^{o} integral, we have
\begin{equation}
\begin{aligned} 
D_{\mathrm{KL}}\left(\bP_{\bZ'}^{t} \| \bP_{\bZ}^{t}\right) &=-\int \mathrm{d} \bP_{\bZ'}^{t} \log \frac{\mathrm{d} \bP_{\bZ}^{t}}{\mathrm{~d} \bP_{\bZ'}^{t}}\\ 
&=\frac{1}{4} \int_{0}^{t} \mathbb{E}\left\|\nabla F_{D}(\bZ'(s), \bw^*)-g_{D, s}(\bZ'(s))\right\|^{2} \mathrm{~d} s \\ 
&=\frac{1}{4} \int_{0}^{t} \mathbb{E}\left\|\nabla F_{D}(\bar{\bZ}(s),\bw^*)-g_{D, s}(\bar{\bZ}(s))\right\|^{2} \mathrm{~d} s,
\end{aligned}
\end{equation}
where the last line follows from the fact that $\mathcal{L}(\bar{\bZ}(s)) = \mathcal{L}(\bZ'(s))$ for each $s$. 

Recall that $T_0 = 0$ and $T_k = \sum_{i=1}^{k}\epsilon_k$.
Then, by the definition of $g_{D, s}$, Jensen's inequality and the $M$-smoothness of $F_{D}$, we have
\begin{equation}\label{eq: distance}
\begin{aligned}
D_{\mathrm{KL}}\left(\bP_{\bZ'}^{T_k} \| \bP_{\bZ}^{T_k}\right)&=\frac{1}{4} \sum_{j=0}^{k-1} \int_{T_{j}}^{T_{j+1}} \mathbb{E}\left\|\nabla F_{D}(\bar{\bZ}(s),\bw^*)-g_{D, s}(\bar{\bZ}(s))\right\|^{2} \mathrm{~d} s \\ 
&\leq \frac{1}{2} \sum_{j=0}^{k-1} \int_{T_{j}}^{T_{j+1}} \mathbb{E}\left\|\nabla F_{D}(\bar{\bZ}(s), \bw^*)-\nabla F_{D}(\bar{\bZ}(\bar{T}(s)), \bw^*)\right\|^{2} \mathrm{~d} s \\ 
&\qquad+\frac{1}{2} \sum_{j=0}^{k-1} \int_{T_j}^{T_{j+1}} \mathbb{E}\left\|\nabla F_{D}(\bar{\bZ}(\bar{T}(s)), \bw^*)-g\left(\bar{\bZ}(\bar{T}(s)),\bar{\bw}(s), \bar{u}_D(s) \right)\right\|^{2} \mathrm{~d} s \\ 
&\leq \frac{ M^{2}}{2} \sum_{j=0}^{k-1} \int_{T_{j}}^{T_{j+1}} \mathbb{E}\|\bar{\bZ}(s)-\bar{\bZ}(\bar{T}(s))\|^{2} \mathrm{~d} s \\ 
&\qquad+\frac{1}{2} \sum_{j=0}^{k-1} \int_{T_j}^{T_{j+1}} \mathbb{E}\left\|\nabla F_{D}(\bar{\bZ}(\bar{T}(s)), \bw^*)-g\left(\bar{\bZ}(\bar{T}(s)),\bar{\bw}(s), \bar{u}_D(s) \right)\right\|^{2} \mathrm{~d} s.
\end{aligned}
\end{equation}
To estimate the first summation on the right side of (\ref{eq: distance}), we consider some $s\in[T_{j},T_{j+1})$. By (\ref{eq: time-homogeneous Markov process}), we have
\begin{equation}
\begin{aligned} 
& \bar{\bZ}(s)-\bar{\bZ}(T_j) =-(s-T_j) g\left(\bZ_j, \bw_j, u_{D,j}\right)+\sqrt{2}(B(s)-B(T_j)) \\ 
&=-(s-T_j) \nabla F_{D}\left(\bZ_j, \bw^*\right)+(s-T_{j})\left(\nabla F_{D}\left(\bZ_j, \bw^*\right)-g\left(\bZ_j,  \bw_j, u_{D,j}\right)\right)+\sqrt{2}(B(s)-B(T_j)).
\end{aligned}
\end{equation}
Therefore, by (\ref{boundF}), Assumption \ref{ass3}, Lemma \ref{l2bound} and Lemma \ref{thm1A}, we have 
\begin{equation}
\begin{aligned}
&\mathbb{E} \| \bar{\bZ}(s)-\bar{\bZ}(T_j) \|^{2}  \\ 
&\leq 3 \epsilon_{j+1}^{2} \mathbb{E}\left\|\nabla F_{D}\left(\bZ_j, \bw^*\right)\right\|^{2}+3 \epsilon_{j+1}^{2} \mathbb{E}\left\|\nabla F_{D}\left(\bZ_{j}, \bw^*\right)-g\left(\bZ_{j}, \bw_{j}, u_{D,j}\right)\right\|^{2}+6 \epsilon_{j+1} d_{z} \\ 
&\leq 3 \epsilon_{j+1}^{2} \mathbb{E}\left\|\nabla F_{D}\left(\bZ_j, \bw^*\right)\right\|^{2}+
6 \epsilon_{j+1} d_{z} \\ 
& + 6 \epsilon_{j+1}^{2} \left( \mathbb{E}\left\|\nabla F_{D}\left(\bZ_{j}, \bw^*\right)
-\nabla F_{D}\left(\bZ_{j}, \bw_j\right)\right\|^2+ \mathbb{E} \left\| \nabla F_{D}\left(\bZ_{j}, \bw_j\right)-g\left(\bZ_{j}, \bw_{j}, u_{D,j}\right)\right\|^{2}\right) \\ 
&\leq 9 \epsilon_{j+1}^{2}\left(M^{2} \mathbb{E}\left\|\bZ_{j}\right\|^{2}+B^{2}\right)+6 d_{\bz} \epsilon_{j+1} + 6 \epsilon_{j+1}^2 (\xi M^2 \gamma_j + \delta_g(M^2 G_z+\xi M^2 \gamma_j+B^2)) \\ 
&\leq 9 \epsilon_{j+1}^{2}\left(M^2G_z + B^2 \right)+6 d_{\bz}\epsilon_{j+1}  + 6 \xi M^2\gamma_j \epsilon_{j+1}^2+ 6 \delta_g \epsilon_{j+1}^2 (M^2 G_z+ \xi M^2 \gamma_j+B^2).
\end{aligned}
\end{equation}
Consequently, we can bound the first summation on the right-hand side of (\ref{eq: distance}) 
as follows: 
\begin{equation}\label{eq: first term}
\begin{aligned}
&\sum_{j=0}^{k-1} \int_{T_{j}}^{T_{j+1}} \mathbb{E}\|\bar{\bZ}(s)-\bar{\bZ}(\bar{T}(s))\|^{2} \mathrm{~d} s \\ 
&\leq 9\left(M^2G_z + B^2\right) \sum_{j=0}^{k-1}  \epsilon_{j+1}^{3} +6 d_{z} \sum_{j=0}^{k-1} \epsilon_{j+1}^2 + 6\xi M^2\sum_{j=0}^{k-1} \gamma_j\epsilon_{j+1}^{3} + 
 6 \delta_g \sum_{j=0}^{k-1}\epsilon_{j+1}^3 (M^2 G_z+ \xi M^2 \gamma_j+B^2).
\end{aligned}
\end{equation}
Similarly, by Lemma \ref{thm1A}, the second summation on the right-hand side of (\ref{eq: distance}) can be bounded as follows:
\begin{equation}\label{eq: second term}
\begin{aligned}
& \sum_{j=0}^{k-1} \int_{T_j}^{T_{j+1}} \mathbb{E}\left\|\nabla F_{D}(\bar{\bZ}(\bar{T}(s)), \bw^*)-g\left(\bar{\bZ}(\bar{T}(s)),\bar{\bw}(s),\bar{u}_{D,s}\right)\right\|^{2} \mathrm{~d} s \\ 
&= \sum_{j=0}^{k-1} \epsilon_{j+1} \mathbb{E}\left\|\nabla F_{D}\left(\bZ_{j}, \bw^*\right)-g\left(\bZ_{j}, \bw_{j},u_{D,j}\right)\right\|^{2} \\ 
& \leq 2 \sum_{j=0}^{k-1} \epsilon_{j+1} \left(\mathbb{E} \left \| \nabla F_{D}\left(\bZ_{j}, \bw^*\right)-
F_{D}\left(\bZ_{j}, \bw_j\right) \right\|^2+ \mathbb{E} \left\| F_{D}\left(\bZ_{j}, \bw_j\right)-
g\left(\bZ_{j}, \bw_{j},u_{D,j}\right)\right\|^{2} \right ) \\
&\leq 2 \xi M^2\sum_{j=0}^{k-1}\gamma_j \epsilon_{j+1}+ 2 \delta_g \sum_{j=0}^{k-1} \epsilon_{j+1}(M^2 G_z+\xi M^2 \gamma_j+B^2). 
\end{aligned}
\end{equation}
Substituting Equations (\ref{eq: first term}) and (\ref{eq: second term}) into (\ref{eq: distance}), we obtain
\begin{equation}
\begin{split}
D_{\mathrm{KL}}\left(\mathbf{P}_{\bZ'}^{T_k} \| \mathbf{P}_{\bZ}^{T_k}\right) & \leq \frac{9}{2}\left(M^4G_z + M^2B^2\right) \sum_{j=0}^{k-1}\epsilon_{j+1}^{3} + 3 M^2 d_{z} \sum_{j=0}^{k-1} \epsilon_{j+1}^2 + 3\xi M^4 \sum_{j=0}^{k-1} \gamma_j\epsilon_{j+1}^{3}  
+ \xi M^2 \sum_{j=0}^{k-1}\gamma_j \epsilon_{j+1} \\
& + \delta_g \sum_{j=0}^{k-1} \epsilon_{j+1} (3M^2 \epsilon_{j+1}^2+1) (M^2 G_z+\xi M^2 \gamma_j+B^2).
\end{split}
\end{equation}
Since $\mu_{D, \bw_k,T_k}=\mathcal{L}(\bZ_{k}|D, \bw_k)$ and $\nu_{D, \bw^*, T_k}=\mathcal{L}(\bZ(t)|D, \bw^*)$, the data-processing inequality for the Kullback-Leibler divergence gives
\begin{equation}
\begin{aligned} 
& D_{\mathrm{KL}}\left(\mu_{D, \bw_k, T_k} \| \nu_{D, \bw^*, T_k}\right) 
\leq  D_{\mathrm{KL}}\left(\mathbf{P}_{\bZ'}^{T_k} \| \mathbf{P}_{\bZ}^{T_k}\right) \\ 
\leq & \frac{9}{2}\left(M^4G_z + M^2B^2\right) \sum_{j=0}^{k-1}\epsilon_{j+1}^{3} + 3 M^2 d_{z} \sum_{j=0}^{k-1} \epsilon_{j+1}^2 + 3\xi M^4 \sum_{j=0}^{k-1} \gamma_j\epsilon_{j+1}^{3}  
+ \xi M^2 \sum_{j=0}^{k-1}\gamma_j \epsilon_{j+1} \\
& + \delta_g \sum_{j=0}^{k-1} \epsilon_{j+1} (3M^2 \epsilon_{j+1}^2+1) (M^2 G_z+\xi M^2 \gamma_j+B^2) 
\\ 
& \leq (C_0 \delta_g + C_1 \gamma_1) T_k,
\end{aligned}
\end{equation}
for some constants $C_0>0$ and $C_1>0$. 
\end{proof}

\begin{assumption} \label{ass6}
The probability law $\mu_0$ of the initial hypothesis $\bw_0$ has a bounded and strictly positive density 
$p_0$ with respect to the Lebesgue measure on $\mathbb{R}_{d_{\bz}}$, and 
\[
\kappa_0:=\log \int_{\mathbb{R}^{d_{\bz}}} e^{\|\bw\|^2} p_0(\bw) d\bw < \infty. 
\]
\end{assumption}

\begin{lemma} \label{W2bound}  Suppose Assumption \ref{ass6} and the conditions of Lemma \ref{thm1A} hold. Then there exist 
some constants $\tilde{C}_0>0$ and $\tilde{C}_1>0$ such that 
\[
\mathbb{W}_2^2(\mu_{D,T_k}, \nu_{D,T_k}) \leq (\tilde{C}_0 \sqrt{\delta_g}+\tilde{C}_1 \sqrt{\gamma_1}) T_k^2,
\]
where $\mathbb{W}_2(\cdot,\cdot)$ denotes 2-Wasserstein distance. 
\end{lemma}
\begin{proof}
The proof of Lemma \ref{W2bound} follows that of Proposition 8 of \cite{raginsky2017non} closely. 
First, we apply Corollary 2.3 of \cite{bolley2005weighted} to get the inequality 
\begin{equation} \label{w2boundB}
\mathbb{W}_2^2(\mu_{D,T_k}, \nu_{D,T_k}) \leq C T_k \left( D_{\mathrm{KL}}(\mu_{D,T_k}, \nu_{D,T_k})+ 
\sqrt{ D_{\mathrm{KL}}(\mu_{D,T_k}, \nu_{D,T_k})} \right),
\end{equation}
for some constant $C$, for which we assume both $\mu_{D,T_k}$ and $\nu_{D,T_k}$ 
have finite second moments.  Further, by substituting (\ref{KLboundA}) into (\ref{w2boundB}), we can complete the proof. 
\end{proof}
 
\begin{lemma} \label{W2boundF}  Suppose Assumption \ref{ass6} and the conditions of Lemma \ref{thm1A} hold. Then there exist 
some constants $\hat{C}_0>0$, $\hat{C}_1>0$ and $\hat{C}_2$ such that 
\[
\mathbb{W}_2(\mu_{D,T_k}, \pi^*) \leq (\hat{C}_0 \delta_g^{1/4}+\hat{C}_1 \gamma_1^{1/4})T_k + \hat{C}_2 e^{-T_k/c_{LS}},
\]
where $c_{LS}$ denotes the logarithmic Sobolev constant of $\pi^*=\pi_D(\bz|\bw^*)$. 
\end{lemma}
The proof of Lemma \ref{W2boundF} follows that of Proposition 10 in \cite{raginsky2017non} closely, 
and it is thus omitted.

\section{Proof of Theorem \ref{thm3}} \label{sectproof2}

{\bf Notation}: We use $(x,y,z)$ to denote a generic observation in the dataset 
$(\bX_n,\bY_n,\bZ_n)$.

 \begin{assumption} \label{ass7} The EFI network satisfies the conditions:
 \begin{itemize}
     \item[(i)] The parameter space $\mathcal{W}_n$ (of $\bw_n$) is convex and compact. 
     \item[(ii)] \textcolor{black}{$\mathbb{E} (\log\pi(y,z|x,\bw_n))^2 < \infty$ for any $\bw_n\in \mathcal{W}_n$.}

 \end{itemize}
 \end{assumption} 



\begin{assumption}\label{ass8} For any positive integer $n$, the following conditions hold:
\begin{enumerate}  
    \item[(i)] \textcolor{black}{$Q^*(\bw_n)$ is continuous in $\bw_n$ and uniquely maximized at some point $\bw_n^b$;}
    \item[(ii)] \textcolor{black}{for any $\epsilon>0$, $sup_{\bw\in\mathcal{W}_n\backslash B(\epsilon)}Q^*(\bw_n)$ exists, where $B(\epsilon)=\{\bw_n:\|\bw_n-\bw_n^b\|<\epsilon\}$, and $\delta=Q^*(\bw_n^b)-sup_{\bw_n \in \mathcal{W}_n \backslash B(\epsilon)}Q^*(\bw_n)>0$.}
\end{enumerate}
\end{assumption}

\paragraph{Proof of Lemma \ref{lemma:equivalent}}  


{\color{black}
\begin{proof} 
Suppose that $\pi(\bw_n|\bX_n,\bY_n)$ has a different maximizer that minimizes $D_{KL}(\bw_n)$ as well. Let $\bw_n^{\dag}$ denote such a maximizer, which is different from 
$\bw_n^*$ but maintains $\bZ_n^* \sim \pi(\bz|\bX_n,\bY_n, \bw_n^{\dag})$. 
For $\bw_n^{\dag}$, similar to (\ref{Q2eq}), we have 
\begin{equation} \label{Q3eq}
 \begin{split}
 \widetilde{\mG}(\bw_n|{\bw}_n^{\dag})& :=\frac{1}{n}\int \log \pi(\bY_n,\bZ_n^*|\bX_n,\bw_n) d\pi(\bZ_n^*|\bX_n,\bY_n,{\bw}_n^{\dag}) +\frac{1}{n} \log\pi(\bw_n) \\ 
 &=  \frac{1}{n} \Big\{ \log\pi(\bw_n|\bX_n,\bY_n) - \int \log \frac{\pi(\bZ_n^*|\bX_n,\bY_n,{\bw}_n^{\dag})}{ \pi(\bZ_n^*|\bX_n,\bY_n,\bw_n)} d \pi(\bZ_n^*|\bX_n,\bY_n,{\bw}_n^{\dag}) \\
 & + \int \log \pi(\bZ_n^*|\bX_n,\bY_n,{\bw}_n^{\dag}) d \pi(\bZ_n^*|\bX_n,\bY_n,{\bw}_n^{\dag})+c \Big\},\\
 \end{split}
 \end{equation}
 which, by the non-negativeness of the Kullback-Leibler divergence, implies that 
 $\bw_n^{\dag}$ is also the maximizer of  $\widetilde{\mG}(\bw_n|{\bw}_n^{\dag})$. 
 By (\ref{eq:sameloss2}), (\ref{QQeq}), and Assumption \ref{ass8}, we would have 
 \[
 \|\hat{\bw}_n^*-\bw_n^{\dag}\| \stackrel{p}{\to} 0, \quad \mbox{as $n\to \infty$},
 \]
 following the proof of Lemma 2 in \cite{SunLiang2022kernel}.
 This contradicts with the uniqueness of $\hat{\bw}_n^*$. 
 
Therefore, if $\hat{\bw}_n^*$ is unique, then  $\bw_n^*$ is unique. Subsequently, 
we have $\|\hat{\bw}_n^*-\bw_n^*\| \stackrel{p}{\to} 0$ as $n\to \infty$, which completes the proof. 
\end{proof}
}

 We follow \cite{SunSLiang2021} to make the following assumption on the DNN model embedded in the EFI network, for which the random errors $\bZ_n$ are assumed to be known. 
The sparse DNN has $H_n-1$ hidden layers, and each layer consists of $L_j$ hidden units. 
Specifically, we use $L_0$ and $L_{H_n}$ to denote the input and output dimensions, respectively. 
The weights and biases of the sparse DNN are specified by 
$\bw_n$,  and the structure of the sparse DNN is specified  by $\Lambda_n$, 
a binary vector corresponding to the elements of $\bw_n$. 

\begin{assumption} \label{ass9}
\begin{itemize}
 \item[(i)] The complete data $(x,y,z)$ is bounded by 1 entry-wisely, i.e. $(x,y,z)\in \Omega=[-1,1]^{p_n}$, and the density of $(x,y,z)$ is bounded in its support $\Omega$ uniformly with respect to $n$.
 
 \item[(ii)] The underlying true sparse DNN $(\tilde{\bw}_n^*,\tilde{\Lambda}_n^*)$ satisfies the following conditions:  
 \begin{itemize}
 \item[(ii-1)] The network structure satisfies: 
 $r_n H_n\log n+r_n \log\overline{L}+s_n\log p_n\leq C_0 n^{1-\varepsilon}$, where $0<\varepsilon<1$ is a small constant, 
 $r_n$ denotes the connectivity of $\tilde{\Lambda}_n^*$, $\overline{L}=\max_{1\leq j\leq H_n-1} L_j$ denotes the maximum hidden layer width,  and $s_n$ denotes the input dimension of $\tilde{\Lambda}_n^*$.
 
 \item[(ii-2)] The network weights are polynomially bounded:  $\|\tilde{\bw}_n^*\|_{\infty}\leq E_n$, where 
  $E_n=n^{C_1}$ for some constant $C_1>0$. 
 \end{itemize}
 \item[(iii)] The activation function $\psi$ used in the DNN is Lipschitz continuous with a Lipschitz constant of 1.
 
 \item[(iv)] The mixture Gaussian prior (\ref{mixtureprior}) satisfies the conditions:
 $\rho_n = O( 1/\{K_n[n^{H_n}(\overline{L} p_n)]^{\tau'}\})$ for some constant $\tau'>0$,
  $E_n/\{H_n\log n+\log \overline{L}\}^{1/2}  \lesssim \sigma_{1,n} \lesssim n^{\alpha'}$ for some 
  constant $\alpha'>0$, and 
$\sigma_{0,n}  \lesssim \min\big\{ 1/\{\sqrt{n} K_n (n^{3/2} \sigma_{1,0}/H_n)^{H_n}\}$, 
 $1/\{\sqrt{n} K_n (n E_n/H_n)^{H_n}\} \big\}$, where $K_n=\sum_{h=1}^{H_n} (L_{h-1} \times L_h+L_h)$ denotes the total number of parameters  of the fully connected DNN. 

 \item[(v)] \textcolor{black}{For the normal regression case, $y=f(x,\btheta_0)+\sigma z$ with 
 $z\sim N(0,1)$, the function $f(x,\btheta_0)$ is Lipschitz continuous with respect to $\btheta_0$; and for the logistic regression case,  $\log(P(Y=1)/(1-P(Y=1)))=\mu(x,\btheta)$, the logit link function $\mu(x,\btheta)$ is Lipschitz continuous with respect to $\btheta$. }  
 \end{itemize}
 \end{assumption}

\begin{remark} \label{rem:Ksize}
If we further assume that the exponent $0\leq C_1< \frac{1}{2}$ and the connectivity $r_n=O(n^{\zeta'})$ for some $0<\zeta'<\frac{1}{2}-C_1-\varepsilon'$ and $0<\varepsilon'<\frac{1}{2}-C_1-\zeta'$.
Then, based on the proof of  Theorem 1  and the followed remark in  \cite{Liang2018missing}, it is easy to figure out that $K_n$ is allowed to increase with the sample size $n$ in an exponential rate: $K_n \prec \exp(n^{2\varepsilon'})$.
 By the arguments provided in \cite{SunSLiang2021}, the sparse DNN approximation under the above assumptions is achievable for quite a few classes  of functions, such as 
 bounded $\alpha$-H\"older smooth functions \cite{Schmidt-Hieber2017Nonparametric},  piecewise smooth functions with fixed input dimensions \cite{petersen2018optimal}, and the functions that can be represented by an affine system \cite{bolcskei2019optimal}.
 \end{remark} 

 \begin{remark} \label{boundZ} 
\textcolor{black}{Assumption \ref{ass9}-(i) restricts $\Omega$, the domain of the complete data $(x, y, z)$, to a bounded set $[-1, 1]^{p_n}$. To satisfy this condition, we can add a data transformation/normalization layer to the DNN model, ensuring that the transformed input values fall within the set $\Omega$.
In particular, the transformation/normalization layer can form an 1-1 mapping and contain no tuning parameters. For example, when dealing with the standard Gaussian random variable, we can transform it to be uniform over (0,1) via the probability integral transformation $\Phi(z)$, where $\Phi(\cdot)$ denotes the CDF of the standard Gaussian random variable.   }
 \end{remark}




For the EFI network, we define
\begin{equation} \label{log-posterioreq}
h_{n}(\bw_n)=\frac{1}{n}\log \pi(\bY_n,\bZ_n^*|\bX_n,\bw_n)+\frac{1}{n}\log \pi(\bw_n), 
\end{equation}
where $\bZ_n^*$ is the true random errors realized in the data $(\bX_n,\bY_n)$ and it is thus independent of $\bw_n$. 
Then the posterior density of $\bw_n$ is given by 
 $\pi(\bw_n|\bX_n,\bY_n,\bZ_n^*) = \frac{ e^{nh_{n}(\bw_n)}}{\int e^{nh_{n}(\bw_n)}d\bw_n}$ 
and, for a function $b(\bw_n)$, the posterior expectation is given
by $\frac{\int b(\bw_n)e^{nh_{n}(\bw_n)}d\bw_n}{\int e^{nh_{n}(\bw_n)}d\bw_n}$.
Recall that we have  defined $\hat{\bw}_n^*=\arg\max_{\bw_n} \pi(\bw_n|\bX_n,\bY_n,\bZ_n^*)$, which is also the global maximizer of $h_n(\bw_n)$. 
Let $B_{\delta}(\bw_n)$ denote an Euclidean ball of radius $\delta$
centered at $\bw_n$. Let $h_{i_{1},i_{2}, \dots,i_{d}}(\bw_n)$
 denote the $d$-th order partial derivative $\frac{\partial^{d}h(\bw_n)}{\partial w_n^{i_{1}}\partial w_n^{i_{2}}\cdots\partial w_n^{i_{d}}}$, 
 let $H_{n}(\bw_n)$ denote the Hessian matrix of $h_{n}(\bw_n)$, let
 $h_{ij}$ denote the $(i,j)$-th component of the Hessian matrix, and  
 let $h^{ij}$ denote the $(i,j)$-component of the inverse of the Hessian matrix. 
 Recall that $\tilde{\Lambda}^{*}$  denotes the set of indicators for the 
 connections of the true sparse DNN,  
 $r_n$ denotes the size of the true sparse DNN, 
 and $K_n$ denotes the size of the fully connected DNN. 
  
\begin{assumption} \label{ass10}
There exist positive numbers $\epsilon$, $M$, and $n_0$ such that for any $n>n_0$,  the function $h_n(\bw_n)$ in (\ref{log-posterioreq}) satisfies the following conditions:
\begin{enumerate}
 \item[(i)]  $|h_{i_1,\dots,i_d}(\hat{\bw}_n^*)|<M$ hold for any $\bw_n\in B_{\epsilon}(\hat{\bw}_n^*)$ and any $1\leq i_{1},\dots,i_{d}\leq K_{n}$, where $3 \leq d \leq 4$. 
 
\item[(ii)] $|h^{ij}(\hat{\bw}_n^*)|<M$ if $\tilde{\Lambda}_{n,i}^*=\tilde{\Lambda}_{n,j}^*=1$ and $|h^{ij}(\hat{\bw}_n)| = O(\frac{1}{K_{n}^{2}})$ otherwise, where $\tilde{\Lambda}_{n,i}^*$ denotes the $i$-th element of $\tilde{\Lambda}_{n}$.

\item[(iii)] $\det(-\frac{n}{2\pi}H_{n}(\hat{\bw}_n))^{\frac{1}{2}}\int_{\mathbb{R}^{K_{n}}\setminus B_{\delta}(\hat{\bw}_n)}
e^{n(h_{n}(\bw_n)-h_{n}(\hat{\bw}_n))}d\bw_n=O(\frac{r_{n}^{4}}{n})=o(1)$ for any 
$0<\delta<\epsilon$.
\end{enumerate}
\end{assumption}

 Assumption \ref{ass10}-(i)\&(iii) are typical conditions for Laplace approximation, see e.g., \cite{geisser1990validity}. Assumption \ref{ass10}-(ii) requires the inverse Hessian to 
 have very small values for the elements corresponding to the false connections. Refer to 
 \cite{SunSLiang2021} for its justification.

\paragraph{Proof of Theorem \ref{thm3}} 
\begin{proof} 
\textcolor{black}{As discussed in Section \ref{EFIsect}, we have the likelihood function for the EFI network as 
\[
\pi(\bY_n|\bX_n,\bZ_n,\bw) = C e^{-\lambda U_n(\bZ_n,\bw;\bX_n,\bY_n)}.
\]
In the context of this proof, we assume that $\bZ_n=(z_1,z_2,\ldots,z_n)^T$ is known. 
Additionally, we use $d_t(p,p^*)=t^{-1} (\int p^*(p^*/p)^t -1)$ to denote a divergence measure for two 
distributions $p$ and $p^*$. It is easy to see that 
$d_t$ converges to the KL-divergence as $t\downarrow 0$.  
}

\textcolor{black}{
\paragraph{Normal Regression}
For the normal linear/nonlinear regression, we essentially have
 the energy function:  
 \begin{equation} \label{normLosseq1}
 U_n(\bZ_n,\bw;\bX_n,\bY_n)= \sum_{i=1}^n \|y_i-f(x_i,\btheta_0)-\sigma z_i\|^2,
 \end{equation}
as $\lambda \to \infty$, where $\btheta:=(\btheta_0,\log(\sigma))=G(x_i,y_i,z_i,\bw)$ is a constant function    
over the observations $\{(y_i,x_i,z_i): i=1,2,\ldots, n\}$. 
Therefore, as $\lambda \to \infty$, (\ref{normLosseq1}) enforces 
the output $\btheta$ of the DNN to satisfy the relationship:
\[
y_i=f(x_i,\btheta_0)+\sigma z_i, \quad i=1,2,\ldots,n, 
\]
i.e., $y \sim N(f(x,\btheta_0),\sigma^2)$. 
A direct calculation shows that the divergence $d_1(\cdot,\cdot)$ of two Gaussian distributions $p(x):=N(f(x,\btheta_0), \sigma^2)$ 
and $q(x):=N(f(x,\btheta_0'), \varsigma^2)$ is given by
\[
d_1(q,p)=\frac{\varsigma^2/\sigma}{\sqrt{2 \varsigma^2-\sigma^2}} e^{ \frac{ \|f(x,\btheta_0)-f(x,\btheta_0')\|^2}{2\varsigma^2-\sigma^2}}-1,
\]
provided that $2 \varsigma^2-\sigma^2>0$. 
The divergence $d_1(\cdot,\cdot)$ is a function of the two factors $\|f(x,\btheta_0)-f(x,\btheta_0')\|^2$
and $|\log (\sigma)-\log (\varsigma)|$. In particular, if both the factors goes to 0, then 
$d_1(\cdot,\cdot)$ goes to 0.
Therefore, to bound the value of $d_1(\cdot,\cdot)$, one can bound 
\[
\|f(x,\btheta_0)-f(x,\btheta_0')\|^2+|\log (\sigma)-\log (\varsigma)|^2 = O(\|\btheta-\btheta'\|^2) =O(\|G(x,y,z,\bw)-G(x,y,z,\bw')\|^2),  
\]
provided that $f(x,\btheta_0)$ is Lipschitz continuous with respect to $\btheta_0$, where 
$\btheta'= G(x,y,z,\bw')$ and $\bw'$ denotes the corresponding DNN weights. 
This result implies that as $\lambda\to \infty$, 
the posterior consistency for the DNN model in the EFI network can be studied as for a conventional normal regression DNN model with input variables $(x,y,z)$ 
and the output variable $\btheta$, provided that $f(x,\btheta_0)$ is Lipschitz continuous with respect to $\btheta_0$. Therefore, by Theorem 2.1 of \cite{SunSLiang2021}, the posterior consistency holds for the DNN model  under Assumption \ref{ass9}. } 

\textcolor{black}{
\paragraph{Logistic Regression} 
For logistic regression,  the reasoning is similar. As implied by (\ref{penalty_logistic1}), we essentially have the following probability mass function for a generic observation $(x,y,z)$:
\begin{equation} \label{logseq}
p_{\lambda}(y|z, x, \btheta)  \propto \exp\{-\lambda \rho((z-\mu)(2y-1))\},
\end{equation}
where $\btheta=G(x,y,z,\bw)$ is a constant function  over the observations $\{(y_i,x_i,z_i): i=1,2,\ldots, n\}$, and $\mu=\mu(x,\btheta)=x^T \btheta$ for the linear case. 
As $\lambda\to \infty$, the two events 
$\{Z < \mu\}$ and $\{Y=1\}$ are asymptotically  equivalent, i.e., 
$\{Z < \mu\} \Longleftrightarrow \{Y=1\}$ with probability 1. 
Due to the monotonicity of the function $\frac{1}{1+e^{-z}}$,   
$\left\{\frac{1}{1+e^{-Z}}< \frac{1}{1+e^{-\mu(x,\btheta)}} \right \} \Longleftrightarrow \{Y=1\}$ with probability 1 as $\lambda \to \infty$.
Furthermore,  since $Z$ follows the logistic distribution,  $\frac{1}{1+e^{-Z}}$ is uniform on (0,1).  Therefore, as $\lambda \to \infty$, 
(\ref{logseq}) enforces the output $\btheta$ of the DNN model to 
 satisfy the following relationship:  
\begin{equation} \label{logseq2}
\mu(x,\btheta)=\log(P(Y=1)/(1-P(Y=1))). 
\end{equation} 
Following the calculation in \cite{Liang2018BNN}, the divergence $d_1(\cdot,\cdot)$ (up to a  multiplicative constant) of two logistic  distributions, with respective logit link functions  $\mu(x,\btheta)$ and $\mu(x, \btheta')$, is given by 
\[
\|\mu(x,\btheta)-\mu(x,\btheta')\|^2= O(\|\btheta-\btheta'\|^2)=O(\|G(x,y,z,\bw)-G(x,y,z,\bw')\|^2),  
\]
provided that $\mu(x,\btheta)$ is Lipschitz continuous with respect to $\btheta$, where 
$\btheta'= G(x,y,z,\bw')$ and $\bw'$ denotes the corresponding DNN weights.  
This result implies that as $\lambda\to \infty$, the posterior consistency for the DNN model in the EFI network can be studied as for a conventional logistic regression 
DNN model with input variables $(x,y,z)$ and the output variable $\btheta$, provided that the logit link function $\mu(x,\btheta)$ is Lipschitz continuous with respect to $\btheta$. Therefore,  by Theorem 2.1 of \cite{SunSLiang2021}, the posterior consistency holds for the EFI network under Assumption \ref{ass9}.
}

Furthermore, by Assumption \ref{ass7}, the parameter space $\mathcal{W}_n$ is compact and convex.  
Therefore, for any bounded function $b(\bw_n)$, the posterior mean $\mathbb{E} (b(\bw_n))$ is a consistent 
 estimator of $b(\tilde{\bw}_n^*)$ under posterior consistency. 
For the inverse mapping estimator $\hat{g}(x,y,z,\bw_n)$,  by Assumption \ref{ass7}-(i) and 
Assumption (\ref{ass9})-(i), it is bounded and 
\[
|\hat{g}_{i_1,\dots,i_d}(x,y,z,\bw_n)|= \left
|\frac{\partial^{d} \hat{g}(x,y,z,\bw_n)}{\partial\bw_n^{i_{1}}\partial\bw_n^{i_{2}}\cdots\partial\bw_n^{i_{d}}} \right|<M,
\]
holds for 
some constant $M$, for any $1 \leq d \leq 2$ and $1\leq i_{1},\dots,i_{d}\leq K_{n}$. 
Then, under Assumption \ref{ass10} and by Theorem 2.3 of \cite{SunSLiang2021},  $\hat{g}(x,y,z,\hat{\bw}_n^*)$ (as an approximator to the posterior mean $\mathbb{E} g(x,y,z,\bw_n)$) 
 forms a consistent estimator of $\btheta^*$.

 Finally, by Lemma  \ref{lemma:equivalent}, $\|\hat{\bw}_n^*-\bw_n^*\| \stackrel{p}{\to} 0$ holds,  which implies $\hat{g}(x,y,z,\bw_n^*)$ is also  
 a consistent estimator of $\btheta^*$. This completes the proof. 
\end{proof}

\section{Derivation of EFD for a Regression Example} \label{EFDproof}

Consider the linear regression model as defined in equation  (\ref{structeq}), where $\bbeta\in \mathbb{R}^{p-1}$. For an illustrative purpose, we assume that $\sigma^2$ is known. We set  
\[
G(\bY_n, \bX_n, \bz) =\tilde{G}(\bY_n, \bX_n, \bz)= (\bX_n^T \bX_n)^{-1}\bX_n^T(\bY_n - \sigma \bz),
\]
and the energy function
\[
U_n(\bz)=\|\bY_n-f(\bX_n,\bz,G(\bY_n,\bX_n,\bz))\|^2.
\]
Let $\bR_n = I_n - \bX_n(\bX_n^T \bX_n)^{-1}\bX_n^T$, which is an idempotent matrix of rank $n-p+1$. Then 
\begin{equation}
\begin{split}
J(\bz) = & \bY_n-f(\bX_n,\bz,G(\bY_n,\bX_n,\bz)) \\
= & \bY_n - \bX_n G(\bY_n, \bX_n, \bz) - \sigma\bz \\
= & \bR_n(\bY_n - \sigma \bz).  
\end{split}
\end{equation}
Let $(\bv_1, \dots, \bv_{p-1})$ be the eigenvectors corresponding to the zero eigenvalues of $\bR_n$, i.e. $\bR_n \bv_i = {\bf 0} \in \mathbb{R}^n$ for $i=1,2,\ldots,p-1$. Let $(\bv_{p}, \dots, \bv_{n})$ be the eigenvectors corresponding to the nonzero eigenvalues of $\bR_n$. 
Let $\bV_1 = (\bv_1, \dots, \bv_{p-1}) \in \mathbb{R}^{n \times (p-1)}$ and $\bV_2 = (\bv_{p}, \dots, \bv_{n}) \in \mathbb{R}^{n \times (n-p+1)}$. Then it is clear that 
\begin{equation}
\{\bz: J(\bz) = 0\} = \left\{\frac{1}{\sigma}\bY_n - \bV_1 \bu: \bu \in \mathbb{R}^{p-1}\right\}.
\end{equation}
For any vector $\bz\in \mathbb{R}^n$, 
we can write down the exact form of the decomposition in (\ref{decompeq}) as:
\begin{equation}
\label{decompeq_linear}
\bz = \frac{1}{\sigma}\bY_n - \bV_1 \bu - \bV_2 \bt, 
\end{equation}
where $\bu\in \mathbb{R}^{p-1}$ and $\bt \in \mathbb{R}^{n-p+1}$. Then, for $U_n(\bz) = \|J(\bz) \|^2$, we have
\begin{equation}
\nabla^{2}_{\bt} U_n(\bz) =2 \bV_2^T\bR_n^T \bR_n \bV_2.
\end{equation}
Note that 
\begin{equation}
\text{rank}(\nabla^{2}_{\bt} U_n(\bz)) = \text{rank}(\bR_n \bV_2) = \text{rank}(\bR_n(\bV_1, \bV_2)) = \text{rank}(\bR_n) = n-p+1.
\end{equation}
Therefore, $\det \nabla^{2}_{\bt} U_n(\bz)$ is a positive constant. 
Furthermore, for any $\bz \in \mathcal{Z}_{n}$, it can be written as $\bz = \frac{1}{\sigma}\bY_n - \bV_1 \bu$ for some $\bu \in \mathbb{R}^{p-1}$, and the limiting measure 
has the form 
\begin{equation}
 p_n^*(\bz|\bX_n,\bY_n)=   p_n^*( \frac{1}{\sigma}\bY_n - \bV_1 \bu|\bX_n,\bY_n) \propto \pi_0^{\otimes n}( \frac{1}{\sigma}\bY_n - \bV_1 \bu),
\end{equation}
which corresponds to a truncation of $\pi_0^{\otimes n}(\cdot)$ on the manifold $\mathcal{Z}_{n}$. Therefore,  
\[
\frac{1}{\sigma}\bY_n - \bV_1 \bu \sim N({\bf 0}, I_n).
\]
For  any $\bz \in \mathcal{Z}_{n}$, we set
\begin{equation}
\begin{split}
 \tilde{G}(\bY_n, \bX_n, \bz) & = (\bX_n^T \bX_n)^{-1}\bX_n^T (\bY_n-\sigma (\frac{1}{\sigma} \bY_n-\bV_1\bu)), \\
\end{split}
\end{equation}
and the resulting EFD is given by 
\[
\mu_n^*(\bbeta|\bY_n,\bX_n)=N(\hat{\bbeta}, \sigma^2 (\bX_n^T\bX_n)^{-1}),
\]
where $\hat{\bbeta} =  (\bX_n^T \bX_n)^{-1}\bX_n^T \bY_n$. 


\section{More Numerical Results} \label{moreresults}

\subsection{Nonlinear Regression}  \label{EFI:nonlinear}

Nonlinear least squares regression problems are intrinsically hard due to their complex energy landscapes, which may contains some saddle points, local minima or pathological curvatures. 
To test the performance of EFI on nonlinear regression, we took a benchmark dataset, Gauss2, at 
NIST Statistical Reference Datasets (\url{https://www.itl.nist.gov/div898/strd/nls/nls_main.shtml}),
which consists of 250 observations. 
The nonlinear regression function of the example is given by 
\begin{equation}
    y=\beta_1 \exp\{-\beta_2 x\}+ \beta_3 \exp\left\{-\frac{(x-\beta_4)^2}{\beta_5^2}\right\}+\beta_6 \exp\left\{-\frac{(x-\beta_7)^2}{\beta_8^2}\right\}+\epsilon:=f(x,\btheta)+\epsilon, \quad \epsilon\sim N(0,6.25),
\end{equation}
where $\btheta=(\beta_1,\beta_2,\ldots,\beta_8)$ denotes the vector of unknown parameters. The nonlinear regression function represents two slightly-blended Gaussian density 
curves on a decaying exponential baseline plus normally distributed zero-mean noise with known variance 6.25.  For this example, the ``best-available'' OLS solution has been given as shown in 
Table \ref{tab:Gauss2}, which was obtained using 128-bit precision and confirmed by at least two different algorithms and software packages using analytic derivatives.

The EFI method was applied to this example with the experimental settings given in Section \ref{setting:nonlinear} of this supplement. 
Table \ref{tab:Gauss2} compares the parameter estimates and confidence intervals by the OLS and EFI methods. For OLS, the confidence intervals are constructed by Wald's method with the estimates' standard deviations given in the website.  
For EFI, the parameter estimates are obtained by averaging the fiducial $\bar{\btheta}$-samples collected in the simulation, and the confidence intervals are constructed with 2.5\% and 97.5\% quantiles of the fiducial samples. Therefore, the EFI confidence intervals are not necessarily symmetric about the parameter estimates. The comparison shows that the EFI confidence intervals 
tend to be shorter than the OLS confidence intervals. 
More importantly, since EFI and OLS employ different objective functions, they actually converge to different solutions. This can be seen from the confidence intervals of $\beta_3$ resulting from the two methods,
 which have no overlaps. 

\begin{table}[htbp]
\caption{Parameter estimates and confidence intervals of the EFI and ``best-available'' OLS solutions for the Gauss2 example. }
\label{tab:Gauss2}
\vspace{-0.2in}
\begin{center}
\begin{adjustbox}{width=1.0\textwidth}
\begin{tabular}{crcccrcc} \toprule
   &\multicolumn{3}{c}{``Best-available'' OLS} & &  \multicolumn{3}{c}{EFI}\\  
  \cline{2-4} \cline{6-8} 
  Parameter  & Estimate & CI-width & 95\% CI & & Estimate & CI-width & 95\% CI \\ 
  \midrule
$\beta_1$ & 99.0183 & 2.1070 & (97.9649, 100.0718) && 98.8713 & 1.0243 &  (98.3466, 99.3710) \\
$\beta_2$ & 0.0110  & 0.0005 & (0.0107,0.0113)  && 0.0109  & 0.0004 &  (0.0108, 0.0111) \\
$\beta_3$ & 101.8802 & 2.3213 & (100.7196,103.0409) &&  99.2748 & 1.2274 & (98.6559, 99.8832) \\
$\beta_4$ & 107.0310 & 0.5883 &  (106.7368,107.3251) && 107.0377 & 0.6427 & (106.6962, 107.3389) \\
$\beta_5$ & 23.5786 & 0.8897 & (23.1338,24.0234)  && 23.5636 & 0.8447 &  (23.1306, 23.9753) \\
$\beta_6$ & 72.0456  & 2.4195 & (70.8358,73.2553)  &&  72.5515  & 0.8255 &  (72.1315, 72.9570) \\
$\beta_7$ & 153.2701 & 0.7631 & (152.8886,153.6516) && 153.2575 & 0.7788 & (152.8596, 153.6383) \\
$\beta_8$ & 19.5260  & 1.0355 &  (19.0082,20.0437)  && 19.6559  & 1.0776 & (19.1351, 20.2127) \\
\bottomrule
\end{tabular}
 \end{adjustbox}
\end{center}
\end{table}

To further explore the difference of the EFI and OLS solutions, we examined their fitting 
and residual plots in Figure \ref{fig:Gauss2}.
The right plot indicates that EFI tends to have larger 
residuals than OLS. A simple calculation shows that the OLS has a mean-squared-residuals of 4.99, while the EFI has a mean-squared-residuals of 5.72, which is closer to the ideal value 6.25. 
This comparison implies that OLS, which simply minimizes the sum of squared fitting errors, can lead to an overfitting issue even for this reasonably large dataset. EFI performs better in this regard by striking a balance between fitting errors and the likelihood of random errors, as discussed in Section 3.4 of the main text. This balance potentially results in a solution of higher fidelity.

For this example, we have also tried the Bayesian method, which leads to almost the same solution as OLS, as they essentially employ the same objective function. 
 
\begin{figure}[!htbp]
    \centering
    \includegraphics[width=0.95\textwidth,height=3.4in]{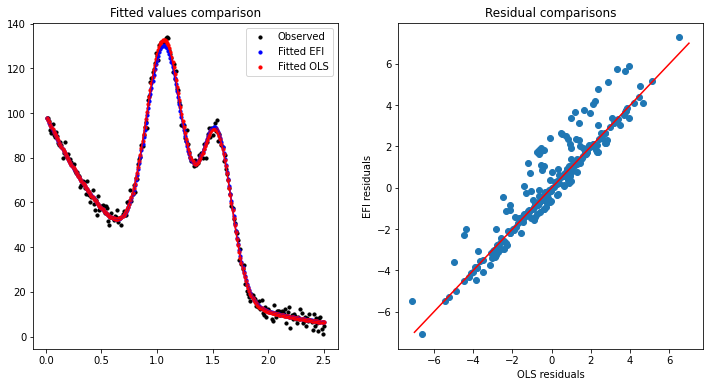}
    \caption{Comparison of the EFI solution with the ``best-available'' OLS solution
    for the Gauss2 example: (left) fitting curves and (right) scatter plot of residuals. }
    \label{fig:Gauss2}
\end{figure}

\subsection{Logistic Regression}  \label{EFI:logistic}

The EFI method can be easily extended to discrete statistical models via approximation or  transformation. For example, for the logistic regression, whose response variable $y_i \in \{0,1\}$ is discrete, making the fitting-error term in (\ref{energyfunction00}) 
 and (\ref{energyfunction11}) less well defined.  To address this issue, we define ReLU function:
$\rho(\delta)=\delta$ if $\delta>0$ and $0$ otherwise, and then replace the fitting-error term 
in (\ref{energyfunction00}) and (\ref{energyfunction11}) by   
   \begin{equation} \label{penalty_logistic1}
 \sum_{i=1}^n \rho\Big((z_i-x_i^T\hat{\btheta}_i)(2y_i-1) \Big),
   \end{equation}
   where $z_1,z_2,\ldots,z_n \stackrel{iid}{\sim} Logistic(0,1)$ with the CDF given by 
   $F(z)=1/(1+e^{-z})$. That is, $z_i$ represents the random error realized in the observation 
   $(x_i,y_i)$. Correspondingly, the fitted value $\tilde{y}_i$ is defined by
   $\tilde{y}_i= 1$ if $F(z_i) \leq F(x_i^T\hat{\btheta}_i)$ and 0 otherwise, where $F(z_i)\sim Uniform(0,1)$ by the probability integral transformation theorem.
   It is easy to see that (\ref{penalty_logistic1}) penalizes the cases $\{i: y_i \ne \tilde{y}_i\}$ and the resulting energy function satisfies 
   Assumption \ref{ass:hwang1a}.  
   In particular, we can have 
   $\Pi_n(\mathcal{Z}_{\check{U}_n})>0$ for this problem, 
   because each $z_i$ can take any value in an interval 
   $(-\infty,a]$ or $[b,\infty)$ (for some $a,b\in \mathbb{R}$) while maintaining
   the zero total-fitting-error given in (\ref{penalty_logistic1}). 
    

\begin{table}[!htbp]
\caption{Comparison of MLE and EFI for inference of logistic regression, 
 where coverage rate (confidence length) is reported for each parameter.}
\label{Table:logistic}
\vspace{-0.2in}
\begin{center}
\begin{tabular}{lccccccc} \toprule
 Method & $\theta_0$ &   $\theta_1$ & $\theta_2$  & $\theta_3$ & $\theta_4$   & Average \\ \midrule
 MLE & 0.94 (0.370)   & 0.96   (0.393) & 0.96 (0.389)  &  0.93 (0.390) & 0.96 (0.394) & 0.95
  \\ 
 EFI &  0.95 (0.378)   & 0.95 (0.390) & 0.95 (0.388) & 0.94 (0.388) & 0.97  (0.393)  & 0.952 \\ \bottomrule 
\end{tabular}
\end{center}
\vspace{-0.2in}
\end{table}


  We simulated 100 datasets from a logistic regression consisting of 4 
 covariates independently drawn from $N(0,1)$. The true regression coefficients were   
 $\btheta=(\theta_0,\theta_1,\ldots,\theta_4)=(1,1,1,-1,-1)$, including the 
 interpret $\theta_0$. The sample size of each dataset was $n=1000$. 
 The numerical results are summarized in Table \ref{Table:logistic}. 
 The comparison with the MLE results indicates the validity of EFI for statistical inference of logistic regression.

 For comparison, we applied GFI to this example by running the R package \textit{gfilogisreg} \citep{gfilogisregR}, but which did not produce results for this example due to a computational instability issue suffered by the package.
 For IM, we refer to \cite{Martin2015GIM}, where the likelihood function is used for inference of $\btheta$ and the confidence intervals are constructed by inverting the Monte Carlo hypothesis tests conducted on a lattice of grid points in $\Theta$. For example, if we take $50$ grid points in each dimension of $\Theta$ and simulate $1000$ samples at each grid point, then we need to simulate a total of $3.125\times 10^{11}$ samples. This is time consuming even for such a 5-dimensional problem.  

For multiclass logistic regression, similar to (\ref{penalty_logistic1}), 
the fitting-error term in (\ref{energyfunction00}) and (\ref{energyfunction11}) can be defined as 
\begin{equation} \label{multilogisticloss}
\sum_{i=1}^n \Big[   \sum_{j \ne m_i} \rho(x_i^T \hat{\btheta}_{i,j} - x_i^T \hat{\btheta}_{i,m_i}) +\rho(z_i-x_i^T \hat{\btheta}_{i,m_i}) \Big],
\end{equation}
where $m_i$ denotes the true class of  the training sample $x_i$, 
and $\hat{\btheta}_{i,j}$ denotes the parameter corresponding to class $j$ for the training sample $x_i$. 

\subsection{Semi-Supervised Learning}

Table \ref{Table:EFI_SSI} presents more examples for the semi-supervised learning. 

\begin{table}[htbp]
\caption{EFI results for different datasets, where the labels of 50\% training samples 
were removed in each run of the 5-fold cross validation. }
\label{Table:EFI_SSI}
\begin{center}
\begin{tabular}{cccccc} \toprule
 Dataset & $n$ & $p$  & Full & Labeled only & Semi \\ \midrule
 Divorce & 170 & 54 &   98.824$\pm$1.052 & 97.647$\pm$1.289  & 98.824$\pm$1.052 \\ 
 Diabetes & 520 & 16 &  89.615$\pm$1.032   & 88.462$\pm$1.088 & 88.846$\pm$1.668 \\
Breast Cancer & 699 & 9 &  96.52$\pm$0.661 & 95.942$\pm$0.485 & 96.232$\pm$0.518 \\
Raisin & 900 & 6 &  85.333$\pm$0.795 & 85.333$\pm$0.659 & 85.556$\pm$0.994 \\ \bottomrule
\end{tabular}
\end{center}
\end{table}

\subsection{EFI for Complex Hypothesis Tests}

\begin{figure}[H]
    \centering
    \includegraphics[width=0.85\textwidth,height=4.4in]{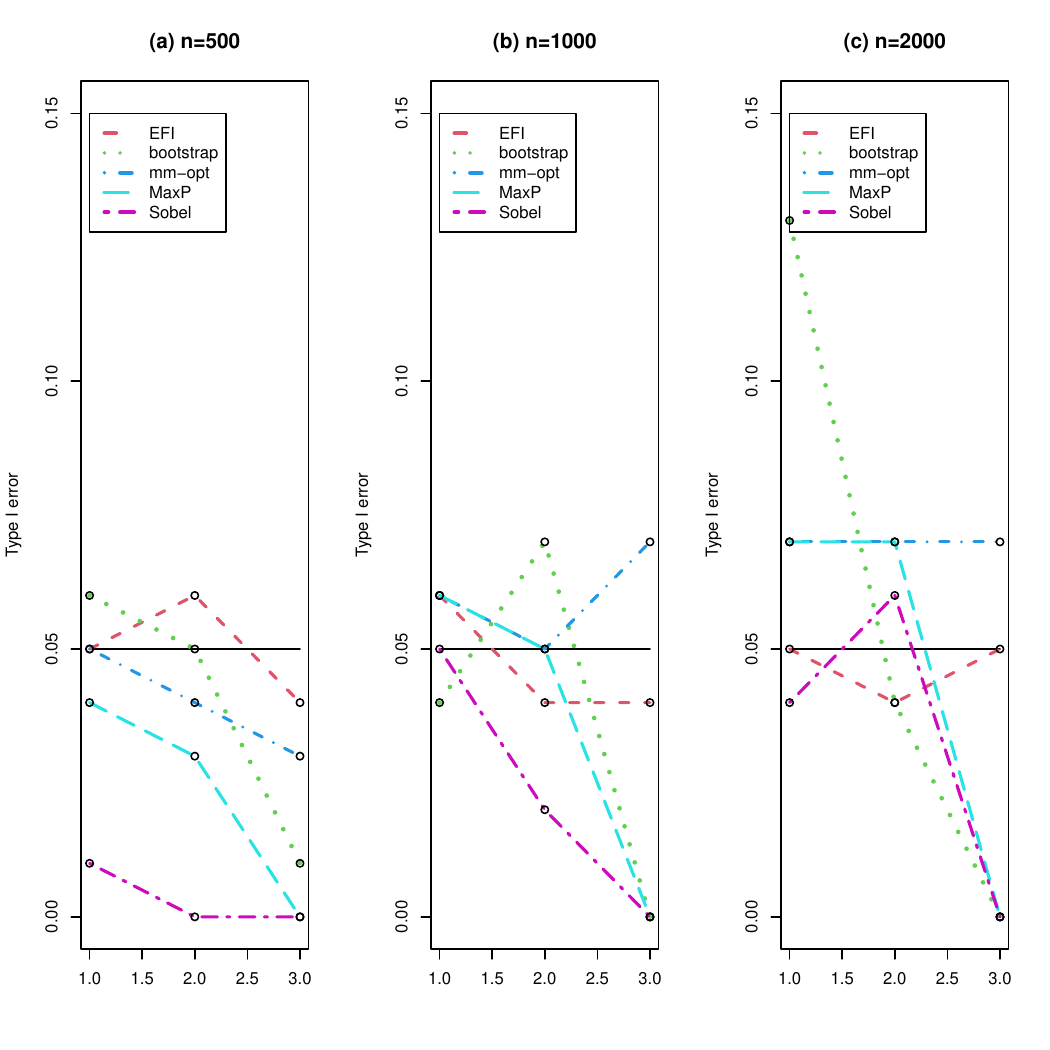}
    \caption{Graphical representation of Table \ref{medtab_typeI}, 
    where `1', `2' and `3' in $x$-axis represent the experimental settings $(\beta,\gamma)=(0.2,0)$,  $(\beta,\gamma)=(0,0.2)$ and 
    $(\beta,\gamma)=(0,0)$, respectively. }  
    \label{fig:Table6}
\end{figure}

\begin{figure}[H]
    \centering
    \includegraphics[width=0.85\textwidth,height=4.4in]{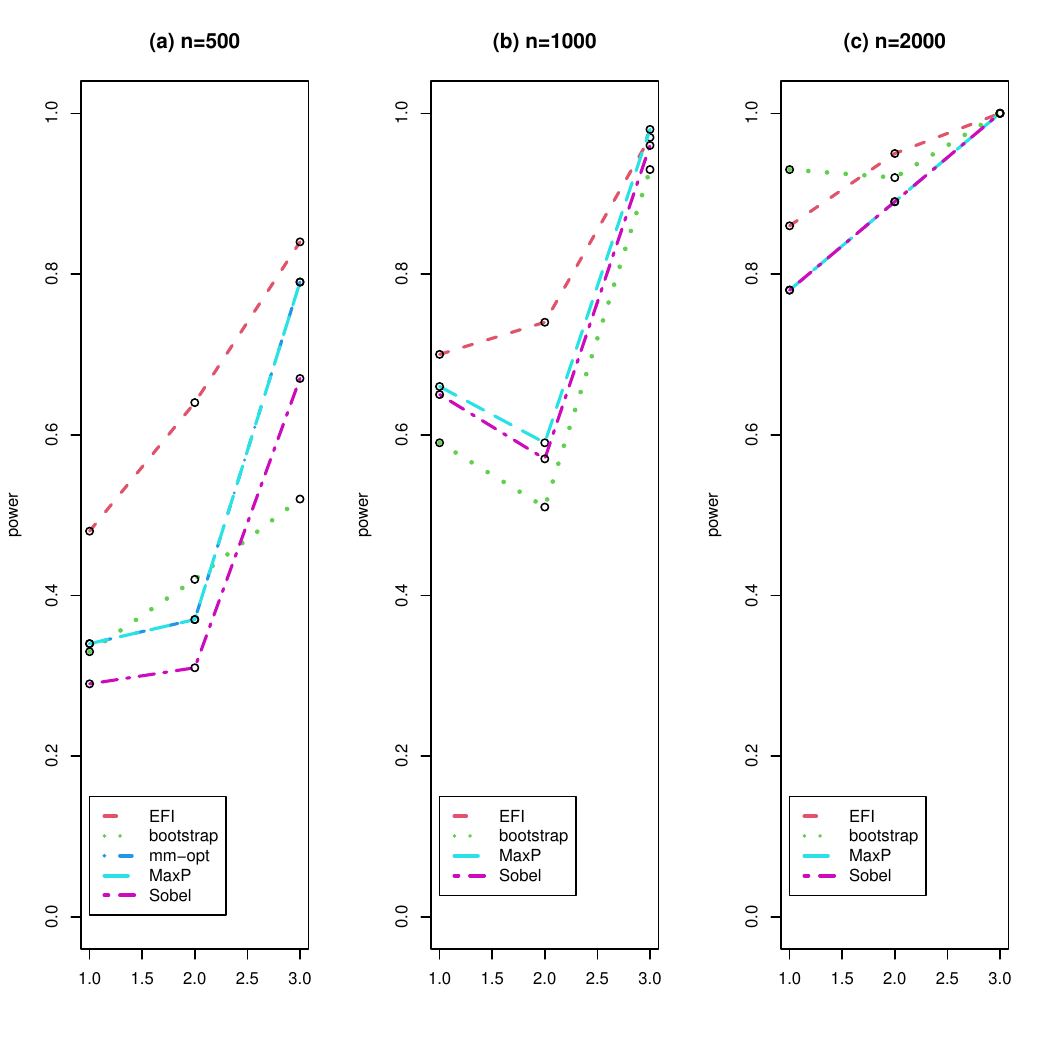}
    \caption{Graphical representation of Table \ref{medtab_power}, 
    where `1', `2' and `3' in $x$-axis represent the experimental settings $(\beta,\gamma)=(0.1,0.4)$,  $(\beta,\gamma)=(-0.1,0.4)$ and 
    $(\beta,\gamma)=(0.2,0.2)$, respectively. }  
    \label{fig:Table7}
\end{figure}

\section{Experimental Setting} \label{moresettings}

To enforce a sparse DNN to be learned for the inverse function $g(\cdot)$, we impose 
the following mixture Gaussian prior on each element of $\bw_n$: 
\begin{equation}\label{sparseprior}
    \pi(w)\sim \rho  N(0,\sigma_{1}^2)+(1-\rho)N(0,\sigma_{0}^2),
\end{equation}
where $w$ denotes a generic element of $\bw_n$ and, unless stated otherwise, 
we set $\rho=1e-2$, $\sigma_0=1e-5$ and $\sigma_1=0.02$.
The elements of $\bw_n$ are {\it a priori} independent. 

In all experiments of this paper, we use {\it ReLU} as the activation function, and 
set the learning rate sequence $\{\epsilon_k\}$ and the step size sequence $\{\gamma_k\}$ in the forms given in Theorem \ref{thm1}. Specifically, we set $\alpha=13/14$ and $\beta=4/7$ unless stated otherwise, and set different values of $C_{\epsilon}$, $c_{\epsilon}$, $C_{\gamma}$ and $c_{\gamma}$ for different experiments as given below. 

\subsection{Linear regression}
For both cases with known and unknown $\sigma^2$, we use a DNN with structure  
$12-300-100-d_{\btheta}$ for inverse function approximation, where $d_{\btheta}$ denotes the dimension of $\btheta$.   

 \paragraph{Known $\sigma^2$}
For EFI-a, we set $\eta=100$ and $\lambda=10$; and for EFI, we set $\eta=10$ and $\lambda=10$. EFI-a and EFI share the same learning rate and step size sequences with $(C_{\epsilon},c_{\epsilon},C_{\gamma},c_{\gamma})=(50000,10000,5000,100000)$. 
We set the burn-in period $\mK=1000$ and the iteration number $M=100,000$.
The Markov chain is thinned by a factor of $B=10$ in sample collection, i.e., $M/B=10,000$ 
$\bar{\btheta}$-samples were collected for calculation of the coverage rates and CI-widths. For EFI, we employ the same parameter settings 
for different activation functions.

 \paragraph{Unknown $\sigma^2$}
 For EFI, we used SGHMC in latent variable sampling, i.e., we simulate $\bZ^{(k+1)}$ in the following formula: 
  \begin{equation} \label{SGHMCeq}
\begin{split}
\bV_n^{(k+1)} &= (1-\zeta) \bV_n^{(k)} +\epsilon_{k+1} \nabla_{\bZ_n} \log \pi(\bZ_n^{(k)}|\bX_n,\bY_n,\bw^{(k)}) +\sqrt{2 \zeta \tau \epsilon_{k+1}} \be^{(k+1)}, \\
\bZ_n^{(k+1)}&=\bZ_n^{(k)}+\bV_n^{(k+1)}, 
\end{split}
\end{equation}
where $\tau=1$, $0<\zeta\leq 1$ is the momentum parameter, $\be^{(k+1)} \sim N(0,I_d)$, 
and $\epsilon_{k+1}$ is the learning rate.
It is worth noting that the algorithm is reduced to SGLD if one sets $\zeta=1$.

In simulations, we set the decaying parameters 
$\alpha=\beta=4/7$, and the Markov chain is thinned by a factor of $B$ as below for $M/B=10,000$ samples were used for calculation of the coverage rates and CI-widths.

 \begin{itemize}
\item $(\eta=2,\lambda=30):$ We set $\zeta=0.025$ and $(C_{\epsilon},c_{\epsilon},C_{\gamma},c_{\gamma})=(6500,100000,1700,100000)$, $(\mK,M)=(10000,50000)$ thinned by $B=5$; 

\item $(\eta=2,\lambda=40):$ We set $\zeta=0.025$ and $(C_{\epsilon},c_{\epsilon},C_{\gamma},c_{\gamma})=(5600,100000,1400,100000)$, $(\mK,M)=(10000,90000)$ thinned by $B=9$; 

\item $(\eta=2,\lambda=50):$   We set $\zeta=0.05$ and $(C_{\epsilon},c_{\epsilon},C_{\gamma},c_{\gamma})=(4000,100000,1000,100000)$,, $(\mK,M)=(10000,200000)$ thinned by $B=20$;

\item $(\eta=4,\lambda=50):$ We set $\zeta=0.005$ and $(C_{\epsilon},c_{\epsilon},C_{\gamma},c_{\gamma})=(1950,80000,490,80000)$,, $(\mK,M)=(10000,120000)$ thinned by $B=12$; 

\end{itemize}

   \subsection{Behrens-Fisher problem}
   We use a DNN with structure 2-20-10-2 and set $\eta=5$ and $\lambda=20$. The burn-in period $\mK=10000$,  the iteration number $M=40000$ and 60000 for $n=50$ and $500$, respectively. The Markov chain is thinned by a factor of $B=4$ and 6 for $n=50$ and 500, 
   respectively, in sample collection. This makes that $M/B=10,000$ samples are used in  calculation of the coverage rates and CI-widths for each case. 
 \begin{itemize}
\item $\sigma_1^2=0.25$, $\sigma_2^2=1:$ (i) for $n=50$, we set $\zeta=0.01$, and $(C_{\epsilon},c_{\epsilon},C_{\gamma},c_{\gamma})=(2500,100000,2500,100000)$; 
(ii) for $n=500$,  we set $\zeta=0.005$, and $(C_{\epsilon},c_{\epsilon},C_{\gamma},c_{\gamma})=(3000,100000,3000,100000)$; 

\item $\sigma_1^2=1$, $\sigma_2^2=1:$  (i) for $n=50$, we set $\zeta=0.05$, and $(C_{\epsilon},c_{\epsilon},C_{\gamma},c_{\gamma})=(2800,100000,2800,100000)$; 
(ii) for $n=500$,  we set $\zeta=0.028$, and $(C_{\epsilon},c_{\epsilon},C_{\gamma},c_{\gamma})=(3100,100000,3100,100000)$.

\end{itemize}

\subsection{Bivariate normal}
 
  We used a DNN with structure 4-80-20-5 for inverse function approximation, and we set $\eta=2$ and $\lambda=50$. 
 We used SGHMC in latent variable sampling as in (\ref{SGHMCeq})
. In simulations, we set the momentum parameter $\zeta=0.1$,  the decaying parameters 
$\alpha=\beta=4/7$, $(C_{\epsilon},c_{\epsilon},C_{\gamma},c_{\gamma})=(4500,100000,1100,100000)$. 
the burn-in period $\mK=10000$,  the iteration number $M=50000$, and the Markov chain is thinned by a factor of $B=5$ in sample collection, i.e., $M/B=10,000$ samples were used for calculation of the coverage rates and CI-widths.

\subsection{Fidelity in Parameter Estimation}
 
We used a DNN with structure 12-300-100-11 for inverse function approximation, and set $(\eta,\lambda)=(2,50)$.
The tempering SGLD algorithm is used in the latent variable sampling step, where we set 
the temperature sequence $\tau_t=max(100*(0.9999)^{t},1)$. 
For the learning rate and step size sequences, we set $(C_{\epsilon},c_{\epsilon},C_{\gamma},c_{\gamma})=(50000000,10000000,50,10000)$.
For sample collections, we set $\mK=50,000$, $M=150,000$, and $B=15$.

 \subsection{Nonlinear Regression in the Supplement} \label{setting:nonlinear}
 
  We used a DNN with structure 3-150-50-8 for inverse function approximation, and we set $(\eta,\lambda)=(500,0.2)$ in order to avoid a local trap of fitting $\bz_n$ to $\by_n$. 
  For the learning rate and step size sequences,  we set $(C_{\epsilon},c_{\epsilon},C_{\gamma},c_{\gamma})=(1,10000000,1,100)$ for iterations $t<50,000$,
  and set $(C_{\epsilon},c_{\epsilon},C_{\gamma},c_{\gamma})=(1000,100000,10,10000)$ 
  for $t\geq 50,000$. For sample collections, we set 
   $\mK=60,000$, $M=150,000$ and $B=15$. 


\subsection{Logistic regression in the Supplement}

 For EFI, we set $\eta=2$ and $\lambda=1000$. 
 We used SGHMC (\ref{SGHMCeq}) in latent variable sampling.
In simulations, we set the momentum parameter $\zeta=0.01$,  the decaying parameters 
$\alpha=\beta=2/7$, 
  $(C_{\epsilon},c_{\epsilon},C_{\gamma},c_{\gamma})=(50000,100000,30000,100000)$. 
the burn-in period $\mK=10000$,  the iteration number $M=50000$, and the Markov chain is thinned by a factor of $B=5$ in sample collection, i.e., $M/B=10,000$ samples were used for calculation of the coverage rates and CI-widths.

   

 

\subsection{EFI for Semi-Supervised Learning}

We used a DNN with structure $(p+2)-90-30-p$ for inverse function approximation, where $p$ corresponds to the dimension of $\bx$ for all cases. 
For EFI on both full label cases, and labeled-data only cases (use $50\%$ of training data), we set $\alpha=\beta=\frac{2}{7},\eta=5,\lambda=200,\mK=10000,M=40000,B=4$ with $\zeta=0.1$, $(C_{\epsilon},c_{\epsilon},C_{\gamma},c_{\gamma})=(100000,100000,2000,100000)$. 
For semi-supervised EFI, the same parameter settings have been used with the exceptions given as follows:


\begin{itemize}
\item {\bf Beast-Cancer:}
 $(\eta,\lambda)=(5,200)$;

\item {\bf Diabetes:}
 $(\eta,\lambda)=(2,500)$ and $(C_{\epsilon},c_{\epsilon},C_{\gamma},c_{\gamma})=(200000,100000,1000,100000)$;
\item {\bf Divorce:}
 $(\eta,\lambda)=(10/3,300)$;
\item {\bf Raisin:}
 $(\eta,\lambda)=(2,500)$.
 \end{itemize}

\subsection{EFI for Complex Hypothesis Tests}

We used a DNN with structure 7-180-30-9 for inverse function approximation, and set $\alpha=\beta=\frac{4}{7},\eta=10,\lambda=10,\mK=10000,M=50000,B=5$. In addition, we varied the values
of other parameters according to the problem and sample size.

\paragraph{Type-I error} For different sample sizes, we set the parameters as follows:

\begin{itemize}
\item $n=500$. For case 1, we set $\zeta=0.1$ and $(C_{\epsilon},c_{\epsilon},C_{\gamma},c_{\gamma})=(290000,100000,4000,100000)$; for case 2, we set $\zeta=0.1$ and  $(C_{\epsilon},c_{\epsilon},C_{\gamma},c_{\gamma})=(100000,100000,2000,100000)$; for
case 3, we set $\zeta=0.1$ and $(C_{\epsilon},c_{\epsilon},C_{\gamma},c_{\gamma})=(100000,100000,2000,100000)$.

\item $n=1000$. For case 1, we set $\zeta=0.1$ and $(C_{\epsilon},c_{\epsilon},C_{\gamma},c_{\gamma})=
(100000,100000,4000,100000)$; for case 2, we set $\zeta=1$ and $(C_{\epsilon},c_{\epsilon},C_{\gamma},c_{\gamma})=(200000,100000,4000,100000)$; 
for case 3, we set $\zeta=0.1$ and  $(C_{\epsilon},c_{\epsilon},C_{\gamma},c_{\gamma}) =(2000,100000,1000,100000)$. 

\item $n=2000$. For case 1 and case 2, we set $\zeta=1$ and  $(C_{\epsilon},c_{\epsilon},C_{\gamma},c_{\gamma})=(200000,100000,4000,100000)$; and for case 3, 
we set $\zeta=0.1$ and $(C_{\epsilon},c_{\epsilon},C_{\gamma},c_{\gamma})=(2000,100000,1000,100000)$.
\end{itemize}


\paragraph{Power} For all cases, we set the parameters 
$\zeta=0.1$ and  $(C_{\epsilon},c_{\epsilon},C_{\gamma},c_{\gamma})=(2000,100000,1000$, $100000)$.

 \newpage 

\bibliographystyle{asa}
\bibliography{reference}

\end{document}